\documentclass{article}

 \usepackage[preprint]{neurips_2026}

\usepackage[utf8]{inputenc} 
\usepackage[T1]{fontenc}    
\usepackage[pagebackref=true]{hyperref}       
\renewcommand*{\backref}[1]{}
\renewcommand*{\backrefalt}[4]{%
    \ifcase #1 \relax
    \or
        (Cited on page #2.)%
    \else
        (Cited on pages #2.)%
    \fi
}
\usepackage{url}            
\usepackage{booktabs}       
\usepackage{amsfonts}       
\usepackage{nicefrac}       
\usepackage{microtype}      
\usepackage{xcolor}         

\title{
    \algntitle : Asynchronous Linear Minimization Oracle Momentum Method
}

\author{%
  Abdurakhmon~Sadiev \\
  KAUST \\
  \texttt{abdurakhmon.sadiev@kaust.edu.sa} \\
  \And
  Artavazd~Maranjyan \\
  KAUST \\
  \texttt{artavazd.maranjyan@kaust.edu.sa} \\
  \And
  Ivan Ilin \\
  KAUST \\
  \texttt{ivan.ilin@kaust.edu.sa} \\
  \And
  Peter~Richt\'{a}rik \\
  KAUST \\
  \texttt{peter.richtarik@kaust.edu.sa}
}

\usepackage{graphicx}
\usepackage{apptools}
\usepackage[flushleft]{threeparttable}
\usepackage{array,booktabs,makecell}
\usepackage{multirow}
\usepackage{nicefrac}
\usepackage{amsthm}
\usepackage{mathtools}
\usepackage{amsmath,amsfonts,bm,amssymb}
\usepackage{dsfont}
\usepackage{wrapfig}
\usepackage{caption}
\usepackage{siunitx}
\usepackage{color}
\usepackage{thm-restate}
\usepackage{nccmath}
\usepackage{empheq}
\usepackage{bbm}
\usepackage{tabularx}

\usepackage{algorithmic}
\usepackage{algorithm}
\definecolor{algcommentgray}{gray}{0.45}
\newcommand{\algcomment}[1]{\hfill {\color{algcommentgray}$\triangleright$ #1}}
\usepackage{filecontents}

\usepackage[capitalize,noabbrev]{cleveref}

\newcommand{\circledOne}{\text{\ding{172}}}
\newcommand{\circledTwo}{\text{\ding{173}}}
\newcommand{\circledThree}{\text{\ding{174}}}
\newcommand{\circledFour}{\text{\ding{175}}}
\newcommand{\circledFive}{\text{\ding{176}}}
\newcommand{\circledSix}{\text{\ding{177}}}
\newcommand{\circledSeven}{\text{\ding{178}}}
\newcommand{\circledEight}{\text{\ding{179}}}

\usepackage{colortbl}
\definecolor{bgcolor}{rgb}{0.76,0.88,0.50}
\definecolor{bgcolor0}{rgb}{0.93,0.99,1}
\definecolor{bgcolor1}{rgb}{0.8,1,1}
\definecolor{bgcolor2}{rgb}{0.8,1,0.8}
\definecolor{bgcolor3}{rgb}{0.50,0.90,0.50}
\usepackage{tcolorbox}
\usepackage{pifont}

\definecolor{mydarkblue}{rgb}{0,0.16,0.54}
\definecolor{mydarkgreen}{RGB}{39,130,67}
\definecolor{mydarkorange}{RGB}{236,147,14}
\definecolor{mydarkgreen}{RGB}{0,100,0}
\definecolor{ruby}{RGB}{155,17,30}
\definecolor{chili}{RGB}{191,0,0}
\definecolor{sangria}{RGB}{146,0,10}
\definecolor{burgundy}{RGB}{128,0,32}
\definecolor{darkred}{RGB}{132,0,0} 
\definecolor{cherry}{RGB}{192,0,0} 
\definecolor{myaccent}{rgb}{0,0.45,0.45}
\definecolor{anotherdarkred}{rgb}{0.45,0,0.05}

\newcommand{\red}{\color{cherry}}

\usepackage[textsize=tiny]{todonotes}

\usepackage{xspace}
\usepackage{relsize}

\newcommand{\norm}[1]{\left\| #1 \right\|}
\newcommand{\sqnorm}[1]{\left\| #1 \right\|^2}

\newcommand{\flr}[1]{\left\lfloor #1\right\rfloor} 
\newcommand{\inp}[2]{\left\langle#1,#2\right\rangle} 

\newcommand{\R}{\mathbb{R}} 
\newcommand{\N}{\mathbb{N}} 
\newcommand{\E}[1]{\mathbb{E}\left[#1\right]}

\newcommand{\Exp}[1]{{\mathbb{E}}\left[#1\right]}
\newcommand{\ExpSub}[2]{{\mathbb{E}}_{#1}\left[#2\right]}

\newcommand{\cD}{\mathcal{D}}

\newcommand{\cO}{\mathcal{O}}

\newcommand{\cT}{\mathcal{T}}

\newcommand{\eqdef}{\coloneqq} 

\newcommand{\minimize}{\mathop{\mathrm{minimize}}}

\newcommand{\lmo}{\operatorname{lmo}}

\makeatletter
\newcommand{\vast}{\bBigg@{4}}

\DeclareMathOperator*{\argmin}{arg\,min}

\def\<{\left\langle}
\def\>{\right\rangle}
\def\[{\left[}
\def\]{\right]}
\def\({\left(}
\def\){\right)}

\definecolor{myblue}{rgb}{0.3,0.25,0.2}
\definecolor{myorange}{rgb}{0.3,0.25,0.2}
\newcommand{\checkmarkgreen}{\textbf{\color{myblue}\ding{52}}}
\newcommand{\crossmarkred}{{\color{myorange}\ding{56}}}

\hypersetup{ %
  colorlinks=true,
  linkcolor=linkcolor,
  citecolor=linkcolor,
  filecolor=linkcolor,
  urlcolor=linkcolor,
}
\definecolor{linkcolor}{rgb}{0,0.08,0.45} 

\definecolor{alggray}{gray}{0.35} 
\newcommand{\algname}[1]{\text{\color{alggray}\sffamily\relscale{0.93}#1}\xspace}
\newcommand{\algntitle}{{\red Ringmaster~LMO}\xspace}
\newcommand{\algn}{\algname{\red Ringmaster~LMO}}

\newcommand{\ringmaster}{\algname{Ringmaster~ASGD}}
\newcommand{\ringmastertitle}{Ringmaster~ASGD}
\newcommand{\ringmastermuon}{\algname{Ringmaster~Muon}}
\newcommand{\ringmastermuonagnostic}{\algname{Parameter-agnostic Ringmaster~Muon}}

\newcommand{\rennalamuon}{\algname{Rennala~Muon}}
\newcommand{\sgd}{\algname{SGD}}

\newcommand{\hogwild}{\algname{HOGWILD!}}

\newcommand{\muon}{\algname{Muon}}
\newcommand{\muondelayadaptive}{\algname{Delay-adaptive Muon}}

\newcommand{\adam}{\algname{Adam}}

\newcommand{\adamw}{\algname{AdamW}}

\newcommand{\scion}{\algname{Scion}}

\usepackage{thmtools}
\usepackage{thm-restate}
\usepackage{mdframed}

\theoremstyle{plain}
\newtheorem{theorem}{Theorem}[section]

\newtheorem{lemma}[theorem]{Lemma}

\newtheorem{assumption}[theorem]{Assumption}
\theoremstyle{definition}

\theoremstyle{remark}

\theoremstyle{plain} 

\newmdenv[
  font=\normalfont\bfseries,
  linecolor=black,
  linewidth=0.8pt,
  topline=false,
  bottomline=false,
  leftline=false,
  rightline=false,
  backgroundcolor=gray!13,
  skipabove=2pt,
  skipbelow=2pt,
  innertopmargin=-5pt,
  innerbottommargin=4pt,
  innerrightmargin=4pt,
  innerleftmargin=4pt,
]{myboxed}

\newmdtheoremenv[
  font=\normalfont\bfseries,
  linecolor=black,
  linewidth=0.8pt,
  topline=false,
  bottomline=false,
  leftline=false,
  rightline=false,
  backgroundcolor=gray!13,
  skipabove=2pt,
  skipbelow=2pt,
  innertopmargin=-5pt,
  innerbottommargin=4pt,
  innerrightmargin=4pt,
  innerleftmargin=4pt,
]{boxedtheorem}[theorem]{Theorem}

\newmdtheoremenv[
  font=\normalfont\bfseries,
  linecolor=black,
  linewidth=0.8pt,
  topline=false,
  bottomline=false,
  leftline=false,
  rightline=false,
  backgroundcolor=gray!13,
  skipabove=2pt,
  skipbelow=2pt,
  innertopmargin=-5pt,
  innerbottommargin=4pt,
  innerrightmargin=4pt,
  innerleftmargin=4pt,
]{boxedlemma}[theorem]{Lemma}

\newmdtheoremenv[
  font=\normalfont\bfseries,
  linecolor=black,
  linewidth=0.8pt,
  topline=false,
  bottomline=false,
  leftline=false,
  rightline=false,
  backgroundcolor=gray!13,
  skipabove=2pt,
  skipbelow=2pt,
  innertopmargin=-5pt,
  innerbottommargin=4pt,
  innerrightmargin=4pt,
  innerleftmargin=4pt,
]{boxeddefinition}[theorem]{Definition}

\newmdtheoremenv[
  font=\normalfont\bfseries,
  linecolor=black,
  linewidth=0.8pt,
  topline=false,
  bottomline=false,
  leftline=false,
  rightline=false,
  backgroundcolor=gray!13,
  skipabove=2pt,
  skipbelow=2pt,
  innertopmargin=-5pt,
  innerbottommargin=4pt,
  innerrightmargin=4pt,
  innerleftmargin=4pt,
]{boxedassumption}[theorem]{Assumption}

\newtheorem{innercustomthm}{Theorem}
\newenvironment{restate-theorem}[1]
  {\renewcommand\theinnercustomthm{#1}\innercustomthm}
  {\endinnercustomthm}

\newtheorem{innercustomlemma}{Lemma}
\newenvironment{restate-lemma}[1]
  {\renewcommand\theinnercustomlemma{#1}\innercustomlemma}
  {\endinnercustomlemma}

\newtheorem{innercustomproposition}{Proposition}
\newenvironment{restate-proposition}[1]
  {\renewcommand\theinnercustomproposition{#1}\innercustomproposition}
  {\endinnercustomproposition}

\newenvironment{restate-boxedtheorem}[1]
  {\renewcommand\theinnercustomthm{#1}\begin{myboxed}\begin{innercustomthm}}
  {\end{innercustomthm}\end{myboxed}}

\newenvironment{restate-boxedlemma}[1]
  {\renewcommand\theinnercustomlemma{#1}\begin{myboxed}\begin{innercustomlemma}}
  {\end{innercustomlemma}\end{myboxed}}

\newcommand*{\sketchproofname}{Sketch of Proof}

\begin{document}
\maketitle

\begingroup
\renewcommand{\thefootnote}{}
\footnotetext{Code: \href{https://github.com/vectozavr/ringmaster-lmo}{\texttt{github.com/vectozavr/ringmaster-lmo}}.}
\endgroup

\begin{abstract}
%
%
%
%
    Muon has recently emerged as a strong alternative to \adamw for training neural networks, with encouraging large-scale pretraining results and growing evidence that matrix-structured updates can be faster in practice.
    Yet Muon, and more generally Linear Minimization Oracle (LMO) based methods, are typically used synchronously.
    This is problematic in heterogeneous distributed systems, where workers complete gradient computations at different speeds and synchronous training must repeatedly wait for slower workers.
    In this work, we introduce \algn, an asynchronous LMO-based momentum method for unconstrained stochastic nonconvex optimization.
    Our method builds on the delay-thresholding idea of \ringmaster.
    For \sgd-type methods, \ringmaster achieves optimal time complexity by discarding overly stale gradients.
    \algn extends this mechanism to general LMO-based updates. We establish convergence guarantees under generalized $(L_0,L_1)$-smoothness and further develop a parameter-agnostic variant with decreasing stepsizes and adaptive delay thresholds.
    Finally, we translate our iteration guarantees into time complexity bounds under heterogeneous worker computation times.
    In the classical Euclidean smooth setting, these bounds recover the optimal time complexity of \ringmaster.
    Experiments on stochastic quadratic problems and NanoChat language-model pretraining show that the advantages of \algn grow with system heterogeneity and that the method outperforms strong synchronous and asynchronous baselines.
\end{abstract}
\section{Introduction}\label{sec:introduction}
For more than a decade, \adam \citep{kingma2014adam} has been the default choice for training large neural networks, with its decoupled weight-decay variant \adamw \citep{loshchilov2017decoupled} more recently becoming the default variant.
For matrix-valued layers, however, these methods treat each weight matrix as a vector, rather than explicitly using its matrix structure.
Recently, \citet{jordan2024muon} proposed the \muon optimizer, showing that using the matrix structure of hidden-layer weights can be beneficial.
\muon forms the update direction from the matrix gradient via approximate orthogonalization.
This matrix-structured update has led to strong empirical results, including the \texttt{nanoGPT} speedrun where \muon improves on the \adamw baseline \citep{modded_nanogpt_2024}.
Recent pretraining benchmarks also report that \muon can be substantially more compute-efficient than \adamw \citep{shah2025practical,liu2025muon,wen2026fantastic}.
More generally, \muon fits into the framework of linear minimization oracle (LMO)-based optimizers introduced by \citet{pethick2025scion}.
This viewpoint makes the geometry of the update explicit: the choice of norm determines the update direction, and for matrix parameters the spectral-norm ball recovers the orthogonalized direction used by \muon.

Despite this progress, LMO-based optimizers are still typically used within synchronous training pipelines.
When several workers compute gradients in parallel, the usual synchronous implementation must wait for the slowest worker before forming the next update.
Asynchronous methods, going back at least to \hogwild \citep{recht2011hogwild}, avoid this synchronization barrier by updating the model whenever a worker returns a gradient.
The price is delay: a gradient arriving at iteration $k$ may have been computed at an older iterate with delay $\delta_k\ge 0$.
Recent analyses have clarified the effect of such delays for asynchronous \sgd, and delay-thresholding variants can achieve optimal time complexity under heterogeneous worker speeds \citep{tyurin2023optimal,maranjyan2025ringmaster,maranjyan2026ringleader}.
These results, however, are built for \sgd-type updates in Euclidean setting, and do not directly apply to LMO-based optimizers.

This leaves the central question of this paper:
\begin{center}
    \emph{Can asynchronous LMO-based methods recover the optimal time complexity known for Euclidean setting.}
\end{center}
We show that this can be done.
Building on the delay-thresholding mechanism of \ringmaster \citep{maranjyan2025ringmaster} and the LMO-based momentum framework of \citet{pethick2025scion}, we introduce an asynchronous LMO-based method for unconstrained stochastic nonconvex optimization.
To the best of our knowledge, this is the first asynchronous LMO-based method with time complexity guarantees under heterogeneous worker computation times.
In the classical Euclidean smooth setting, the bound recovers the optimal time complexity of \ringmaster.
\begin{table}
    \centering
    \caption{High-level comparison with the closest theoretical works.
    All listed methods can be viewed through the LMO-based update perspective; the norm column reports the geometry covered by the available analysis, and `Euclidean' indicates results specialized to normalized \sgd.
    The asynchronous column indicates whether the method is analyzed in the delayed-gradient asynchronous setting, the parameter-agnostic column indicates whether the method avoids tuning the main schedule from problem parameters, and the generalized smoothness column indicates whether the theory covers our generalized $(L_0,L_1)$-smoothness model (\Cref{ass:gen_smooth}).}
    \label{tab:method_comparison}
    \begin{tabular}{ccccc}
        \toprule
        Method & Norm & Asynchronous & \makecell{Parameter-\\agnostic} & \makecell{Generalized\\$(L_0,L_1)$-\\smoothness} \\
        \midrule
        \makecell{Existing LMO-based methods \\
        \citep{pethick2025scion,pethick2025generalized}}
        & Any norm & \crossmarkred & \crossmarkred & \checkmarkgreen \\
        \cmidrule(r){1-1}
        \makecell{\algname{NSGD-M} \\
        \citep{hubler2024parameter_agnostic}}
        & Euclidean & \crossmarkred & \checkmarkgreen & \checkmarkgreen \\
        \cmidrule(r){1-1}
        \makecell{Ringmaster Asynchronous \algname{NSGD-M} \\ 
        \citep{wu2026optimalasynchronous}}
        & Euclidean & \checkmarkgreen & \crossmarkred & \crossmarkred \\
        \cmidrule(r){1-1}
        \makecell{\algn \textbf{[new]}\\ 
        (\Cref{alg})}
        & Any norm & \checkmarkgreen & \checkmarkgreen & \checkmarkgreen \\
        \bottomrule
    \end{tabular}
\end{table}
\subsection{Contributions}
Our contributions are the following:
\begin{itemize}
    \item \textbf{First asynchronous LMO-based method.}
    We propose \algn (\Cref{alg}), an asynchronous LMO-based method that combines momentum with delay-thresholding.
    To the best of our knowledge, this is the first asynchronous LMO-based method with time complexity guarantees under heterogeneous worker computation times.
    The method extends the \ringmaster delay-control principle from \sgd-type updates to LMO-based updates; see \Cref{sec:method}.
    \item \textbf{Convergence under generalized smoothness.}
    The analysis in \Cref{sec:fixed_parameter,sec:param_agnostic} establishes convergence under generalized $(L_0,L_1)$-smoothness \citep{zhang2020gradient,zhang2020improved}, which includes the classical smooth case as $L_1=0$.
    The resulting rates recover the standard stochastic LMO-type dependence while allowing delayed gradients generated by asynchronous workers; see \Cref{thm:iteration_complexity,thm:iteration_complexity_param_agnostic}.
    \item \textbf{Parameter-agnostic scheduling.}
    \Cref{sec:param_agnostic} gives a variant with decreasing stepsizes and adaptive delay thresholds.
    This removes the need to tune the main schedule using problem parameters and, in particular, provides a parameter-agnostic extension that was not present in the original \ringmaster analysis; see \Cref{thm:iteration_complexity_param_agnostic}.
    \item \textbf{time complexity guarantees.}
    \Cref{thm:time_complexity,thm:time_complexity_param_agnostic,thm:time_arbitrary} translate the iteration complexity into time complexity under both the fixed computation model and the universal time-varying computation model.
    In the standard smooth Euclidean setting, these bounds recover the same optimal time complexity as \ringmaster.
    \item \textbf{Experiments.}
    \Cref{sec:experiments} compares the proposed method against synchronous and asynchronous baselines.
    On both stochastic quadratic problems and NanoChat pretraining, the proposed method is competitive in mild heterogeneity and its advantage grows as worker-speed heterogeneity increases.
\end{itemize}
\Cref{tab:method_comparison} places our method in the context of the closest theoretical works.
\subsection{Related work}\label{sec:related_work}
We review the prior work most closely related to our setting.
\paragraph{LMO-based and matrix-structured optimizers.}
The empirical success of \muon \citep{jordan2024muon,modded_nanogpt_2024} has motivated a line of work seeking to explain and generalize its matrix-structured update.
\citet{bernstein2024oldoptimizernewnorm} relate \muon to older adaptive and preconditioned methods, including \algname{Shampoo} \citep{gupta2018shampoo}, by emphasizing the role of the underlying norm.
\citet{pethick2025scion} give a broader optimization framework in which the update direction is chosen by a linear minimization oracle over a norm ball.
LMOs are classical in Frank-Wolfe and conditional-gradient methods \citep{frank1956algorithm,hazan2008sparse,clakson2010coresets,jaggi2013revisitingfrankwolfe}, but the \scion framework uses them for unconstrained stochastic optimization and recovers several normalized or structured-gradient methods through the choice of norm.
In parallel, \citet{kovalev2025understandinggradientorthogonalizationdeep} interprets gradient orthogonalization through non-Euclidean trust-region optimization.
More broadly, the concept of spectral descent and orthogonalized updates has historical roots in deep learning \citep{carlson2015stochasticspectral,carlson2015preconditioned,carlson2016stochastic,tuddenham2022orthogonalising}.
Subsequent work has refined this theory under generalized smoothness \citep{riabinin2025gluon, pethick2025generalized}, incorporated second-order momentum variants \citep{khirirat2025better}, analyzed inexact Newton-Schulz orthogonalization \citep{shulgin2026beyondtheideal}, and adapted the LMO viewpoint to layerwise neural-network training \citep{riabinin2025gluon}.
These works are synchronous: they do not analyze delayed gradients produced by asynchronous workers, nor do they provide time complexity guarantees under heterogeneous computation times.

\paragraph{Asynchronous stochastic optimization.}
Asynchronous stochastic optimization removes the synchronization barrier by updating the model whenever a worker returns a gradient \citep{tsitsiklis1986distributed,agarwal2011distributed,recht2011hogwild,maranjyan2025thesis}.
The main analytical difficulty is delay: the gradient used at iteration $k$ may have been computed at an earlier iterate $x_{k-\delta_k}$.
Classical analyses control this error through bounded-delay assumptions \citep{feyzmahdavian2016asynchronous,lian2015asynchronous,arjevani2020tight}, while more recent work uses delay-adaptive stepsizes \citep{mishchenko2022asynchronous,koloskova2022sharper,wu2022delayadaptive}.
For heterogeneous worker speeds, the relevant objective is time complexity rather than only counting iterations.
\citet{tyurin2023optimal} identified optimal time complexities under a fixed computation model.
This lower-bound perspective was later broadened to cover time-varying worker behavior \citep{tyurin2024tighttimecomplexitiesparallel} as well as variance-reduced parallel stochastic optimization \citep{tovmasyan2026rennala_mvr}.
\ringmaster \citep{maranjyan2025ringmaster} achieves these guarantees for asynchronous \sgd by accepting only gradients whose delay is below a threshold.
\citet{maranjyan2026ringleader} and \citet{mahran2026rescaled} address the corresponding gap for data- and system-heterogeneous settings.
These results are Euclidean \sgd-type methods; they do not cover LMO-based directions or the non-Euclidean geometries needed to recover \muon-like updates.

\paragraph{Closest baselines and parameter agnosticism.}
The closest asynchronous baseline to our method is \citet{wu2026optimalasynchronous}, which combines the delay-thresholding mechanism of \citet{maranjyan2025ringmaster} with normalized \sgd and momentum under heavy-tailed noise.
Their method corresponds to the Euclidean-norm case of our update and therefore does not cover general LMOs, matrix-structured directions, or the $(L_0,L_1)$-smoothness model used here.
Generalized $(L_0,L_1)$-smoothness was introduced as a way to capture losses whose local smoothness can grow with the gradient norm \citep{zhang2020gradient,zhang2020improved}, and it is also the setting of recent synchronous LMO analyses \citep{riabinin2025gluon, pethick2025generalized, khirirat2025better}.
Separately, \citet{hubler2024parameter_agnostic} develop parameter-agnostic normalized \sgd with momentum under relaxed smoothness in the synchronous Euclidean setting.
Our parameter-agnostic variant brings this idea to asynchronous LMO-based optimization: the same schedule controls both the momentum averaging and the delay threshold, removing the need to tune the threshold from problem parameters.
Overall, our work connects these three lines by extending delay-thresholding from Euclidean \sgd to general LMO-based momentum methods and by translating the resulting iteration bounds into time complexity guarantees.
\section{Preliminaries and setup}\label{sec:preliminaries}
This section fixes the optimization model, the LMO geometry, the distributed computation model, and the assumptions used in the analysis.
\subsection{Optimization model and LMO geometry}\label{sec:lmo_geometry}
We study the nonconvex stochastic optimization problem
\begin{equation*}
    \minimize_{x \in \R^d}
    \left\{ f(x) \coloneqq \ExpSub{\xi \sim \cD}{f(x;\xi)} \right\} ,
\end{equation*}
where $x \in \R^d$ is the model parameter, $\cD$ is the data distribution, $\xi \sim \cD$ is a data sample, and $f(x;\xi)$ is the corresponding loss.
Let $\norm{\cdot}$ be a norm on $\R^d$.
Following \citet{pethick2025scion}, the update direction is defined through a linear minimization oracle (LMO) over the unit ball of this norm:
\begin{equation*}
    \lmo(y) \in \argmin_{u~:~\norm{u} \le 1} \inp{y}{u} \quad \text{for } y\ne 0,
    \qquad
    \lmo(0) \eqdef 0.
\end{equation*}
When the minimizer is not unique, any minimizer may be chosen.
This definition implies that the magnitude of the update is always bounded because $\norm{\lmo(y)}\le 1$.
\subsection{Function and oracle assumptions}\label{sec:assumptions}
We use the following assumptions throughout the convergence analysis.
\begin{assumption}[Lower boundedness]\label{ass:lower_bound}
  There exists $f^* > -\infty$ such that $f(x) \geq f^*$ for all $x \in \R^d$.
  We define $\Delta_0 \eqdef f(x_0) - f^*,$ where $x_0$ is the starting point of the optimization methods.
\end{assumption}
\begin{assumption}[Unbiased stochastic gradients with bounded variance]\label{ass:unbiased_bounded_variance}
  For every $\xi$, the function $f(x; \xi)$ is differentiable with respect to its first argument $x$.
  Moreover, the stochastic gradients are unbiased and have bounded variance $\sigma^2 \geq 0$, that is,
  \begin{align*}
     &\ExpSub{\xi \sim \cD}{\nabla f(x;\xi)} = \nabla f(x),
        \quad \forall x \in \R^d , \\ 
     &\ExpSub{\xi \sim \cD}{\sqnorm{\nabla f(x;\xi) - \nabla f(x)}_2} \leq \sigma^2,
        \quad \forall x \in \R^d .
\end{align*}
\end{assumption}
\begin{assumption}[$(L_0, L_1)$--smoothness]\label{ass:gen_smooth}
  There exist constants $L_0,L_1 \ge 0$ such that, for all $x, y \in \R^d$,
  \begin{equation*}
    \norm{\nabla f(x) - \nabla f(y)}_*
      \le
      \sup_{\theta \in [0,1]}
      \left(
        L_0 + L_1 \norm{\nabla f(\theta x + (1 - \theta)y)}_*
      \right)
      \norm{x - y}.
  \end{equation*}
\end{assumption}
Here $\norm{\cdot}_*$ denotes the dual norm associated with $\norm{\cdot}$, i.e.
$$
    \norm{y}_* \eqdef \sup_{u~:~\norm{u} \le 1} \inp{y}{u}.
$$
\Cref{ass:gen_smooth} reduces to the standard $L_0$-smoothness condition with respect to $\norm{\cdot}$ when $L_1=0$.
When $L_1>0$, the effective local smoothness constant is allowed to grow with the gradient norm, which motivates generalized smoothness models for objectives arising in modern machine learning \citep{chen2023generalized,pethick2025generalized}.

\Cref{ass:unbiased_bounded_variance} controls stochastic-gradient noise in the Euclidean norm, whereas \Cref{ass:gen_smooth} is stated in the norm induced by $\norm{\cdot}$ and its dual.
We denote by
\begin{equation*}
    \rho \eqdef \sup_{z \ne 0} \frac{\norm{z}_*}{\norm{z}_2}
\end{equation*}
the corresponding norm-equivalence constant, so that $\norm{z}_* \le \rho\norm{z}_2$ for all $z \in \R^d$.
\subsection{Distributed computation model}\label{sec:distributed_computation_model}
We consider a distributed setting with $n$ workers.
We write $[n]\coloneqq \{1,\dots,n\}$.
Each worker $i$ has access to the same data distribution~$\cD$.
To model heterogeneous worker speeds, we use the fixed computation model of \citet{mishchenko2022asynchronous}.
We extend the analysis to arbitrarily changing computation times in \Cref{sec:arbitrary_time}, while keeping the fixed model in the main text for a clearer presentation.
\begin{assumption}[Fixed computation model]
  \label{assump:fixed_time}
  For each worker $i \in [n]$, computing a stochastic gradient $\nabla f(x;\xi)$ takes $\tau_i > 0$ seconds.
  Without loss of generality, we assume $\tau_1 \le \tau_2 \le \cdots \le \tau_n$.
  For $m\in[n]$, define
  \begin{equation*}
    H_m \coloneqq \left(\frac{1}{m}\sum_{i=1}^{m}\frac{1}{\tau_i}\right)^{-1},
  \end{equation*}
  the harmonic mean of the first $m$ computation times $\tau_1,\dots,\tau_m$.
\end{assumption}
Since $\tau_1\le\cdots\le\tau_n$, the sequence $H_m$ is non-decreasing in $m$.

Throughout this paper, we assume zero-latency communication between the workers and the server, adopting the standard modeling assumption used in prior asynchronous optimization literature \citep{mishchenko2022asynchronous, koloskova2022sharper, tyurin2023optimal, maranjyan2025ringmaster}.
While communication costs are undeniably important in practice, modeling them properly requires the adoption of specific mechanisms like compression or local steps.
Incorporating those techniques introduces a separate set of theoretical challenges that are largely orthogonal to our primary objective: understanding asynchronous dynamics in the presence of heterogeneous computation speeds.
\section{\algntitle method}\label{sec:method}
We now describe the asynchronous LMO-based method.
A central server maintains the iterate and momentum, while $n$ workers repeatedly compute stochastic gradients.
Whenever worker $i$ finishes, it returns a gradient computed at the iterate it last received from the server.
If the server accepts that gradient, it first updates the momentum and then applies the LMO direction:
\begin{equation}\label{eq:algo_update}
    \begin{aligned}
        m_{k+1} &=
        \begin{cases}
            (1-\alpha_0)m_0 + \alpha_{\mathrm{init}} g_0, & k = 0, \\
            (1-\alpha_k)m_k + \alpha_k g_k, & k \ge 1,
        \end{cases} \\ 
        x_{k+1} &= x_k + \eta_k \lmo(m_{k+1}).
    \end{aligned}
\end{equation}
Here $g_k = \nabla f\big(x_{k-\delta_k}; \xi_{k-\delta_k}^{i_k}\big)$ is the gradient that arrives from worker $i_k$ at server iteration $k$, $\delta_k$ is its delay, $\alpha_k$ is the momentum parameter, and $\eta_k$ is the stepsize.
Because $g_k$ may have been computed at the stale iterate $x_{k-\delta_k}$ rather than at $x_k$, large delays can degrade the update.
To control this effect, we use the delay-thresholding principle of \ringmaster \citep{maranjyan2025ringmaster}: the server accepts an arriving gradient only if $\delta_k < R_k$.
The threshold sequence $R_k$ controls the freshness-throughput tradeoff: smaller thresholds are more conservative, while larger thresholds allow more asynchronous parallelism.
Allowing $R_k$ to vary with $k$ will be important for the parameter-agnostic variant in \Cref{sec:param_agnostic}.
\Cref{alg} gives the full description of the method.
\begin{algorithm}[htb]
  \caption{\algn}
  \label{alg}
  \begin{algorithmic}[1]
      \STATE \textbf{Input:} initial iterate $x_0 \in \R^d$; step-size sequence $\eta_k > 0$; delay-threshold sequence $R_k > 0$; momentum parameter sequence $\alpha_k \in (0,1]$ with $\alpha_0 = 1$ and $\alpha_{\mathrm{init}} \in (0,1]$; number of model updates $K \in \N$
      \STATE Each worker $i \in [n]$ starts computing $\nabla f(x_0; \xi_0^i)$ in parallel
      \STATE $k \gets 0$
      \WHILE{$k < K$}
          \STATE Receive $g_k = \nabla f\bigl(x_{k-\delta_k}; \xi_{k-\delta_k}^{i_k}\bigr)$ from worker $i_k$
            \algcomment{$\delta_k$ is the delay of the arriving gradient}
          \IF{$\delta_k < R_k$}
              \STATE $m_{k+1} \gets \begin{cases} (1-\alpha_0)m_0 + \alpha_{\mathrm{init}} g_0, & k = 0, \\ (1-\alpha_k)m_k + \alpha_k g_k, & k \ge 1. \end{cases}$ \algcomment{$m_0$ need not be set since $\alpha_0=1$}
              \STATE $x_{k+1} \gets x_k + \eta_k \lmo(m_{k+1})$
              \STATE The server sends $x_{k+1}$ to worker $i_k$
              \STATE Worker $i_k$ starts computing $\nabla f\bigl(x_{k+1}; \xi_{k+1}^{i_k}\bigr)$
              \STATE $k \gets k + 1$
          \ELSE
              \STATE Discard the stale gradient $g_k$
              \STATE The server sends $x_k$ to worker $i_k$
              \STATE Worker $i_k$ starts computing $\nabla f\bigl(x_k; \xi_k^{i_k}\bigr)$
          \ENDIF
      \ENDWHILE
  \end{algorithmic}
\end{algorithm}

Algorithm~\ref{alg} discards stale gradients only after they arrive, which may appear to waste computation.
In practice, the server can often determine in advance that an in-flight computation will miss the current threshold and interrupt that worker early instead of waiting for the stale gradient to finish; this computation-stopping implementation was also proposed in \ringmaster \citep{maranjyan2025ringmaster}.
One can go further when worker runtimes are random rather than fixed \citep{maranjyan2025mindflayer}: the runtime distributions can be learned online and used to make better assignment or interruption decisions on the fly \citep{maranjyan2025ata}.
These systems-level refinements are complementary to our main objective, so we keep the simpler fixed-time model with delay-based acceptance and rejection in the main text.

\paragraph{Special cases.}
If every accepted gradient is fresh, then \eqref{eq:algo_update} reduces to the synchronous LMO-based momentum method \citep{pethick2025scion, kovalev2025understandinggradientorthogonalizationdeep}.
In \Cref{alg}, this is enforced by setting $R_k = 1$ for all $k$, which means that only gradients with $\delta_k = 0$ are accepted.

If the LMO is defined with respect to the Euclidean norm, then the update reduces to normalized \sgd with momentum \citep{hazan2015beyondconvexity,cutkosky2020momentum}:
\begin{equation*}
    x_{k+1} = x_k - \eta_k \frac{m_{k+1}}{\norm{m_{k+1}}_2}
\end{equation*}
instead of \eqref{eq:algo_update}.
In this case, \Cref{alg} reduces to the method studied by \citet{wu2026optimalasynchronous}.
Their analysis, however, is restricted to the Euclidean setting and classical smoothness ($L_1 = 0$), and it does not cover the parameter-agnostic threshold schedule studied here.

\paragraph{Other induced methods.}
Because the method is defined through an LMO, changing the underlying norm immediately yields asynchronous variants of other normalized optimizers, as discussed by \citet{pethick2025scion}.
For example, choosing the spectral norm recovers \muon \citep{jordan2024muon}, while choosing the max norm $\norm{\cdot}_\infty$ recovers \algname{SignSGD} or \algname{Signum} \citep{bernstein2024oldoptimizernewnorm}.
See \citet{pethick2025scion} for a more comprehensive discussion of which norm choices recover which methods.

%
\section{Theoretical analysis}\label{sec:theory}
We begin with the fixed-parameter regime before turning to the parameter-agnostic case.
\subsection{Fixed-parameter version}\label{sec:fixed_parameter}
We first state the corresponding iteration-complexity guarantee of \algn (Algorithm~\ref{alg}).
%
%
\begin{theorem}[Fixed-parameter iteration complexity; proof in \Cref{proof:iteration_complexity}]\label{thm:iteration_complexity}
    Suppose Assumptions~\ref{ass:lower_bound}--\ref{ass:gen_smooth} hold.
    Run \Cref{alg} with $\alpha_0=1$.
    Set the momentum parameter to the constant $\alpha$ for all $k\ge 1$,
    and use the constant delay threshold $R$ for all $k\ge 0$, where
    \begin{equation*}
        \alpha_k = \alpha_{\mathrm{init}} = \alpha = \min\left\{1, \frac{\Delta_0^{\nicefrac{1}{2}}L_0^{\nicefrac{1}{2}}}{\rho\sigma K^{\nicefrac{1}{2}}}\right\}, \quad \forall k \geq 1,
        \quad
        R_k = R = \frac{1}{\alpha}, \quad \forall k \ge 0,
    \end{equation*}
    with the convention that, when $\sigma=0$, we take $\alpha=1$,
    and choose the constant stepsize as follows:
    \begin{itemize}
        \item If $L_1=0$, take
        \begin{equation*}
            \eta_k = \eta = \min\left\{\sqrt{\frac{\Delta_0}{L_0 K}}, \frac{\Delta_0^{\nicefrac{3}{4}}}{ L_0^{\nicefrac{1}{4}}(\rho\sigma)^{\nicefrac{1}{2}} K^{\nicefrac{3}{4}}}\right\}, \quad \forall k \ge 0.
        \end{equation*}
        \item If $L_1>0$, take
        \begin{equation*}
            \eta_k = \eta = \min\left\{\sqrt{\frac{\Delta_0}{L_0 K}}, \frac{1}{8L_1}, \frac{\Delta_0^{\nicefrac{1}{2}}L_0^{\nicefrac{1}{2}}}{8L_1\rho\sigma K^{\nicefrac{1}{2}}},  \frac{\Delta_0^{\nicefrac{3}{4}}}{ L_0^{\nicefrac{1}{4}}(\rho\sigma)^{\nicefrac{1}{2}} K^{\nicefrac{3}{4}}}\right\}, \quad \forall k \ge 0.
        \end{equation*}
    \end{itemize}
    Then, after $K \ge 1$ iterations, the iterates of \Cref{alg} satisfy
    \begin{equation*}
        \min_{k \in \{0,\dots, K-1\}} \Exp{\|\nabla f(x_k)\|_*} \leq \cO\left(\sqrt{\frac{L_0\Delta_0}{K}} + \frac{L_1\Delta_0}{K} + \frac{L_1 \Delta_0^{\nicefrac{1}{2}} \rho\sigma}{L_0^{\nicefrac{1}{2}}K^{\nicefrac{1}{2}}} + \frac{L_0^{\nicefrac{1}{4}}\Delta_0^{\nicefrac{1}{4}}(\rho\sigma)^{\nicefrac{1}{2}}}{K^{\nicefrac{1}{4}}} + \frac{\rho\sigma}{K^{\nicefrac{1}{2}}}
    \right).
    \end{equation*}
    In the case $L_1=0$, the terms containing $L_1$ in the preceding bound are zero.
    Consequently, the number of iterations needed to reach an $\varepsilon$-stationary point satisfies
    \begin{equation*}
        K_{\varepsilon}
        = \cO\left(
            \frac{L_0\Delta_0 }{\varepsilon^2}
            + \frac{L_1\Delta_0 }{\varepsilon}
            + \frac{L_0\Delta_0(\rho\sigma)^2}{\varepsilon^4}
            + \frac{L_1^2\Delta_0(\rho\sigma)^2}{L_0\varepsilon^2}
            + \frac{(\rho\sigma)^2}{\varepsilon^2}
            \right).
    \end{equation*}
    Again, when $L_1=0$, the terms containing $L_1$ are omitted.
\end{theorem}
\paragraph{Discussion}
In the standard smooth Euclidean case ($\rho = 1$ and $L_1 = 0$), our rate matches the lower bounds proved by \citet{arjevani2023lower} for first-order stochastic smooth non-convex optimization.

When $L_1 > 0$, the deterministic part of our complexity matches the best known rates under generalized smoothness and agrees with lower bounds for gradient descent \citep{zhang2020gradient,crawshaw2022robustness}\footnote{These comparisons are limited to algorithm-specific lower bounds, since lower bounds for the full class of generalized $(L_0,L_1)$-smooth functions remain open.}.
\algname{GGNC} \citep{pethick2025generalized} has the same complexity in the non-stochastic case. In the stochastic case, however, our bounds are tighter because of the optimal parameter selection.
\algname{Gluon} \citep{riabinin2025gluon} has the same deterministic rate when each layer shares the same generalized smoothness constants.
However, their analysis assumes asymmetric generalized smoothness, which is less general than symmetric generalized smoothness \citep{chen2023generalized}.
Moreover, our stochastic complexity bounds are tighter than those of \algname{GGNC} and \algname{Gluon}.

In the $\ell_1$-norm case, our bounds match those of \algname{SignSGD} \citep{crawshaw2022robustness}, although that work assumes coordinate-wise generalized smoothness.
\citet{khirirat2025better} proved better rates $\cO(\varepsilon^{-3})$ for \algname{LMO-MVR} and \algname{LMO-SOM}, but their algorithms assume access either to multiple stochastic gradient queries with shared randomness or to Hessian-vector products.
These assumptions are restrictive and, especially in the former case, difficult to adapt to the asynchronous setting.
Finally, in the bounded-variance case, our method recovers the rate of \algname{RANSGD-M} from \citet{wu2026optimalasynchronous}, since in the Euclidean setting our method reduces to \algname{RANSGD-M}.
%
%
\paragraph{Time complexity}
The time complexity under \Cref{assump:fixed_time} follows directly from \citet{maranjyan2025ringmaster}.
Their \ringmaster method and our \algn differ only in how they form updates; this changes the iteration complexity, but the same time complexity analysis applies.
More concretely, we need the following lemma.
\begin{lemma}[Duration of $R$ updates {\citep[Lemma 4.1]{maranjyan2025ringmaster}}]\label{lem:time_R}
    Let the delay threshold of \Cref{alg} be fixed as $R_k \equiv R$.
    Under \Cref{assump:fixed_time}, the time needed to complete any $R$ consecutive updates is at most
    \begin{equation*}
        t(R) \coloneqq 2 \min_{m \in [n]} H_m \left( 1 + \frac{R}{m} \right).
    \end{equation*}
\end{lemma}
With this lemma we can get the time complexity of our method with the following theorem.

\begin{theorem}[Fixed-parameter total time complexity; proof in \Cref{proof:time_complexity}]\label{thm:time_complexity}
    Under the assumptions and parameter choices of \Cref{thm:iteration_complexity}, under \Cref{assump:fixed_time}, the time complexity of reaching an $\varepsilon$-stationary point satisfies
    \begin{align*}
        T_\varepsilon
        =
        \mathcal{O}\left(
            \min_{m \in [n]} H_m
            \left(
                \frac{L_0\Delta_0}{\varepsilon^2}
                + \frac{L_1\Delta_0}{\varepsilon} 
                + \frac{\sqrt{L_0\Delta_0}}{\varepsilon}
                + \frac{L_0\Delta_0(\rho\sigma)^2}{m\varepsilon^4}
                + \frac{L_1^2\Delta_0(\rho\sigma)^2}{mL_0\varepsilon^2}
                + \frac{(\rho\sigma)^2}{m\varepsilon^2}
            \right)
        \right).
    \end{align*}
\end{theorem}
\begin{proof}[Proof sketch]
    Let $K_\varepsilon$ denote the iteration bound from \Cref{thm:iteration_complexity}, and recall that in the fixed-parameter regime the delay threshold is constant with $R_k \equiv R = \nicefrac{1}{\alpha}$. We partition the first $K_\varepsilon$ updates into $\lceil \nicefrac{K_\varepsilon}{R} \rceil$ consecutive blocks of length at most $R$ and apply \Cref{lem:time_R} to each block. Therefore,
    $
        T_\varepsilon
        \le
        \left\lceil \nicefrac{K_\varepsilon}{R} \right\rceil t(R).
    $
    Substituting the choices of $\alpha$ and the bound on $K_\varepsilon$ from \Cref{thm:iteration_complexity} yields the claimed time complexity bound.
\end{proof}
\paragraph{Discussion}
In the standard smooth Euclidean case ($\rho = 1$ and $L_1 = 0$), our bound reduces to the optimal time complexity scaling identified by \citet{tyurin2023optimal}\footnote{The proved lower bounds hold in the case when $0 < \varepsilon \leq c' \sqrt{L_0\Delta_0}$ for some constant $c' > 1$, and matches with our bound from \Cref{thm:time_complexity}.}.
In particular, it recovers the same optimal time complexity as \ringmaster \citep{maranjyan2025ringmaster} and \algname{RANSGD-M} \citep{wu2026optimalasynchronous}.
Thus, extending delay thresholding from normalized \sgd to general LMO-based updates does not incur any loss in the classical Euclidean regime.
\subsection{Parameter-agnostic case}\label{sec:param_agnostic}
We now turn to the parameter-agnostic case and begin with its iteration complexity.
%
%
\begin{theorem}[Parameter-agnostic iteration complexity; proof in \Cref{proof:iteration_complexity_param_agnostic}]\label{thm:iteration_complexity_param_agnostic}
    Suppose Assumptions~\ref{ass:lower_bound}--\ref{ass:gen_smooth} hold.
    Run \Cref{alg} with $R_k=\nicefrac{1}{\alpha_k}$, $\alpha_{\mathrm{init}} = \alpha_0=1$, and
    $\alpha_k=k^{-\nicefrac{1}{2}}$ for $k\ge 1$.
    In the first two cases, let $\eta>0$ be any constant and define
    \begin{equation*}
        \Psi(L_0,L_1)
        \coloneqq
        e^{L_1^2\eta^2}\nicefrac{\Delta_0}{\eta}
        + \rho\sigma
        + e^{L_1^2\eta^2}L_0\eta .
    \end{equation*}
    Choose the stepsizes as follows, and let $\Psi$ denote the corresponding value in the final bound:
    \begin{itemize}
        \item If $L_1=0$, take $\eta_k=\frac{\eta}{(k+1)^{\nicefrac{3}{4}}}$ and
        $\Psi=\Psi(L_0,0)=\nicefrac{\Delta_0}{\eta}+\rho\sigma+L_0\eta$.
        \item If $L_1>0$, take $\eta_k=\frac{\eta}{17(k+1)^{\nicefrac{3}{4}}}$ and
        $\Psi=\Psi(L_0,L_1)$.
        \item If $L_1>0$ is known, take $\eta_k=\frac{1}{17L_1(k+1)^{\nicefrac{3}{4}}}$ and
        $\Psi=L_1\Delta_0+\rho\sigma+\nicefrac{L_0}{L_1}$, which is the preceding envelope with $\eta=\nicefrac{1}{L_1}$ up to universal constants.
    \end{itemize}
    Then, after $K\ge 1$ iterations,
    \begin{equation*}
        \min_{k\in\{0,\dots,K-1\}}\Exp{\norm{\nabla f(x_k)}_*}
        \le
        \cO\left(\frac{\Psi\log K}{K^{\nicefrac{1}{4}}}\right).
    \end{equation*}
    Consequently, the number of iterations needed to reach an $\varepsilon$-stationary point satisfies
    \begin{equation*}
        K_{\varepsilon}
        =
        \cO\left(
            \frac{\Psi^4}{\varepsilon^4}
            \left(\log\frac{1}{\varepsilon}\right)^4
        \right).
    \end{equation*}
\end{theorem}
\paragraph{Discussion}
In the Euclidean case, our parameter-agnostic bound matches that of \algname{NSGD-M} from \citet{hubler2024parameter_agnostic}, while working under symmetric generalized smoothness rather than their asymmetric variant; for twice-differentiable objectives, the two notions are equivalent.
When $L_1$ is unknown, our bound contains the factor $e^{L_1^2\eta^2}$, which is unavoidable and is predicted by the lower bounds of \citet{hubler2024parameter_agnostic}.
The parameter-agnostic analysis of \algname{Gluon} \citep{riabinin2025gluon} does not contain this factor, but it relies on more restrictive assumptions.

Compared with existing LMO-based results, our guarantee does not require prior knowledge of either $K$ or $L_1$.
In particular, \citet{khirirat2025better} analyze several LMO-based momentum variants, but their parameter choices depend on the time horizon and generalized smoothness constants.
Finally, relative to \citet{wu2026optimalasynchronous}, we provide, to the best of our knowledge, the first parameter-agnostic guarantee for an asynchronous method under generalized smoothness, and more broadly the first convergence guarantee for asynchronous optimization in this setting.
%
%
\paragraph{Time complexity}
We now analyze the time complexity of \Cref{alg} in the parameter-agnostic case, where the threshold $R_k$ changes over time according to the square-root schedule in the theorem above.
This case is slightly more difficult than the fixed-threshold case below.
\begin{lemma}[Time complexity for a square-root delay threshold; proof in \Cref{proof:sqrt_threshold}]\label{lemma:time_sqrt}
  Consider \Cref{alg} with delay-threshold sequence given in \Cref{thm:iteration_complexity_param_agnostic}, i.e. $R_0=1$ and $R_k = \sqrt{k}$ for $k \ge 1$.
  Under \Cref{assump:fixed_time}, for any $K\ge 1$, the time needed to complete the first $K$ iterations satisfies
  \begin{equation*}
        T(K) = \mathcal{O}\left( \min_{m \in [n]} H_m \left( \sqrt{K} + \frac{K}{m} \right) \right).
  \end{equation*}
\end{lemma}
\begin{proof}[Proof sketch]
  We group iterations by the integer threshold level
  $\mathcal{K}_r \coloneqq \{k\in\{0,\dots,K-1\} \mid \lfloor R_k \rfloor = r\}$.
  Since $R_k$ is non-decreasing, each $\mathcal{K}_r$ is a contiguous block.
  The fixed-threshold bound from \Cref{lem:time_R} can therefore be applied blockwise, after splitting $\mathcal{K}_r$ into groups of at most $r$ updates.
  For the square-root rule $R_0=1$ and $R_k=\sqrt{k}$ for $k\ge 1$, we have $|\mathcal{K}_r|\le 2r+2$, so each threshold level contributes only a constant number of such groups.
  Summing over $r\le \lfloor R_{K-1}\rfloor\le \sqrt{K}$ gives the stated bound.
\end{proof}
Combining \Cref{thm:iteration_complexity_param_agnostic} with \Cref{lemma:time_sqrt} yields the total time complexity for reaching an $\varepsilon$-stationary point.
\begin{theorem}[Total time complexity]\label{thm:time_complexity_param_agnostic}
  Under the assumptions and parameter choices of \Cref{thm:iteration_complexity_param_agnostic}, let $K_\varepsilon$ be the number of iterations needed to reach an $\varepsilon$-stationary point.
  Then, under \Cref{assump:fixed_time}, the corresponding time complexity
  $T_\varepsilon \coloneqq T(K_\varepsilon)$ satisfies
  \begin{equation*}
      T_\varepsilon
      =
      \widetilde{\mathcal{O}}\left(
          \min_{m \in [n]} H_m \left(
              \frac{\Psi^2}{\varepsilon^2}
              + \frac{\Psi^4}{m\varepsilon^4}
          \right)
      \right),
  \end{equation*}
  where $H_m$ is the harmonic mean of the first $m$ computation times $\tau_1,\dots,\tau_m$, $\Psi$ is the corresponding quantity from \Cref{thm:iteration_complexity_param_agnostic}, and $\widetilde{\mathcal{O}}(\cdot)$ hides logarithmic factors.
\end{theorem}
\begin{proof}
  By \Cref{thm:iteration_complexity_param_agnostic},
  $
      \sqrt{K_\varepsilon} = \widetilde{\mathcal{O}}\left(\frac{\Psi^2}{\varepsilon^2}\right).
  $
  Applying \Cref{lemma:time_sqrt} with $K = K_\varepsilon$ gives
  \begin{equation*}
    T_\varepsilon
    =
    \mathcal{O}\left(
        \min_{m \in [n]} H_m \left(\sqrt{K_\varepsilon} + \frac{K_\varepsilon}{m}\right)
    \right)
    =
    \widetilde{\mathcal{O}}\left(
        \min_{m \in [n]} H_m \left(
            \frac{\Psi^2}{\varepsilon^2}
            + \frac{\Psi^4}{m\varepsilon^4}
        \right)
    \right).
  \end{equation*}
\end{proof}
\paragraph{Discussion.}
Up to logarithmic factors, this bound has the same worker-heterogeneity dependence as \Cref{thm:time_complexity}, namely
$
    \min_{m \in [n]} H_m \left(\cdot + \frac{\cdot}{m}\right),
$
so it preserves the delay-thresholding benefit of adapting to the faster workers.
Its advantage is parameter agnosticism: the square-root schedule attains this guarantee without knowing $K$ or $L_1$.
To the best of our knowledge, no existing asynchronous method provides a comparable parameter-agnostic time complexity guarantee.
\section{Experiments}\label{sec:experiments}

\paragraph{Reproducibility.}
The implementation and code for reproducing our experiments are available at
\href{https://github.com/vectozavr/ringmaster-lmo}{\texttt{github.com/vectozavr/ringmaster-lmo}}.

\paragraph{Overview.}
We test the main empirical prediction of the theory: delay thresholding should be most useful when asynchronous workers have heterogeneous runtimes.
Choosing the spectral-norm LMO in \Cref{alg} recovers \muon, so our two tested methods are \ringmastermuon and its parameter-agnostic variant \ringmastermuonagnostic.
We compare them with two asynchronous \muon baselines, \rennalamuon and \muondelayadaptive, on a stochastic quadratic benchmark and on NanoChat language-model pretraining.

\paragraph{Methods and tuning.}
For all fixed-parameter methods, we use the same internal \muon configuration: momentum coefficient $\beta=0.95$, five Newton--Schulz iterations, and Nesterov lookahead.
Thus the tuning only concerns asynchronous control parameters such as the stepsize scale $\eta$, delay threshold, or baseline-specific refresh parameter.
For \ringmastermuonagnostic, we use the schedule motivated by \Cref{thm:iteration_complexity_param_agnostic},
\begin{equation*}
    \eta_k=\frac{\eta}{(k+1)^{3/4}},\qquad
    \alpha_0=1,\quad \alpha_k=k^{-1/2}\quad(k\ge 1),\qquad
    R_k=\max\{1,\lfloor 1/\alpha_k \rfloor\},
\end{equation*}
so the only tuned parameter is the scale $\eta$.
Throughout the experimental grids, $\eta$ denotes the tuned stepsize scale; for delay-adaptive baselines, it is the nominal scale before applying the method-specific delay correction.
All methods are tuned independently for each benchmark and delay regime under the same simulated runtime budget.

\paragraph{Runtime model and compute resources.}
All experiments use a server--worker simulator instead of a physical multi-GPU deployment.
The simulator follows the asynchronous protocol in \Cref{alg}: a worker receives the current model, computes one stochastic gradient, returns after its assigned runtime, and is then either accepted or rejected according to the method's asynchronous rule.
The horizontal axis in every plot is this simulated runtime, not the host machine's elapsed wall-clock time.
Indexing workers by increasing deterministic runtime as $i=0,\dots,n-1$, we use the following three delay profiles:
\begin{equation}
    \label{eq:experiment_delay_profiles}
    g_i^{\mathrm{sim}}=1,\qquad
    g_i^{\mathrm{sub}}=1+\sqrt{i},\qquad
    g_i^{\mathrm{lin}}=1+i.
\end{equation}
The first profile is nearly homogeneous, the second introduces moderate heterogeneity, and the third creates a strong straggler tail.
We add a small independent perturbation to worker runtimes so that the similar-delay regime is not exactly tied across workers.
The synthetic quadratic experiments were run on a MacBook Pro with an Apple M1 Max chip and 64GB RAM using Python 3.12.
The NanoChat gradients were computed on one NVIDIA A100-PCIE-40GB GPU using \texttt{bfloat16} autocast and \texttt{torch.compile}; the multi-worker dynamics were simulated on top of these measured gradient times.

\paragraph{Synthetic quadratic benchmark.}
Our synthetic benchmark is the standard worst-case Nesterov tridiagonal quadratic in dimension $d=1729$:
\begin{equation*}
    f(x)=\frac{1}{2}x^\top A x-b^\top x,\qquad
    A=\frac{1}{4}\operatorname{tridiag}(-1,2,-1),\qquad
    b=-\frac{1}{4}e_1.
\end{equation*}
We initialize at $x_0=\sqrt{d}\,e_1$, compute the exact minimizer by solving $Ax=b$, and report the objective gap $f(x^t)-f^\star$ as a function of simulated runtime.
The stochastic oracle is
\begin{equation*}
    g(x,\xi)=\nabla f(x)+\xi\mathbf{1},
\end{equation*}
where $\xi\sim\mathcal{N}(0,0.01^2)$ is a scalar Gaussian perturbation shared across all coordinates.
For this benchmark, worker $i$ has deterministic base runtime $g_i$, where $g_i$ is chosen from the three profiles in \eqref{eq:experiment_delay_profiles}.
We simulate $n=6174$ workers and tune each method over a runtime horizon of $2000$ simulated seconds, with one trial per hyperparameter setting.
The search grids are
\begin{itemize}
    \item \ringmastermuon: $\eta \in \{5^{-6}, 5^{-5}, \dots, 5^1\}$ and $R \in \{1, 2, 4, 6, 8, 16, 32\}$;
    \item \ringmastermuonagnostic: $\eta \in \{5^{-4}, 5^{-3}, \dots, 5^3\}$;
    \item \rennalamuon: $\eta \in \{5^{-6}, 5^{-5}, \dots, 5^1\}$ and $B \in \{1, 2, 4, 6, 8, 16, 32\}$;
    \item \muondelayadaptive: $\eta \in \{5^{-8}, 5^{-7}, \dots, 5^{-1}\}$.
\end{itemize}

\begin{figure*}[t]
    \centering
    \begin{minipage}[t]{0.32\textwidth}
        \centering
        \includegraphics[width=\linewidth]{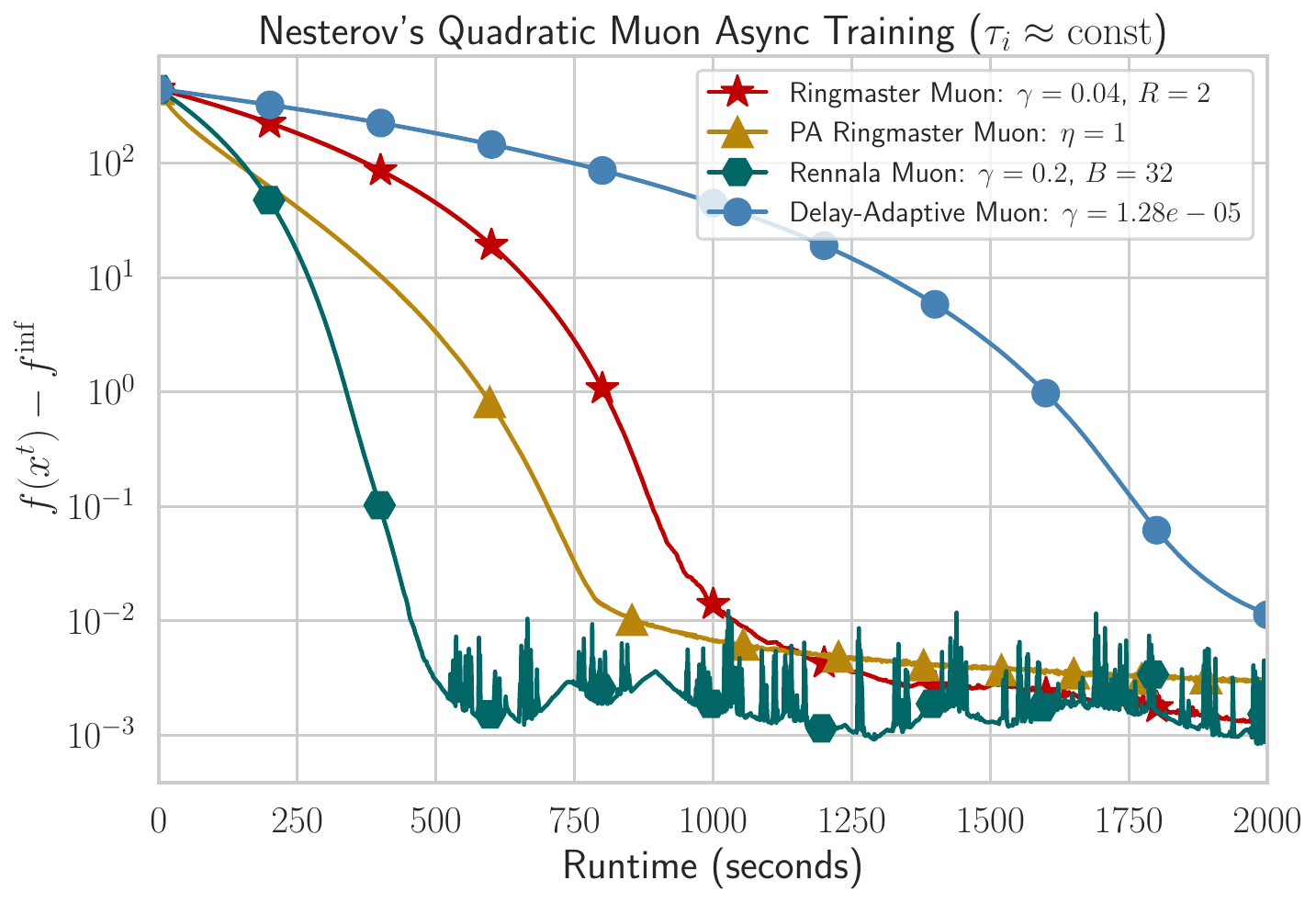}
        \small (a) Nearly homogeneous delays
    \end{minipage}\hfill
    \begin{minipage}[t]{0.32\textwidth}
        \centering
        \includegraphics[width=\linewidth]{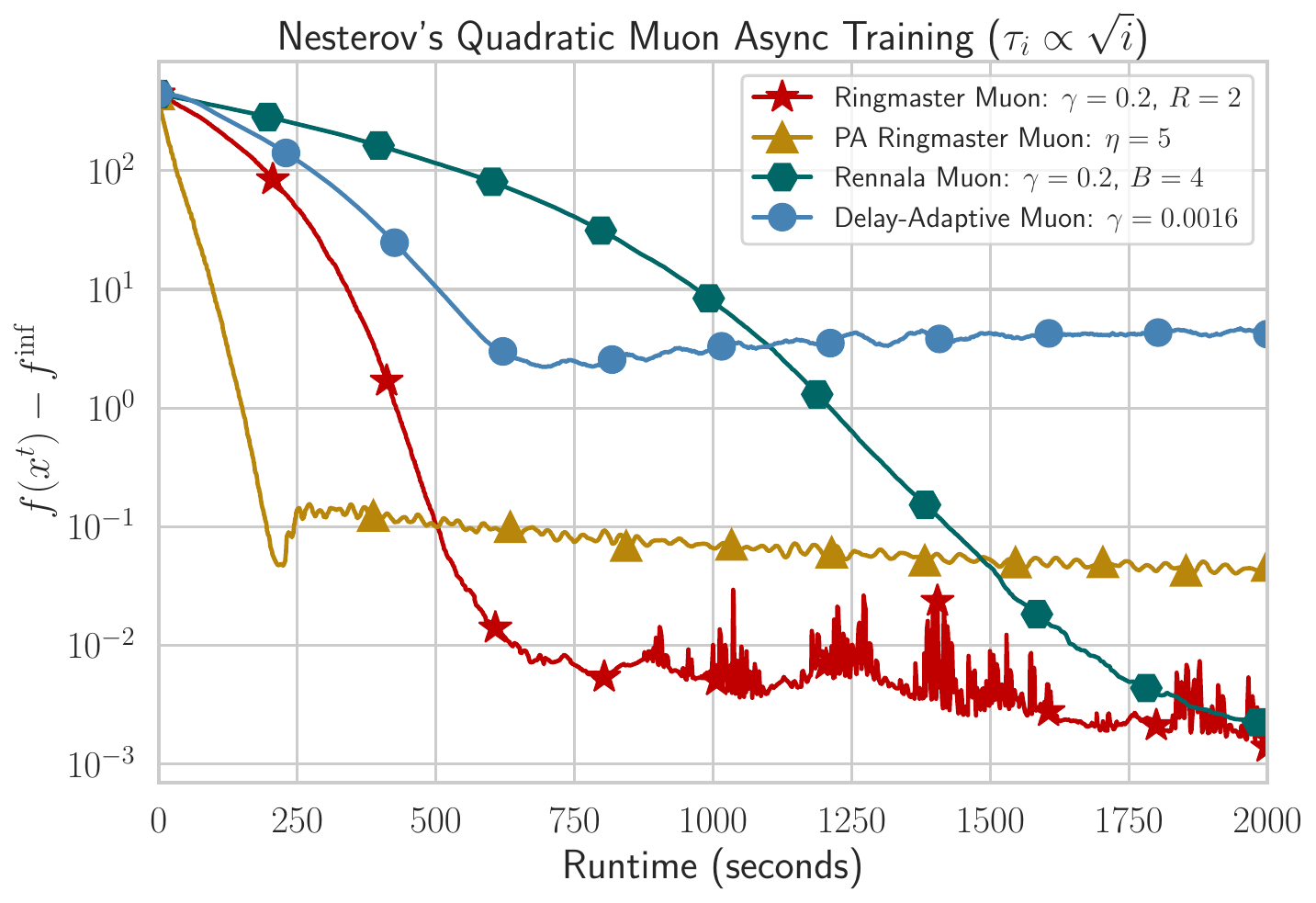}
        \small (b) Sublinear delays
    \end{minipage}\hfill
    \begin{minipage}[t]{0.32\textwidth}
        \centering
        \includegraphics[width=\linewidth]{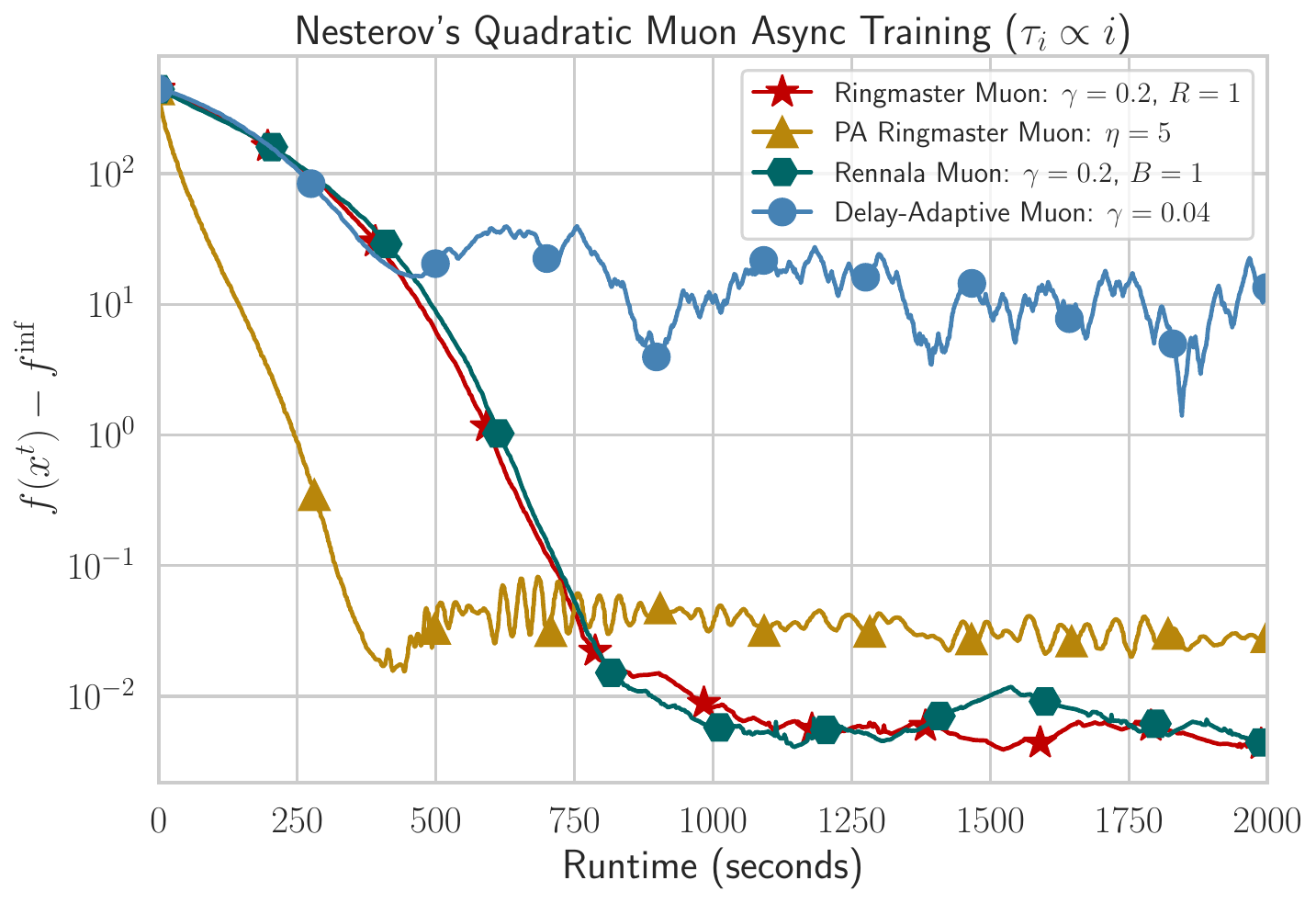}
        \small (c) Linear delays
    \end{minipage}
    \caption[Quadratic comparison under three delay regimes]{Comparison on the stochastic tridiagonal quadratic objective under similar, sublinear, and linear worker-delay regimes. The advantage of \ringmastermuon appears as the delay heterogeneity becomes stronger.}
    \label{fig:quadratic_experiments_placeholders}
\end{figure*}

\Cref{fig:quadratic_experiments_placeholders} shows the expected tradeoff.
In the similar-delay regime, where workers are nearly homogeneous, \ringmastermuon slightly underperforms the strongest baselines because discarding updates provides little benefit and can waste gradient evaluations.
As heterogeneity increases, the value of thresholding becomes clearer: \ringmastermuon performs best in the sublinear-delay regime and remains highly competitive in the linear-delay regime, where it matches \rennalamuon.

\paragraph{NanoChat benchmark.}
Our language-modeling experiments use an architecture-faithful implementation of NanoChat \citep{nanochat}.
The model has context length $2048$, window pattern \texttt{SSSL}, vocabulary size $8192$, $6$ transformer blocks, hidden size $192$, and grouped-query attention with $3$ query heads and $3$ key-value heads.
Each block uses RMS normalization, rotary positional embeddings, value embeddings with gating, learned scalar coefficients for residual connections, and the MLP nonlinearity $x\mapsto\operatorname{ReLU}(x)^2$.
The tokenizer is trained once using an $8192$-wordpiece BPE vocabulary on the Karpathy \texttt{climbmix-400b-shuffle} shards while preserving the GPT-style regex split pattern of the reference implementation.
The data pipeline prepends a BOS token to each document and packs documents into length-$2049$ streams using a best-fit heuristic, so each stochastic gradient processes fully packed sequences of length $2048$.

Optimization uses device batch size $4$.
The optimization variable concatenates all trainable tensors into a single flattened representation.
Within this representation, Muon orthogonalization is applied blockwise only to matrix-valued tensors inside transformer blocks; embeddings, output weights, and scalar residual parameters use the standard momentum recursion without orthogonalization.
To avoid the overhead of a separate evaluation pass, we report the training loss of the most recently computed minibatch from the worker that last returned a gradient.

\begin{table}[tbp]
    \centering
    \caption[Measured NanoChat gradient times across GPUs]{Measured stochastic-gradient times for the NanoChat setup used in our experiments: a $6$-layer, $192$-hidden model with sequence length $2048$ and device batch size $4$, run on a single GPU with \texttt{torch.compile} enabled and mixed precision where supported. Each entry reports the mean and standard deviation over $100$ synchronized gradient steps after $3$ warm-up iterations. The table shows substantial cross-device variability, with mean step times ranging from $14.18$\,ms on an H100 to $215.20$\,ms on a GTX 1080 Ti, which supports the heterogeneous delay ranges used in our simulated runtime model.}
    \label{tab:gpu_gradient_times}
    \begin{tabular}{@{}lccc@{}}
        \toprule
        GPU & Precision & Mean time (ms) & Std. (ms) \\
        \midrule
        NVIDIA H100-PCIE-80GB & bf16 & 14.18 & 7.32  \\
        NVIDIA A100-PCIE-40GB & bf16 & 22.74 & 11.64  \\
        NVIDIA RTX 4090 24GB & bf16 & 26.53 & 13.11  \\
        NVIDIA RTX 3090 24GB & bf16 & 38.42 & 15.10 \\
        NVIDIA V100-PCIE-16GB & fp16 & 64.15 & 20.30 \\
        NVIDIA T4 16GB & fp16 & 128.60 & 35.50 \\
        NVIDIA GTX 1080 Ti 11GB & fp32 & 215.20 & 40.15 \\
        \bottomrule
    \end{tabular}
\end{table}

To anchor the simulated runtime to measured hardware performance, we first profile one stochastic-gradient computation for the NanoChat model on the A100.
After $3$ warm-up steps, we average $100$ synchronized gradient computations and obtain a mean step time of $\widehat{\tau}=22.74$\,ms, as reported in \Cref{tab:gpu_gradient_times}.
The NanoChat simulator then converts the dimensionless profiles in \eqref{eq:experiment_delay_profiles} into milliseconds by assigning worker $i$ the deterministic base runtime
\begin{equation*}
    \tau_i=\widehat{\tau}g_i,
\end{equation*}
where $g_i$ is chosen as $g_i^{\mathrm{sim}}$, $g_i^{\mathrm{sub}}$, or $g_i^{\mathrm{lin}}$ depending on the delay regime.
Finally, we inject half-normal runtime noise with standard deviation equal to $5\%$ of this deterministic base runtime.
For $n=32$, the similar-delay regime is centered at $22.74$\,ms per worker; before adding noise, the sublinear regime ranges from $22.74$\,ms to $128.62$\,ms and the linear regime ranges from $22.74$\,ms to $727.57$\,ms.
For $n=8$, the corresponding upper bounds are $64.31$\,ms and $181.89$\,ms.
These simulated ranges are consistent with the cross-device variability in \Cref{tab:gpu_gradient_times}, while remaining tied to a measured gradient time for the exact NanoChat setup used in the experiments.

We evaluate NanoChat with $n=32$ and $n=8$ simulated workers.
Hyperparameters are tuned over a runtime horizon of $500$ simulated seconds with one trial per configuration.
The search grids are
\begin{itemize}
    \item \ringmastermuon: $\eta \in \{5^{-2}, 5^{-1}, 1, 5\}$ and $R \in \{1, 2, 4, 8, 16\}$;
    \item \ringmastermuonagnostic: $\eta \in \{5^{-2}, 5^{-1}, 1, 5, 25\}$;
    \item \rennalamuon: $\eta=5$ and $B \in \{1, 2, 4, 8, 16\}$;
    \item \muondelayadaptive: $\eta \in \{5^{-2}, 5^{-1}, 1, 5\}$.
\end{itemize}

\begin{figure*}[t]
    \centering
    \begin{minipage}[t]{0.32\textwidth}
        \centering
        \includegraphics[width=\linewidth]{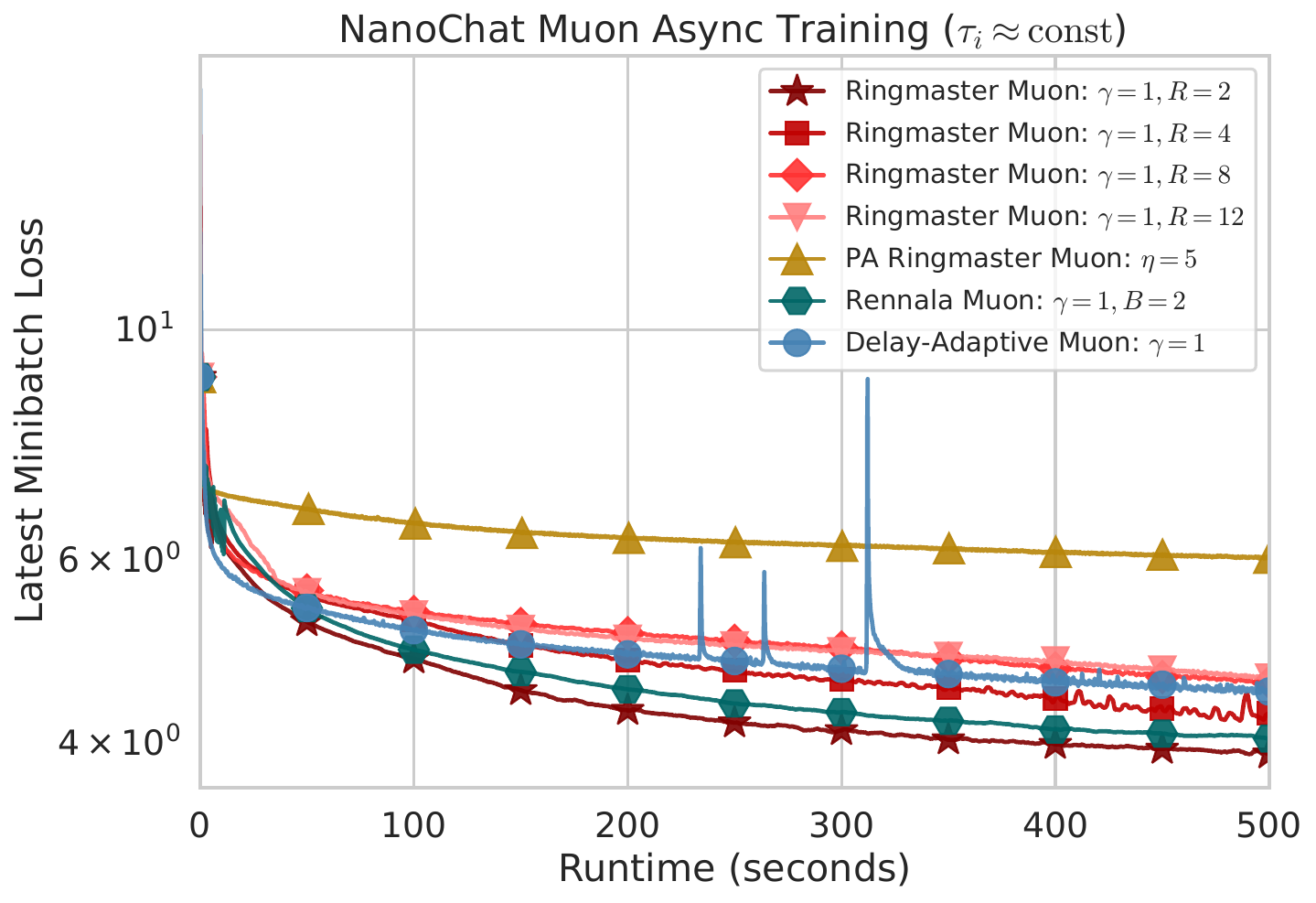}
        \small (a) $n=32$, nearly homogeneous delays
    \end{minipage}\hfill
    \begin{minipage}[t]{0.32\textwidth}
        \centering
        \includegraphics[width=\linewidth]{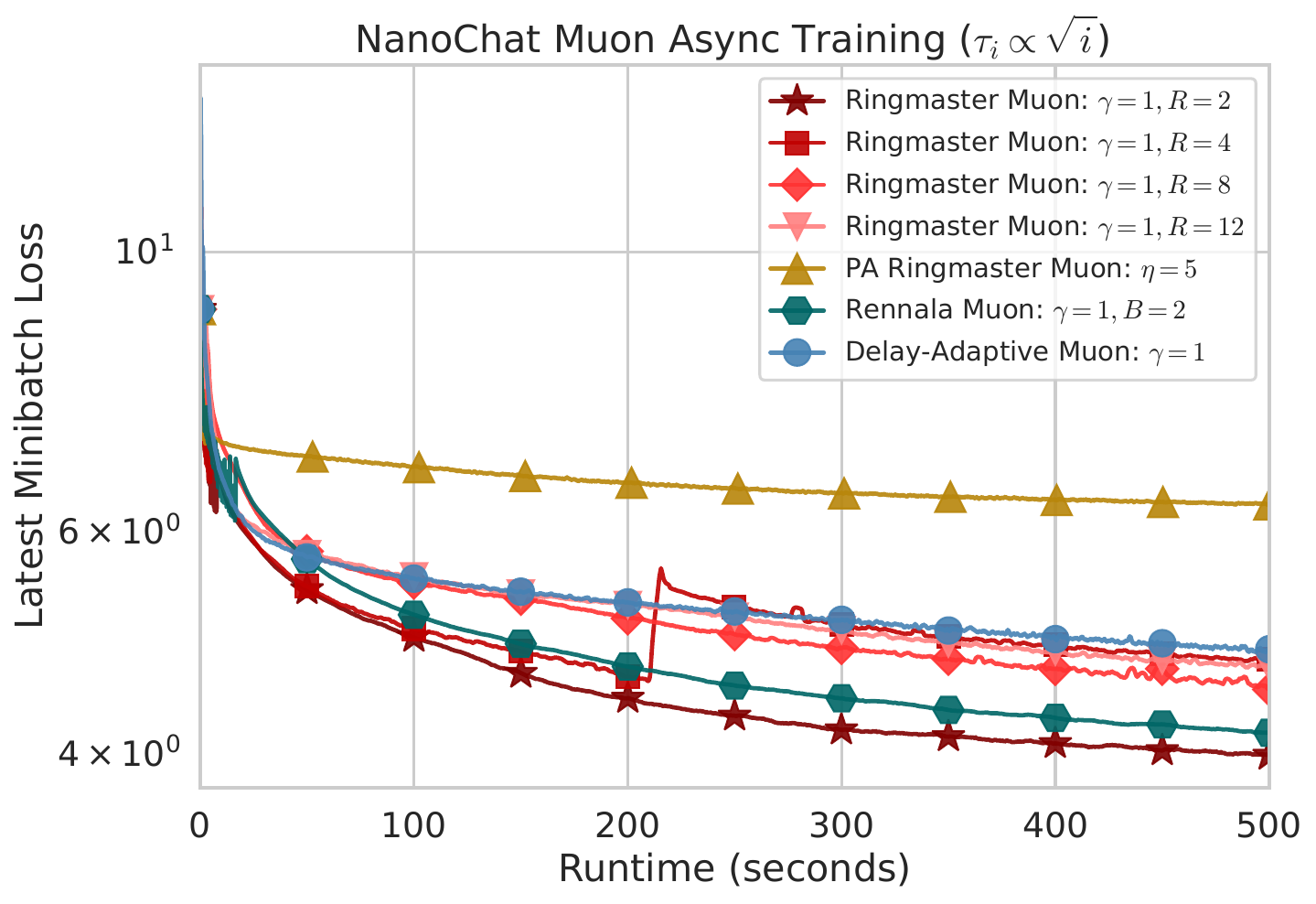}
        \small (b) $n=32$, sublinear delays
    \end{minipage}\hfill
    \begin{minipage}[t]{0.32\textwidth}
        \centering
        \includegraphics[width=\linewidth]{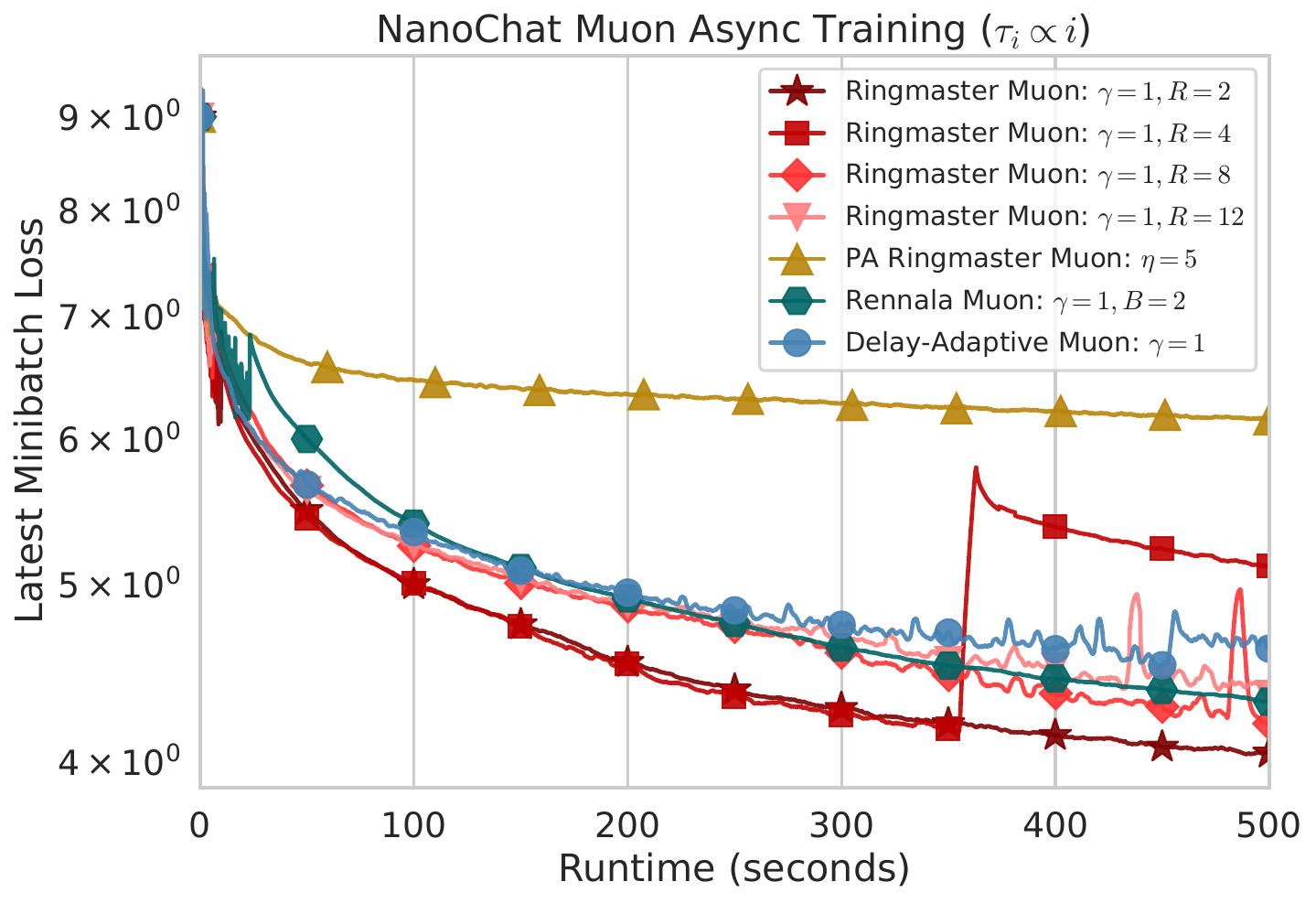}
        \small (c) $n=32$, linear delays
    \end{minipage}
    \caption[NanoChat comparison with 32 simulated workers]{NanoChat training loss versus simulated runtime for a $6$-layer, $192$-hidden, context-length-$2048$ model with $32$ simulated workers. The three panels correspond to nearly homogeneous, sublinear, and linear delay profiles, with delays anchored to the measured A100 mean gradient time. \ringmastermuon is competitive in the nearly homogeneous regime and achieves the lowest loss as heterogeneity increases, with the clearest margin in the linear-delay setting.}
    \label{fig:nanochat_32_experiments_placeholders}
\end{figure*}

\begin{figure*}[t]
    \centering
    \begin{minipage}[t]{0.32\textwidth}
        \centering
        \includegraphics[width=\linewidth]{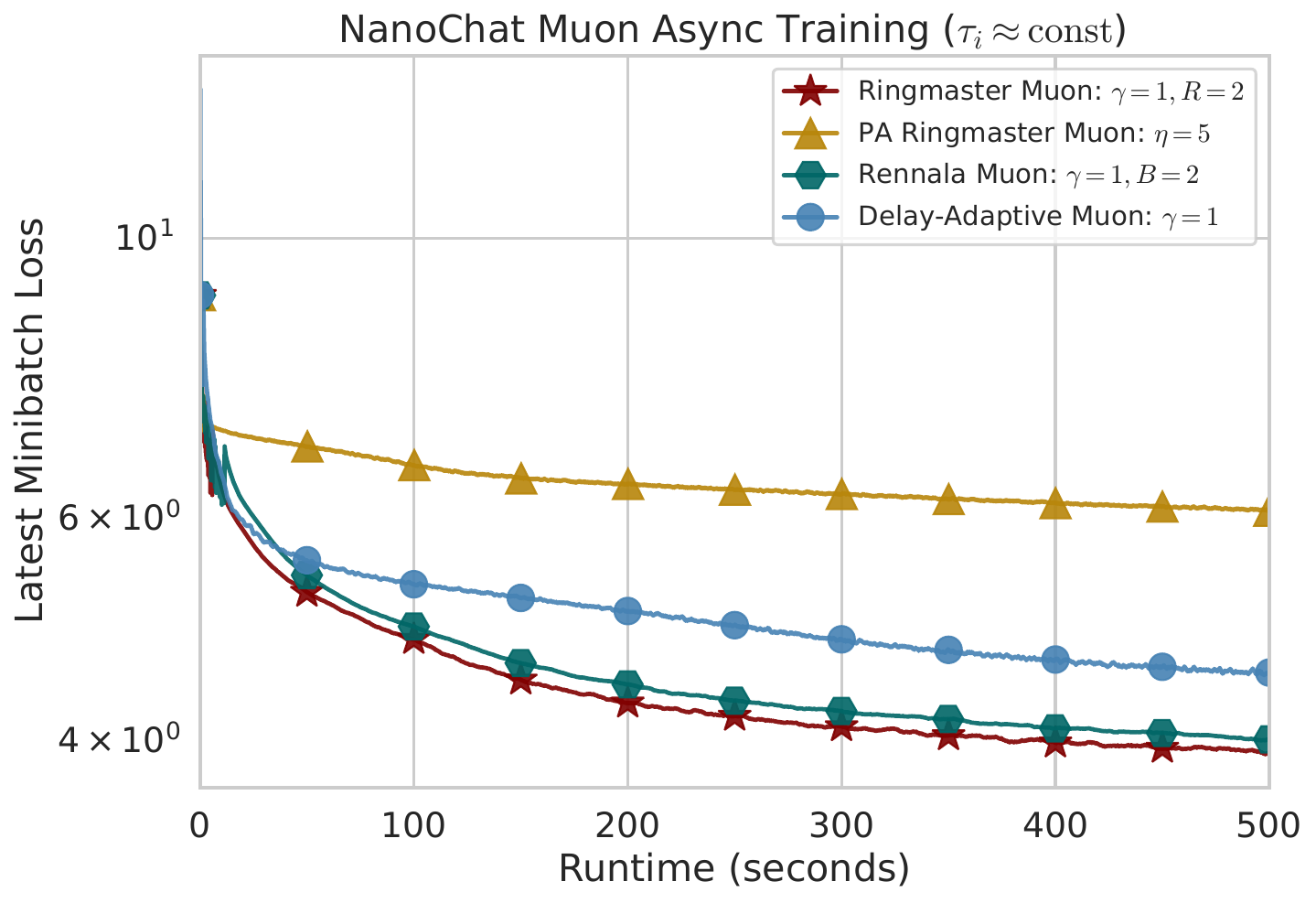}
        \small (a) $n=8$, nearly homogeneous delays
    \end{minipage}\hfill
    \begin{minipage}[t]{0.32\textwidth}
        \centering
        \includegraphics[width=\linewidth]{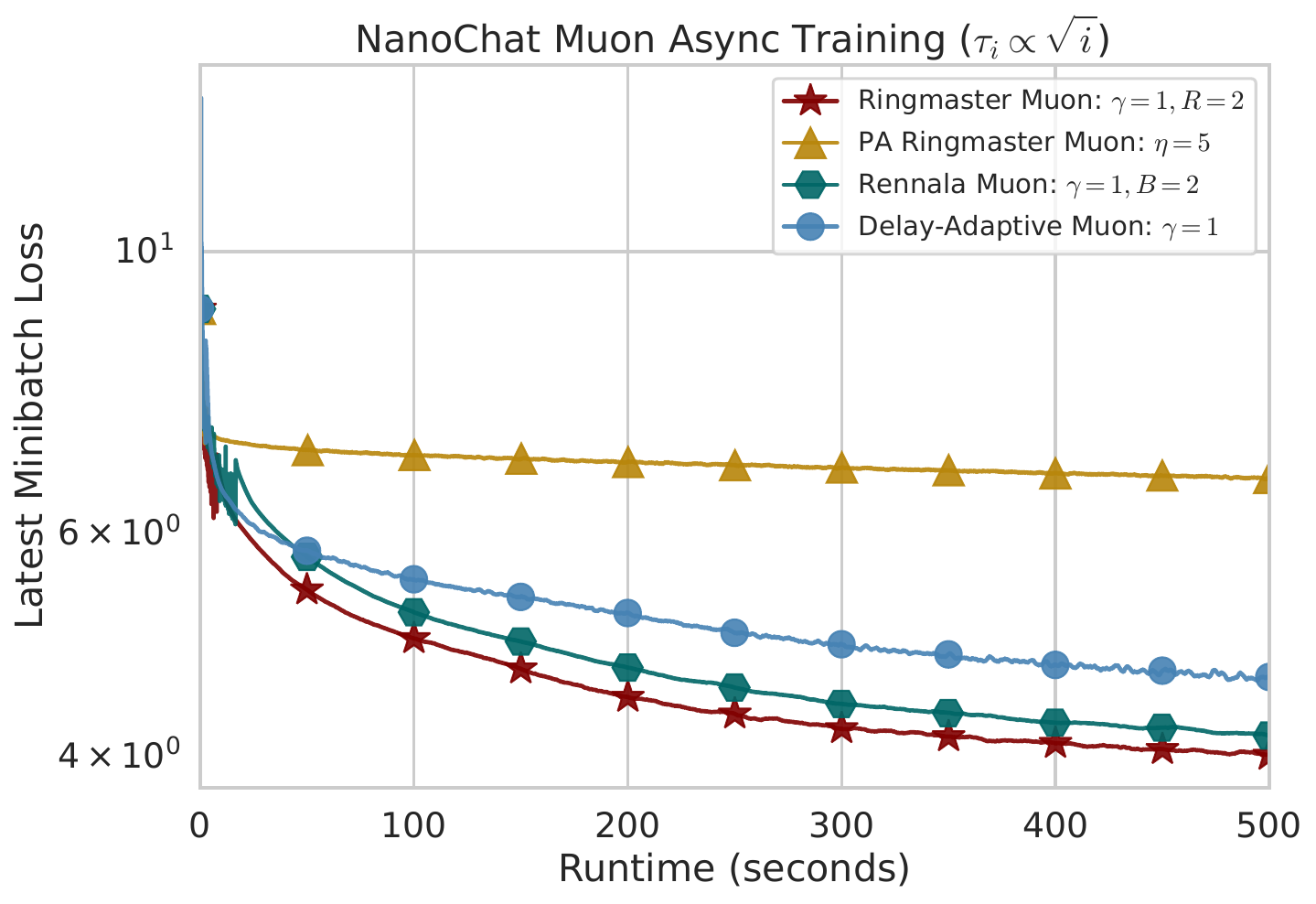}
        \small (b) $n=8$, sublinear delays
    \end{minipage}\hfill
    \begin{minipage}[t]{0.32\textwidth}
        \centering
        \includegraphics[width=\linewidth]{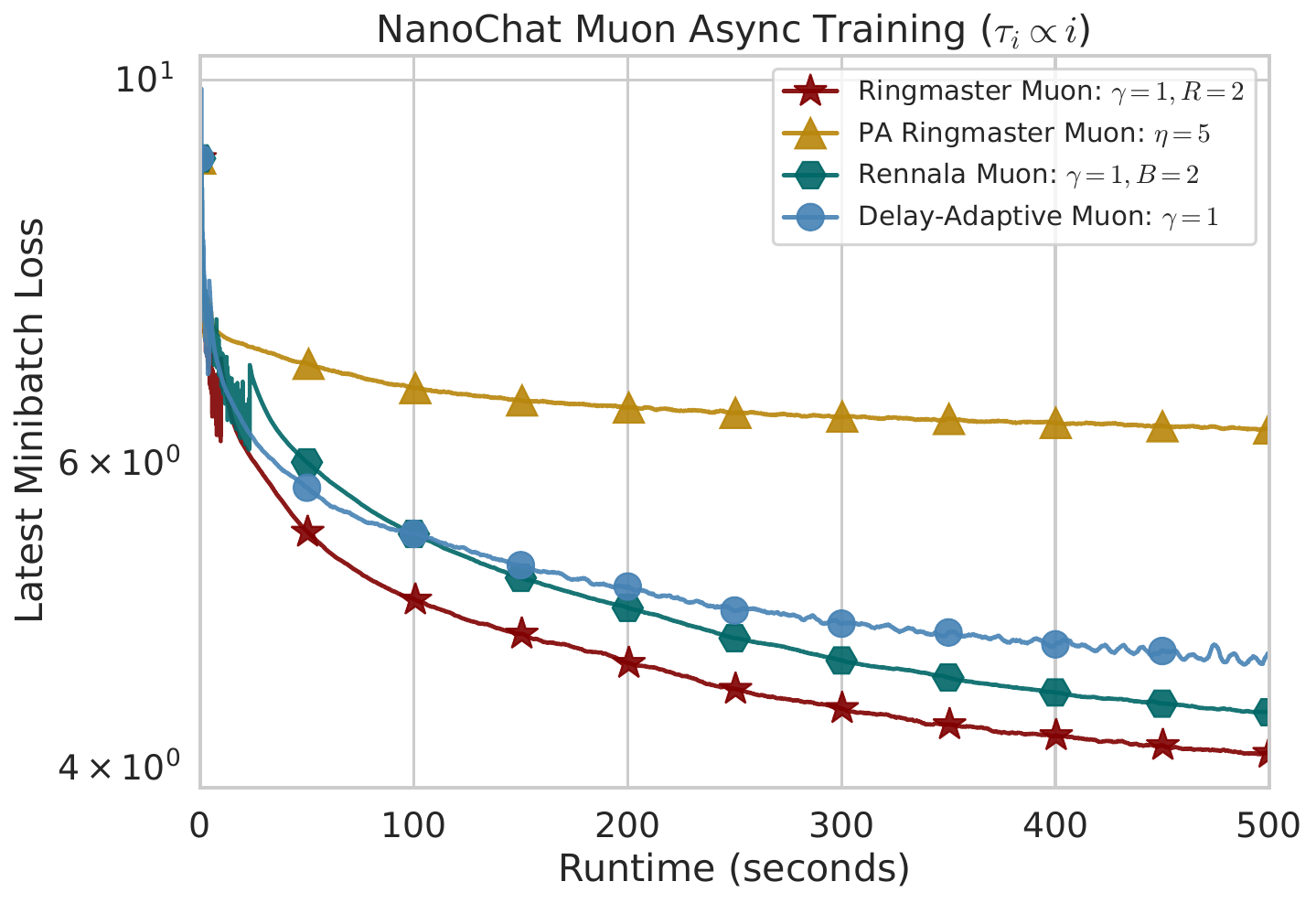}
        \small (c) $n=8$, linear delays
    \end{minipage}
    \caption[NanoChat comparison with 8 simulated workers]{NanoChat training loss versus simulated runtime for the same model with $8$ simulated workers. The ordering is preserved at the smaller worker count, again with a larger gap in the more heterogeneous delay regimes.}
    \label{fig:nanochat_8_experiments_placeholders}
\end{figure*}

\Cref{fig:nanochat_32_experiments_placeholders} shows that \ringmastermuon achieves the lowest training loss across all three $32$-worker regimes.
In the nearly homogeneous setting, the margin over the strongest asynchronous baselines is modest, consistent with the limited straggler effect.
As heterogeneity increases, the gain widens, with the clearest separation in the linear-delay regime.
\Cref{fig:nanochat_8_experiments_placeholders} shows the same qualitative ordering for $8$ workers.
Together with the quadratic benchmark, these results support the main empirical claim: delay thresholding is most useful in the heterogeneous, straggler-dominated settings targeted by the theory, and it does not require changing the underlying \muon update.
\section{Conclusion}\label{sec:conclusion}
We introduced \algn, an asynchronous LMO-based momentum method for stochastic nonconvex optimization, together with a parameter-agnostic variant.
Our analysis establishes iteration and time complexity guarantees under generalized $(L_0,L_1)$-smoothness and shows that, in the classical Euclidean smooth setting, the time bound recovers the optimal time complexity known for \ringmaster.
Experiments on stochastic quadratics and NanoChat pretraining further show that the benefits of asynchrony become more pronounced as worker-speed heterogeneity increases.

An important limitation of our analysis is that it assumes a homogeneous data distribution across workers; extending both the method and the theory to data-heterogeneous distributed settings is a natural next step.
Another limitation is that lower bounds for the asynchronous time complexity of general LMO-based methods in arbitrary norms remain open.
This question is challenging because matching lower bounds are not known even for the corresponding synchronous case. Developing such lower bounds for the parameter-agnostic case appears even harder.
\begin{ack}
    The research reported in this publication was supported by funding from King Abdullah University of Science and Technology (KAUST): i) KAUST Baseline Research Scheme, ii) CRG Grant ORFS-CRG12-2024-6460, and iii) Center of Excellence for Generative AI, under award number 5940.
\end{ack}
{
\small
\bibliographystyle{plainnat}
\bibliography{bib}
}
\newpage
\appendix
%
\section{Auxiliary lemmas}
We collect several auxiliary estimates used in the convergence proofs, including a summation/product bound, consequences of generalized $(L_0,L_1)$-smoothness, and a descent inequality for the LMO update.
\begin{lemma}[\cite{hubler2024parameter_agnostic}]
\label{lem:sum_product_bound}
    Consider parameters $q \in (0,1)$ and $p \geq 0$. Let $a, b \in \mathbb{N}$ such that $b \geq a \geq 2$.
    \begin{enumerate}
        \item The following inequality is true:
        \begin{equation}
            \label{eq:sum_product_bound_1}
            \prod^b_{t=a}(1-t^{-q}) \leq \exp\left(\frac{1}{1-q}\left(a^{1-q} - b^{1-q}\right)\right)
        \end{equation}
        \item If $p \leq q$, then the following inequality holds:
        \begin{equation}
            \label{eq:sum_product_bound_2}
            \sum^b_{t=a} t^{-p} \prod^t_{\tau =a} (1-\tau^{-q}) \leq (a-1)^{q-p} \exp\left(\frac{a^{1-q} - (a-1)^{1-q}}{1-q}\right).
        \end{equation}
        \item If the initial index satisfies $a \geq \max \{ p^{\frac{1}{1-q}}, (\frac{p-q}{2})^{\frac{1}{1-q}} \}$, then the following inequality holds:
        \begin{equation}
            \label{eq:sum_product_bound_3}
            \sum^b_{t=a} t^{-p} \prod^b_{\tau =t+1} (1-\tau^{-q}) \leq 2 \exp\left(\frac{1}{1-q}\right)(b+1)^{q-p}.
        \end{equation}
        In particular, for $p \leq 1$, these requirements are satisfied for all valid $a, b$.
    \end{enumerate}
\end{lemma}
Throughout this section, we assume the objective function is symmetric $(L_0,L_1)$-smooth. To facilitate the proof of our main convergence theorem, we first establish an auxiliary lemma detailing several of its useful properties.
\begin{lemma}[\cite{chen2023generalized}]
    \label{lem:gen_smooth}
    Let $f$ satisfy Assumption~\ref{ass:lower_bound} \textit{(lower boundedness)} and Assumption~\ref{ass:gen_smooth} (generalized smoothness).
    Then, for all $x,y\in \R^d$, the following holds:
    \begin{gather}
        f(y) \le f(x) + \inp{\nabla f(x)}{y-x} +\frac{1}{2}\left(L_0 + L_1\norm{\nabla f(x)}_*\right)\exp\left(L_1\norm{x-y}\right)\norm{x-y}^2, \label{eq:lem_smooth_1}\\
        \norm{\nabla f(x) - \nabla f(y)}_* \le \left(L_0 + L_1\norm{\nabla f(x)}_*\right)\exp\left(L_1\norm{x-y}\right)\norm{x-y} , \label{eq:lem_smooth_2}\\
        \norm{\nabla f(x)}^2_* \le 4 \left(L_0 + L_1 \norm{\nabla f(x)}_*\right)(f(x) - f_*). \label{eq:lem_smooth_3}
    \end{gather}
\end{lemma}
\begin{proof}
    Inequalities \eqref{eq:lem_smooth_1} and \eqref{eq:lem_smooth_2} were originally proven by \cite{chen2023generalized} for the Euclidean case, and subsequently extended to the non-Euclidean setting by \cite{khirirat2025better}.

    Regarding inequality \eqref{eq:lem_smooth_3}, the Euclidean case was established by \cite{gorbunov2025methods}.
    Here, we extend this result to the non-Euclidean case.
    To do so, we define $y = x + \frac{\alpha \norm{\nabla f(x)}_*}{L_0 + L_1 \norm{\nabla f(x)}_*}\operatorname{lmo}(\nabla f(x))$ and apply \eqref{eq:lem_smooth_1}:
    \begin{align*}
        f_* &\le f(y) \le f(x) + \inp{\nabla f(x)}{y-x} +\frac{1}{2}\left(L_0 + L_1\norm{\nabla f(x)}_*\right)\exp\left(L_1\norm{x-y}\right)\norm{x-y}^2\\
        &\le f(x) + \frac{\alpha \norm{\nabla f(x)}_*}{L_0 + L_1 \norm{\nabla f(x)}_*}\inp{\nabla f(x)}{\operatorname{lmo}(\nabla f(x))}\\
        &\quad + \frac{1}{2}(L_0 + L_1 \norm{\nabla f(x)}_*) \frac{\alpha^2 \norm{\nabla f(x)}_*^2}{(L_0 + L_1 \norm{\nabla f(x)}_*)^2}\exp\left(\frac{\alpha L_1\norm{\nabla f(x)}_*}{L_0 + L_1 \norm{\nabla f(x)}_*}\right)\\
        &\le f(x) - \frac{\alpha \norm{\nabla f(x)}_*^2}{L_0 + L_1 \norm{\nabla f(x)}_*} + \frac{\alpha^2 \exp(\alpha)}{2} \frac{\norm{\nabla f(x)}_*^2}{L_0 + L_1 \norm{\nabla f(x)}_*} ~.
    \end{align*}
    Setting $\alpha = \nicefrac{1}{2}$, we obtain:
    \begin{align*}
        f_* &\le f(x) - \frac{\norm{\nabla f(x)}_*^2}{2(L_0 + L_1 \norm{\nabla f(x)}_*)} +\frac{ \norm{\nabla f(x)}_*^2}{4(L_0 + L_1 \norm{\nabla f(x)}_*)}\\
        &= f(x) - \frac{ \norm{\nabla f(x)}_*^2}{4(L_0 + L_1 \norm{\nabla f(x)}_*)} ~.
    \end{align*}
\end{proof}
The following lemma establishes a standard descent inequality for the update step $x_{k+1} = x_{k} + \eta_k\operatorname{lmo}(m_{k+1})$.
\begin{lemma}[\cite{khirirat2025better}]
    \label{lem:descent_lemma}
    Under Assumption~\ref{ass:lower_bound} \textit{(lower boundedness)} Assumption~\ref{ass:unbiased_bounded_variance} \textit{(bounded variance)}  and Assumption~\ref{ass:gen_smooth} \textit{(generalized smoothness)}, the iterates generated by \Cref{alg} satisfy:
    \begin{equation}
        \sum_{k=0}^{K-1}\eta_k\phi_k\norm{\nabla f(x_k)}_* \leq \Delta_0 + 2\sum^{K-1}_{k=0}\eta_k\norm{\hat{e}_{k+1}}_* +\frac{L_0}{2}\sum^{K-1}_{k=0}\exp\left(L_1\eta_k\right)\eta_k^2 ~,
    \end{equation}
    where we define the error term as $\hat{e}_{k+1} \eqdef m_{k+1} - \nabla f(x_k)$ and the constant $\phi_k  = 1-\frac{L_1\eta_k}{2}\exp(L_1\eta_k)$.
\end{lemma}
\begin{proof}
    By \Cref{lem:gen_smooth}, we have:
    \begin{align*}
        f(x_{k+1})
        &\le f(x_k) + \inp{\nabla f(x_k)}{x_{k+1}-x_k}\\
            &\quad +\frac{1}{2}\left(L_0 + L_1\norm{\nabla f(x_k)}_*\right)\exp\left(L_1\norm{x_{k+1}-x_k}\right)\norm{x_{k+1}-x_k}^2\\
        &\le f(x_k) + \eta_k\inp{\nabla f(x_k)}{\operatorname{lmo}(m_{k+1})} +\frac{1}{2}\left(L_0 + L_1\norm{\nabla f(x_k)}_*\right)\exp\left(L_1\eta_k\right)\eta_k^2\\
        &= f(x_k) + \eta_k\inp{\nabla f(x_k) - m_{k+1}}{\operatorname{lmo}(m_{k+1})} + \eta_k\inp{m_{k+1}}{\operatorname{lmo}(m_{k+1})} \\
            &\quad +\frac{1}{2}\left(L_0 + L_1\norm{\nabla f(x_k)}_*\right)\exp\left(L_1\eta_k\right)\eta_k^2\\
        &\le f(x_k) + \eta_k\norm{\nabla f(x_k) - m_{k+1}}_* - \eta_k\norm{m_{k+1}}_*\\
            &\quad + \frac{1}{2}\left(L_0 + L_1\norm{\nabla f(x_k)}_*\right)\exp\left(L_1\eta_k\right)\eta_k^2\\
        &\le f(x_k) + 2\eta_k\norm{\nabla f(x_k) - m_{k+1}}_* - \eta_k\norm{\nabla f(x_k)}_*\\
            &\quad + \frac{1}{2}\left(L_0 + L_1\norm{\nabla f(x_k)}_*\right)\exp\left(L_1\eta_k\right)\eta_k^2 ~.
    \end{align*}
    In the steps above, the second inequality follows from the property $\norm{\operatorname{lmo}(\cdot)} \le 1$.
    The third inequality applies Hölder's inequality alongside the definition of the linear minimization oracle (which implies $\inp{m_{k+1}}{\operatorname{lmo}(m_{k+1})} = -\norm{m_{k+1}}_*$).
    Finally, the last step utilizes the triangle inequality.

    Rearranging the terms to isolate $\norm{\nabla f(x_k)}_*$ and summing over $k=0, \dots, K-1$ yields the desired result.
\end{proof}
We remind that momentum updates defined as follows 
\begin{align*}
    m_{1} &= (1-\alpha_0)m_0 + \alpha_{\text{init}} g_0,~\text{for } k = 0;\\
    m_{k+1} &=  (1-\alpha_k)m_k + \alpha_k g_k,~\text{for } k \ge 1~.
\end{align*}
We define the error term as follows
$$
    \hat{e}_{k+1} \eqdef m_{k+1} - \nabla f(x_k) ~.
$$
\begin{lemma}
\label{lem:error_bound}
    Suppose Assumption~\ref{ass:unbiased_bounded_variance} \textit{(bounded variance)} and Assumption~\ref{ass:gen_smooth} \textit{(generalized smoothness)}.
    If stepsize $\eta_k$ is non-increasing and momentum parameter $\alpha_k$ defined as $\alpha_0 =1$ and $\alpha_{\mathrm{init}},~ \alpha_k \in (0,1]$ for $k \ge 1$, then  the iterates of \Cref{alg} satisfy
    \begin{align*}
        \Exp{\norm{\hat{e}_{k+1}}_*}
        &\le \sum_{t=1}^k \prod_{j=t}^{k}(1-\alpha_j) \left(L_0 + L_1\Exp{\norm{\nabla f(x_t)}_*}\right)\exp(L_1\eta_{t-1})\eta_{t-1}\\
            &\quad + \sum_{t=1}^k \prod_{j=t+1}^{k}(1-\alpha_j) \alpha_t (L_0 + L_1 \Exp{\norm{\nabla f(x_t)}_*})\exp\left(L_1 R_t\eta_{t-\delta_t}\right)R_t\eta_{t-\delta_t}\\
                &\quad + \rho  \sigma\sqrt{\sum_{t=1}^k \prod_{j=t+1}^{k}(1-\alpha_j)^2 \alpha_t^2 + \prod_{j=1}^{k}(1-\alpha_j)^2 \alpha_{\mathrm{init}}^2 }\\
                &\quad + \prod_{t = 1}^{k}(1-\alpha_t)(1-\alpha_{\mathrm{init}})\Exp{\norm{\nabla f(x_0)}_*}.
    \end{align*}
\end{lemma}
Throughout the paper, an empty product is interpreted as $1$ and an empty sum as $0$.
\begin{proof}
Using momentum update, we derive for any $k \geq 1$: 
\begin{eqnarray*}
    \hat{e}_{k+1}
    &=& m_{k+1} - \nabla f(x_k) \\
    &=& (1-\alpha_k)m_k +\alpha_k g_k - \nabla f(x_k) \\
    &=& (1-\alpha_k)(m_k- \nabla f(x_{k-1})) + (1-\alpha_k)(\nabla f(x_{k-1}) - \nabla f(x_k)) +\alpha_k g_k - \alpha_k\nabla f(x_k) \\
    &=& (1-\alpha_k)(m_k- \nabla f(x_{k-1})) + (1-\alpha_k)(\nabla f(x_{k-1}) - \nabla f(x_k)) \\
    && +\alpha_k (\nabla f(x_{k-\delta_k}) - \nabla f(x_k))  +\alpha_k( g_k - \nabla f(x_{k-\delta_k}) ) ~.
\end{eqnarray*}
Denoting $D_k \eqdef \nabla f(x_{k-1}) - \nabla f(x_k)$, $S_k \eqdef \nabla f(x_{k-\delta_k}) - \nabla f(x_k)$, and $e_k \eqdef g_k - \nabla f(x_{k-\delta_k})$, we obtain
\begin{eqnarray}
    \hat{e}_{k+1}
    &=& (1-\alpha_k) \hat{e}_{k} + (1-\alpha_k) D_k + \alpha_k S_k +\alpha_k e_k \notag\\
    &=&\prod_{t = 1}^{k}(1-\alpha_t) \hat{e}_{1} + \sum_{t=1}^k \prod_{j=t+1}^{k}(1-\alpha_j)(1-\alpha_t) D_t \notag\\
    && + \sum_{t=1}^k \prod_{j=t+1}^{k}(1-\alpha_j) \alpha_t S_t + \sum_{t=1}^k \prod_{j=t+1}^{k}(1-\alpha_j) \alpha_t e_t ~. \label{eq:sbjdhchsbdbs}
\end{eqnarray}

Now we consider $\hat{e}_{1}$ at iteration $k=1$. Using $\alpha_0=1$ and $\delta_0=0$ (which implies $x_{0-\delta_0} =x_0$), we have 
\begin{eqnarray}
    \hat{e}_{1} &=& m_1 - \nabla f(x_0)  \notag\\
    &=& (1-\alpha_0)m_0 + \alpha_{\mathrm{init}} g_0 - \nabla f(x_0) \notag\\
    &=& \alpha_{\mathrm{init}} (g_0 - \nabla f(x_0)) - (1-\alpha_{\mathrm{init}})\nabla f(x_0)\notag\\
    &=& \alpha_{\mathrm{init}} e_0 - (1-\alpha_{\mathrm{init}})\nabla f(x_0)~. \label{eq:hgsvgvcdvsgdusu} 
\end{eqnarray}

Plugging \eqref{eq:hgsvgvcdvsgdusu} into \eqref{eq:sbjdhchsbdbs}, we have
\begin{eqnarray*}
    \hat{e}_{k+1} &=&  \prod_{t = 1}^{k}(1-\alpha_t) \alpha_{\mathrm{init}} e_0 - \prod_{t = 1}^{k}(1-\alpha_t)(1-\alpha_{\mathrm{init}})\nabla f(x_0)\\
    && +\sum_{t=1}^k \prod_{j=t+1}^{k}(1-\alpha_j)(1-\alpha_t) D_t  + \sum_{t=1}^k \prod_{j=t+1}^{k}(1-\alpha_j) \alpha_t S_t + \sum_{t=1}^k \prod_{j=t+1}^{k}(1-\alpha_j) \alpha_t e_t\\
    &=& \underbrace{\sum_{t=1}^k \prod_{j=t+1}^{k}(1-\alpha_j)(1-\alpha_t) D_t}_{\eqdef\cT_1} + \underbrace{\sum_{t=1}^k \prod_{j=t+1}^{k}(1-\alpha_j) \alpha_t S_t}_{\eqdef \cT_2} \\
    && + \underbrace{\sum_{t=1}^k \prod_{j=t+1}^{k}(1-\alpha_j) \alpha_t e_t + \prod_{t = 1}^{k}(1-\alpha_t) \alpha_{\mathrm{init}} e_0}_{\eqdef \cT_3} - \underbrace{\prod_{t = 1}^{k}(1-\alpha_t)(1-\alpha_{\mathrm{init}})\nabla f(x_0)}_{\eqdef \cT_4}.
\end{eqnarray*}
To continue the proof, we need to bound the norm of terms $\cT_1$, $\cT_2$, $\cT_3$ and $\cT_4$. By triangle inequality, we get
\begin{eqnarray*}
    \norm{\cT_1}_* &\le& \sum_{t=1}^k \prod_{j=t+1}^{k}(1-\alpha_j)(1-\alpha_t) \norm{D_t}_*\\
    &=& \sum_{t=1}^k \prod_{j=t}^{k}(1-\alpha_j) \norm{ \nabla f(x_{t-1}) - \nabla f(x_t)}_*\\
    &\le& \sum_{t=1}^k \prod_{j=t}^{k}(1-\alpha_j) \left(L_0 + L_1\norm{\nabla f(x_t)}_*\right)\exp(L_1\eta_{t-1})\eta_{t-1} ~.
\end{eqnarray*}
We used \Cref{lem:gen_smooth} in the last inequality.
Bounding $\norm{\cT_2}_*$, we obtain
\begin{eqnarray*}
    \norm{\cT_2}_* &\le& \sum_{t=1}^k \prod_{j=t+1}^{k}(1-\alpha_j) \alpha_t \norm{S_t}_*\\
    &=& \sum_{t=1}^k \prod_{j=t+1}^{k}(1-\alpha_j) \alpha_t \norm{\nabla f(x_{t-\delta_t}) - \nabla f(x_t)}_*\\
    &\le& \sum_{t=1}^k \prod_{j=t+1}^{k}(1-\alpha_j) \alpha_t \(L_0 + L_1 \norm{\nabla f(x_t)}_*\)\exp\left(L_1 \norm{x_t - x_{t-\delta_t}}\right)\norm{x_t - x_{t-\delta_t}}\\
    &\le& \sum_{t=1}^k \prod_{j=t+1}^{k}(1-\alpha_j) \alpha_t \(L_0 + L_1 \norm{\nabla f(x_t)}_*\)\\
    &&\quad \times \exp\left(L_1 \norm{\sum^{t-1}_{\tau = t-\delta_t}\eta_{\tau}\textrm{lmo}(m_{\tau+1})}\right)\norm{\sum^{t-1}_{\tau = t-\delta_t}\eta_{\tau}\textrm{lmo}(m_{\tau+1})}.
\end{eqnarray*}
Assuming that stepsize $\eta_k$ is non-increasing, we have
\begin{eqnarray*}
    \norm{\cT_2}_* &\le& \sum_{t=1}^k \prod_{j=t+1}^{k}(1-\alpha_j) \alpha_t \( L_0 + L_1 \norm{\nabla f(x_t)}_* \) \exp\left(L_1 R_t\eta_{t-\delta_t}\right)R_t\eta_{t-\delta_t} ~.
\end{eqnarray*}
Bounding $\norm{\cT_3}_*$, we obtain
\begin{eqnarray*}
    \norm{\cT_3}_* &\le& \rho \norm{\sum_{t=1}^k \prod_{j=t+1}^{k}(1-\alpha_j) \alpha_t e_t +\prod_{t = 1}^{k}(1-\alpha_t) \alpha_{\mathrm{init}} e_0}_2.
\end{eqnarray*}
Taking expectation, we have
\begin{eqnarray*}
    \Exp{\norm{\cT_3}_*} &\le& \rho \Exp{\norm{\sum_{t=1}^k \prod_{j=t+1}^{k}(1-\alpha_j) \alpha_t e_t +\prod_{t = 1}^{k}(1-\alpha_t) \alpha_{\mathrm{init}} e_0}_2}\\
    &\le& \rho \sqrt{\Exp{\norm{\sum_{t=1}^k \prod_{j=t+1}^{k}(1-\alpha_j) \alpha_t e_t +\prod_{t = 1}^{k}(1-\alpha_t) \alpha_{\mathrm{init}} e_0}_2^2}}\\
    &\le& \rho \sqrt{\sum_{t=1}^k \prod_{j=t+1}^{k}(1-\alpha_j)^2 \alpha_t^2 \Exp{\norm{e_t}_2^2}+ \prod_{j=1}^{k}(1-\alpha_j)^2 \alpha_{\mathrm{init}}^2 \Exp{\norm{e_0}_2^2}} \\
    &\le& \rho \sigma\sqrt{\sum_{t=1}^k \prod_{j=t+1}^{k}(1-\alpha_j)^2 \alpha_t^2 + \prod_{j=1}^{k}(1-\alpha_j)^2 \alpha_{\mathrm{init}}^2 } ~,
\end{eqnarray*}
where the second inequality uses Jensen's inequality and the third uses iid ness.
Bounding $\norm{\cT_4}_*$, we obtain
\begin{eqnarray*}
    \norm{\cT_4}_* &\le& \prod_{t=1}^{k}(1-\alpha_t)(1-\alpha_{\mathrm{init}})\norm{\nabla f(x_0)}_*.
\end{eqnarray*}
Thus, we have
\begin{eqnarray*}
    \Exp{\norm{\hat{e}_{k+1}}_*}
    &\le&  \sum_{t=1}^k \prod_{j=t}^{k}(1-\alpha_j) \left(L_0 + L_1\Exp{\norm{\nabla f(x_t)}_*}\right)\exp(L_1\eta_{t-1})\eta_{t-1}\\
    && + \sum_{t=1}^k \prod_{j=t+1}^{k}(1-\alpha_j) \alpha_t (L_0 + L_1 \Exp{\norm{\nabla f(x_t)}_*})\exp\left(L_1 R_t\eta_{t-\delta_t}\right)R_t\eta_{t-\delta_t}\\
    && + \rho \sigma\sqrt{\sum_{t=1}^k \prod_{j=t+1}^{k}(1-\alpha_j)^2 \alpha_t^2 + \prod_{j=1}^{k}(1-\alpha_j)^2 \alpha_{\mathrm{init}}^2 }\\
    && + \prod_{t = 1}^{k}(1-\alpha_t)(1-\alpha_{\mathrm{init}})\Exp{\norm{\nabla f(x_0)}_*}.    
\end{eqnarray*}
\end{proof}
\newpage
\section{Proof of \Cref{thm:iteration_complexity}}
\label{proof:iteration_complexity}
This section proves the fixed-parameter result and derives the parameter choices by balancing the terms in the convergence bound.
We first restate the theorem and then optimize the constant momentum, threshold, and stepsize parameters.

Let us first restate the theorem.
\begin{restate-theorem}{\ref{thm:iteration_complexity}}
[Fixed-Parameter Iteration Complexity; Proof in \Cref{proof:iteration_complexity}]
    Suppose Assumptions~\ref{ass:lower_bound}--\ref{ass:gen_smooth} hold.
    Run \Cref{alg} with $\alpha_0=1$ and with the constant momentum parameter and delay threshold
    \begin{equation*}
        \alpha_k = \alpha_{\mathrm{init}} = \alpha = \min\left\{1, \frac{\Delta_0^{\nicefrac{1}{2}}L_0^{\nicefrac{1}{2}}}{\rho\sigma K^{\nicefrac{1}{2}}}\right\}, \quad \forall k \geq 1,
        \quad
        R_k = R = \frac{1}{\alpha}, \quad \forall k \ge 0,
    \end{equation*}
    with the convention that, when $\sigma=0$, we take $\alpha=1$,
    and choose the constant stepsize as follows:
    \begin{itemize}
        \item If $L_1=0$, take
        \begin{equation*}
            \eta_k = \eta = \min\left\{\sqrt{\frac{\Delta_0}{L_0 K}}, \frac{\Delta_0^{\nicefrac{3}{4}}}{ L_0^{\nicefrac{1}{4}}(\rho\sigma)^{\nicefrac{1}{2}} K^{\nicefrac{3}{4}}}\right\}, \quad \forall k \ge 0.
        \end{equation*}
        \item If $L_1>0$, take
        \begin{equation*}
            \eta_k = \eta = \min\left\{\sqrt{\frac{\Delta_0}{L_0 K}}, \frac{1}{8L_1}, \frac{\Delta_0^{\nicefrac{1}{2}}L_0^{\nicefrac{1}{2}}}{8L_1\rho\sigma K^{\nicefrac{1}{2}}},  \frac{\Delta_0^{\nicefrac{3}{4}}}{ L_0^{\nicefrac{1}{4}}(\rho\sigma)^{\nicefrac{1}{2}} K^{\nicefrac{3}{4}}}\right\}, \quad \forall k \ge 0.
        \end{equation*}
    \end{itemize}
    Then, after $K \ge 1$ iterations, the iterates of \Cref{alg} satisfy
    \begin{equation*}
        \min_{k \in \{0,\dots, K-1\}} \Exp{\|\nabla f(x_k)\|_*} \leq \cO\left(\sqrt{\frac{L_0\Delta_0}{K}} + \frac{L_1\Delta_0}{K} + \frac{L_1 \Delta_0^{\nicefrac{1}{2}} \rho\sigma}{L_0^{\nicefrac{1}{2}}K^{\nicefrac{1}{2}}} + \frac{L_0^{\nicefrac{1}{4}}\Delta_0^{\nicefrac{1}{4}}(\rho\sigma)^{\nicefrac{1}{2}}}{K^{\nicefrac{1}{4}}} + \frac{\rho\sigma}{K^{\nicefrac{1}{2}}}
    \right).
    \end{equation*}
    In the case $L_1=0$, the terms containing $L_1$ in the preceding bound are zero.
    Consequently, the number of iterations needed to reach an $\varepsilon$-stationary point satisfies
    \begin{equation*}
        K_{\varepsilon}
        = \cO\left(
            \frac{L_0\Delta_0 }{\varepsilon^2}
            + \frac{L_1\Delta_0 }{\varepsilon}
            + \frac{L_0\Delta_0(\rho\sigma)^2}{\varepsilon^4}
            + \frac{L_1^2\Delta_0(\rho\sigma)^2}{L_0\varepsilon^2}
            + \frac{(\rho\sigma)^2}{\varepsilon^2}
            \right).
    \end{equation*}
    Again, when $L_1=0$, the terms containing $L_1$ are omitted.
\end{restate-theorem}
\begin{proof}
    We start with Lemma~\ref{lem:descent_lemma}, taking $\eta_k = \eta$ and $\alpha_k = 1$ for $k = 0$ and $\alpha_k =\alpha$ for $k \ge 1$,
    \begin{eqnarray}
        0 &\leq& \Delta_0 + 2\eta \sum^{K-1}_{k=0}\|\hat e_{k+1}\|_* +\frac{1}{2}\eta^2 L_0 e^{L_1\eta}K - \eta \sum^{K-1}_{k=0}\left(1-\frac{1}{2}L_1\eta e^{L_1\eta}\right)\|\nabla f(x_k)\|_* ~. \label{eq:bjhsdvbsdk}
    \end{eqnarray}
    Taking expectation, we bound the second term from the last inequality.
    By Lemma~\ref{lem:error_bound} (take $\alpha_{\mathrm{init}} = \alpha$ ), we have
    \begin{eqnarray*}
        2\eta \sum^{K-1}_{k=0}\Exp{\|\hat e_{k+1}\|_*}  &\leq&  \underbrace{2\eta \sum^{K-1}_{k=0}\sum_{t=1}^k (1-\alpha)^{k-t} \left(L_0 + L_1\Exp{\norm{\nabla f(x_t)}_*}\right)e^{L_1\eta}\eta}_{\eqdef \cT_1}\\
        && + \underbrace{2\eta \sum^{K-1}_{k=0} \sum_{t=1}^k(1-\alpha)^{k-t}\alpha (L_0 + L_1 \Exp{\norm{\nabla f(x_t)}_*})e^{L_1 R\eta}R\eta}_{\eqdef \cT_2}\\
         && + \underbrace{2\eta \sum^{K-1}_{k=0}\rho  \sigma\sqrt{\sum_{t=0}^k (1-\alpha)^{2(k-t)}\alpha^2}}_{\eqdef \cT_3} + \underbrace{2\eta \sum^{K-1}_{k=0}(1-\alpha)^{k+1} \Exp{\norm{\nabla f(x_0)}_*}}_{\eqdef \cT_4}~.
    \end{eqnarray*}
    Next, we bound each $\cT_1,~\cT_2,~\cT_3,~\cT_4$ and sum them up. 
    \begin{eqnarray}
        \cT_1 &=& 2L_0\eta^2 e^{L_1\eta} \sum^{K-1}_{k=0}\sum_{t=1}^k (1-\alpha)^{k-t}  +  2L_1\eta^2 e^{L_1\eta} \sum^{K-1}_{k=0}\sum_{t=1}^k (1-\alpha)^{k-t}  \Exp{\norm{\nabla f(x_t)}_*} \notag\\
        &\leq& 2L_0\eta^2 e^{L_1\eta} \sum^{K-1}_{k=0}\sum_{t=0}^{\infty} (1-\alpha)^{t}  +  2L_1\eta^2 e^{L_1\eta} \sum_{t=1}^{K-1}\sum^{K-1}_{k=t} (1-\alpha)^{k-t}  \Exp{\norm{\nabla f(x_t)}_*} \notag\\
        &\leq&  \frac{2}{\alpha}L_0\eta^2 e^{L_1\eta} K + 2L_1\eta^2 e^{L_1\eta} \sum_{t=1}^{K-1}\sum^{\infty}_{k=0} (1-\alpha)^{k}  \Exp{\norm{\nabla f(x_t)}_*} \notag\\
        &\leq& \frac{2}{\alpha}L_0\eta^2 e^{L_1\eta} K + \frac{2}{\alpha}L_1\eta^2 e^{L_1\eta} \sum_{k=1}^{K-1} \Exp{\norm{\nabla f(x_k)}_*}~. \label{eq:bound_T_1}
    \end{eqnarray}
    Now we bound $\cT_2$, By the same argument for bounding $\cT_1$, we obtain
    \begin{eqnarray}
        \cT_2 &\leq& \frac{2}{\alpha}L_0R\eta^2\alpha e^{L_1R\eta} K + \frac{2}{\alpha}L_1R\eta^2\alpha e^{L_1R\eta} \sum_{k=1}^{K-1} \Exp{\norm{\nabla f(x_k)}_*}~. \label{eq:bound_T_2}
    \end{eqnarray}
    Next, we bound $\cT_3$:
    \begin{eqnarray}
        \cT_3 &\leq& 2\eta\rho \sigma \sum^{K-1}_{k=0}\sqrt{\sum_{t=0}^k (1-\alpha)^{2(k-t)}\alpha^2} \notag\\
        &\leq& 2\eta\rho \sigma \sum^{K-1}_{k=0}\sqrt{\sum_{t=0}^{\infty} (1-\alpha)^{t}\alpha^2} \notag\\
        &\leq& 2\eta\rho \sigma \sqrt{\alpha} K.~ \label{eq:bound_T_3}
    \end{eqnarray}
    Finally, we bound $\cT_4$:
    \begin{eqnarray}
        \cT_4 &\leq& 2\eta \sum^{K-1}_{k=0}(1-\alpha)^{k+1} \Exp{\|\nabla f(x_0)\|_*} \notag\\
        &\leq& \frac{2\eta}{\alpha}\Exp{\|\nabla f(x_0)\|_*}~. \label{eq:bound_T_4}
    \end{eqnarray}    
    Plugging \eqref{eq:bound_T_1}, \eqref{eq:bound_T_2}, \eqref{eq:bound_T_3}, and \eqref{eq:bound_T_4} into \eqref{eq:bjhsdvbsdk}, we have
    \begin{eqnarray*}
        0 &\leq& \Delta_0 +\frac{1}{2}\eta^2 L_0 e^{L_1\eta}K - \eta \sum^{K-1}_{k=0}\left(1-\frac{1}{2}L_1\eta e^{L_1\eta}\right)\Exp{\|\nabla f(x_k)\|_*}\\
        && + \frac{2}{\alpha}L_0\eta^2 e^{L_1\eta} K + \frac{2}{\alpha}L_1\eta^2 e^{L_1\eta} \sum_{k=1}^{K-1} \Exp{\norm{\nabla f(x_k)}_*}\\
        && + \frac{2}{\alpha}L_0R\eta^2\alpha e^{L_1R\eta} K + \frac{2}{\alpha}L_1R\eta^2\alpha e^{L_1R\eta} \sum_{k=1}^{K-1} \Exp{\norm{\nabla f(x_k)}_*}\\
        && + 2\eta\rho \sigma \sqrt{\alpha} K + \frac{2\eta}{\alpha}\Exp{\|\nabla f(x_0)\|_*}~.
    \end{eqnarray*}
    Multiplying by $\frac{2}{\eta K}$, we have
    \begin{eqnarray*}
        0 &\leq& \frac{2\Delta_0}{\eta K} + 4\rho \sigma \sqrt{\alpha} +\eta L_0 \left( e^{L_1\eta} + \frac{4}{\alpha} e^{L_1\eta} + \frac{4}{\alpha} R\alpha e^{L_1R\eta} \right) + \frac{4}{\alpha K}\Exp{\|\nabla f(x_0)\|_*}\\
        &&- \frac{2}{K} \sum^{K-1}_{k=0}\left(1-\frac{1}{2}L_1\eta e^{L_1\eta} - \frac{2}{\alpha}L_1\eta e^{L_1\eta} - \frac{2}{\alpha}L_1R\eta\alpha e^{L_1R\eta}\right)\Exp{\|\nabla f(x_k)\|_*} ~.
    \end{eqnarray*}
    Assuming $R = \frac{1}{\alpha}$ and $\eta \leq \frac{\alpha}{8L_1}$ and $\alpha \leq 1$, we have
    \begin{eqnarray*}
         \frac{1}{K} \sum^{K-1}_{k=0}\Exp{\|\nabla f(x_k)\|_*}&\leq& \frac{2\Delta_0}{\eta K} + 4\rho \sigma \sqrt{\alpha}  +  \frac{4}{\alpha K}\Exp{\|\nabla f(x_0)\|_*} + 10\eta \frac{L_0}{\alpha} + 2\eta L_0 ~.
    \end{eqnarray*}

    To select $\alpha$, we minimize the upper bound w.r.t $\alpha$ with restriction that $\alpha < 1$. Thus, we have  $\alpha = \min\left\{1, \frac{\Delta_0^{\nicefrac{1}{2}}L_0^{\nicefrac{1}{2}}}{\rho\sigma K^{\nicefrac{1}{2}}}\right\}$. 
    To do so for $\eta$, we first plug in $\alpha$ in the upper bound and minimize w.r.t. $\eta$ under constraints that $\eta \leq \frac{\alpha}{8L_1}$. This allows us to set a stepsize as 
     $\eta = \min\left\{\sqrt{\frac{\Delta_0}{L_0 K}}, \frac{1}{8L_1}, \frac{\Delta_0^{\nicefrac{1}{2}}L_0^{\nicefrac{1}{2}}}{8L_1\rho\sigma K^{\nicefrac{1}{2}}}, \frac{\Delta_0^{\nicefrac{3}{4}}}{K^{\nicefrac{3}{4}} L_0^{\nicefrac{1}{4}}(\rho\sigma)^{\nicefrac{1}{2}}}\right\}$.
    Also, by \Cref{lem:gen_smooth} (inequality\eqref{eq:lem_smooth_3}) applied at $x_0$, we have
    \begin{equation*}
        \norm{\nabla f(x)}_*^2\leq 
        \begin{cases}
            8L_0\Delta_0 & \text{if} \norm{\nabla f(x)} \leq \frac{L_0}{L_1};\\
            (8L_1\Delta_0)^2 & \text{otherwise};
        \end{cases} 
        \leq \max\left\{8L_0\Delta_0, \left(8L_1\Delta_0\right)^2\right\}.
    \end{equation*}
    Therefore, we derive
    \begin{align*}
        \min_{k \in \{0,\dots, K-1\}} & \E{\norm{\nabla f(x_k)}_*} \\
        &\leq \cO\left(\sqrt{\frac{L_0\Delta_0}{K}} + \frac{\sqrt{L_0\Delta_0} +L_1\Delta_0}{K} + \frac{L_1 \Delta_0^{\nicefrac{1}{2}}}{L_0^{\nicefrac{1}{2}}K^{\nicefrac{1}{2}}} \rho\sigma + \frac{L_0^{\nicefrac{1}{4}}\Delta_0^{\nicefrac{1}{4}}(\rho \sigma)^{\nicefrac{1}{2}}}{K^{\nicefrac{1}{4}}} + \frac{\rho\sigma}{K^{\nicefrac{1}{2}}}\right)\\
        &= \cO\left(\sqrt{\frac{L_0\Delta_0}{K}} + \frac{L_1\Delta_0}{K} + \frac{L_1 \Delta_0^{\nicefrac{1}{2}}}{L_0^{\nicefrac{1}{2}}K^{\nicefrac{1}{2}}} \rho\sigma + \frac{L_0^{\nicefrac{1}{4}}\Delta_0^{\nicefrac{1}{4}}(\rho \sigma)^{\nicefrac{1}{2}}}{K^{\nicefrac{1}{4}}} + \frac{\rho\sigma}{K^{\nicefrac{1}{2}}}\right).
    \end{align*}
    Therefore, iteration complexity is defined as
    \begin{equation*}
        K_{\varepsilon} = \cO\left(\frac{L_0\Delta_0}{\varepsilon^2} + \frac{L_1\Delta_0}{\varepsilon} + \frac{L_1^2\Delta_0(\rho\sigma)^2}{L_0\varepsilon^2} + \frac{L_0\Delta_0(\rho\sigma)^2}{\varepsilon^4} + \frac{(\rho\sigma)^2}{\varepsilon^2}\right).
    \end{equation*}
\end{proof}
\newpage
\section{Proof of \Cref{thm:iteration_complexity_param_agnostic}}
\label{proof:iteration_complexity_param_agnostic}
We next turn to the parameter-agnostic regime, where the momentum and threshold schedules vary with the iteration index. The proof follows the same general template as in the fixed-parameter case, but the time-varying parameters require more careful control of the accumulated error terms.
Let us first restate the theorem before proving it.
\begin{restate-theorem}{\ref{thm:iteration_complexity_param_agnostic}}
     Suppose Assumptions~\ref{ass:lower_bound}--\ref{ass:gen_smooth} hold.
    Run \Cref{alg} with $R_k=\nicefrac{1}{\alpha_k}$, $\alpha_{\mathrm{init}} = \alpha_0=1$, and
    $\alpha_k=k^{-\nicefrac{1}{2}}$ for $k\ge 1$.
    In the first two cases, let $\eta>0$ be any constant and define
    \begin{equation*}
        \Psi(L_0,L_1)
        \coloneqq
        e^{L_1^2\eta^2}\nicefrac{\Delta_0}{\eta}
        + \rho\sigma
        + e^{L_1^2\eta^2}L_0\eta .
    \end{equation*}
    Choose the stepsizes as follows, and let $\Psi$ denote the corresponding value in the final bound:
    \begin{itemize}
        \item If $L_1=0$, take $\eta_k=\frac{\eta}{(k+1)^{\nicefrac{3}{4}}}$ and
        $\Psi=\Psi(L_0,0)=\nicefrac{\Delta_0}{\eta}+\rho\sigma+L_0\eta$.
        \item If $L_1>0$, take $\eta_k=\frac{\eta}{17(k+1)^{\nicefrac{3}{4}}}$ and
        $\Psi=\Psi(L_0,L_1)$.
        \item If $L_1>0$ is known, take $\eta_k=\frac{1}{17L_1(k+1)^{\nicefrac{3}{4}}}$ and
        $\Psi=L_1\Delta_0+\rho\sigma+\nicefrac{L_0}{L_1}$, which is the preceding envelope with $\eta=\nicefrac{1}{L_1}$ up to universal constants.
    \end{itemize}
    Then, after $K\ge 1$ iterations,
    \begin{equation*}
        \min_{k\in\{0,\dots,K-1\}}\Exp{\norm{\nabla f(x_k)}_*}
        \le
        \cO\left(\frac{\Psi\log K}{K^{\nicefrac{1}{4}}}\right).
    \end{equation*}
    Consequently, the number of iterations needed to reach an $\varepsilon$-stationary point satisfies
    \begin{equation*}
        K_{\varepsilon}
        =
        \cO\left(
            \frac{\Psi^4}{\varepsilon^4}
            \left(\log\frac{1}{\varepsilon}\right)^4
        \right).
    \end{equation*}
\end{restate-theorem}
\begin{proof}
    We start with \Cref{lem:descent_lemma}
    \begin{eqnarray*}
        \sum^{K-1}_{k=0}\eta_k\Exp{\norm{\nabla f(x_k)}_*} &\leq& \Delta_0 + \underbrace{\frac{L_0}{2}\sum^{K-1}_{k=0} \eta^2_ke^{L_1\eta_k}}_{\eqdef \cT_1} + \underbrace{2\sum^{K-1}_{k = 0}\eta_k \Exp{\norm{\hat{e}_{k+1}}_*}}_{\eqdef \cT_2} \\
        && + \underbrace{\frac{L_1}{2}\sum^{K-1}_{k=0}\eta_k^2 e^{L_1\eta_k}\Exp{\norm{\nabla f(x_k)}_*}}_{\eqdef \cT_3} .
    \end{eqnarray*}
    We need to bound three terms $\cT_1, \cT_2, \cT_3$.
    \paragraph{Bounding $\cT_1$.}
        The term $\cT_1$ captures the accumulated error dependent purely on the stepsize sequence. To bound this, we substitute our chosen stepsize $\eta_k = \frac{\hat{\eta}}{(k+1)^{\nicefrac{3}{4}}}$ and upper-bound the resulting summation using an integral approximation. Specifically, we have:
    \begin{align*}
        \cT_1 &= \frac{L_0}{2}\sum^{K-1}_{k=0} \frac{\hat{\eta}^2}{(k+1)^{\nicefrac{3}{2}}}\exp\left(L_1\eta_k\right) \\
        &\leq \frac{L_0\hat{\eta}^2\exp\left(L_1\hat{\eta}\right)}{2}\left(1 + \sum^{K-1}_{k=1} \frac{1}{(k+1)^{\nicefrac{3}{2}}}\right)\\
        &\leq \frac{L_0\hat{\eta}^2\exp\left(L_1\hat{\eta}\right)}{2}\left(1 + \int^{K}_{1} \frac{1}{t^{\nicefrac{3}{2}}} \, dt \right) \\
        &= \frac{L_0\hat{\eta}^2\exp\left(L_1\hat{\eta}\right)}{2}\left(1 + \left[ -\frac{2}{\sqrt{t}} \right]_1^K \right) \\
        &= \frac{L_0\hat{\eta}^2\exp\left(L_1\hat{\eta}\right)}{2}\left(3 - \frac{2}{\sqrt{K}}\right).
    \end{align*}
    Thus, by dropping the negative term, we derive the following conservative upper bound:
    \begin{equation}
    \label{eq:bound_on_T_1}
        \cT_1 \leq \frac{3}{2}L_0\hat{\eta}^2\exp\left(L_1\hat{\eta}\right).
    \end{equation}
    \paragraph{Bounding $\cT_2$.}
    Next, we analyze $\cT_2$, which encapsulates the accumulated tracking error. By applying Lemma~\ref{lem:error_bound} (set $\alpha_{\mathrm{init}} = \alpha_0$), we can expand this error. To facilitate a tight convergence analysis, we systematically decompose the resulting expression into five distinct components ($\circledOne$~through $\circledFive$). This isolates the effects of the objective's smoothness, the momentum weights, and the variance:
    \begin{align*}
        \cT_2 &\leq 2\sum^{K-1}_{k=0} \eta_k\sum_{t=1}^k\prod_{j=t}^{k}(1-\alpha_j) \left(L_0 + L_1\Exp{\norm{\nabla f(x_t)}_*}\right)\exp(L_1\eta_{t-1})\eta_{t-1}\\
        &\quad + 2\sum^{K-1}_{k=0} \eta_k\sum_{t=1}^k \prod_{j=t+1}^{k}(1-\alpha_j) \alpha_t \left(L_0 + L_1 \Exp{\norm{\nabla f(x_t)}_*}\right)\exp\left(L_1 R_t\eta_{t-\delta_t}\right)R_t\eta_{t-\delta_t}\\
        &\quad + 2\rho \sigma\sum^{K-1}_{k=0} \eta_k\sqrt{\sum_{t=0}^k \prod_{j=t+1}^{k}(1-\alpha_j)^2 \alpha_t^2} \\
        &=  \circledOne + \circledTwo +  \circledThree + \circledFour + \circledFive~,
    \end{align*}
    where $\circledOne, \circledTwo, \circledThree, \circledFour \text{ and } \circledFive$ are defined as
    \begin{align*}
        \circledOne &\eqdef 2L_0\sum^{K-1}_{k=0}\eta_k\sum_{t=1}^k\prod_{j=t}^{k}(1-\alpha_j)\exp(L_1\eta_{t-1})\eta_{t-1};\\
        \circledTwo &\eqdef 2L_0\sum^{K-1}_{k=0}\eta_k\sum_{t=1}^k\prod_{j=t+1}^{k}(1-\alpha_j)\alpha_t\exp\left(L_1 R_t\eta_{t-\delta_t}\right)R_t\eta_{t-\delta_t};\\
        \circledThree &\eqdef  2L_1\sum^{K-1}_{k=0}\eta_k \sum_{t=1}^k\prod_{j=t}^{k}(1-\alpha_j) \exp(L_1\eta_{t-1})\eta_{t-1} \Exp{\norm{\nabla f(x_t)}_*};\\
        \circledFour &\eqdef 2L_1\sum^{K-1}_{k=0} \eta_k\sum_{t=1}^k \prod_{j=t+1}^{k}(1-\alpha_j) \alpha_t \exp\left(L_1 R_t\eta_{t-\delta_t}\right)R_t\eta_{t-\delta_t} \Exp{\norm{\nabla f(x_t)}_*};\\
        \circledFive &\eqdef 2\rho \sigma\sum^{K-1}_{k=0} \eta_k\sqrt{\sum_{t=0}^k \prod_{j=t+1}^{k}(1-\alpha_j)^2 \alpha_t^2}.
    \end{align*}
    We now proceed to bound these terms individually.
    \paragraph{Bounding $\circledOne$.}
   The first term, $\circledOne$, represents the deterministic error discounted by our momentum parameter.
   By substituting our stepsize $\eta_k$ and momentum parameter $\alpha_j = j^{-1/2}$, we can bound the nested products.
   We then control the resulting harmonic series by upper-bounding it with a logarithmic integral:
    \begin{align*}
        \circledOne &\leq 2L_0\hat{\eta}^2\exp(L_1\hat{\eta})\sum^{K-1}_{k=0}\sum_{t=1}^k\prod_{j=t}^{k}(1-j^{-1/2})t^{-3/4}(k+1)^{-3/4}\\
        &\leq 2L_0\hat{\eta}^2\exp(L_1\hat{\eta}) \sum^{K-1}_{k=0} \frac{18}{k+1}\\
        &\leq 36L_0\hat{\eta}^2\exp(L_1\hat{\eta}) \left(\int^{K}_{1} \frac{1}{t} \, dt +1\right)\\
        &\leq 36L_0\hat{\eta}^2\exp(L_1\hat{\eta}) \left(\log K +1\right).
    \end{align*}
    \paragraph{Bounding $\circledTwo$.}
    The second term tracks the error impacted by the delay parameter $\delta_t \leq R_t$. Recalling that $R_t = \nicefrac{1}{\alpha_t}$ for $t \geq 1$, we have $\alpha_t R_t = 1$. We first bound the exponential term:
    \begin{align*}
        \exp\left(L_1 R_t \eta_{t-\delta_t}\right) 
        &= \exp\left(L_1 \frac{1}{\alpha_t} \eta_{t-\delta_t}\right)\\
        &\leq \exp\left(L_1 \hat{\eta} \frac{\sqrt{t}}{(t-\delta_t+1)^{\nicefrac{3}{4}}}\right) \\
        &\leq \exp\left(L_1 \hat{\eta} \frac{\sqrt{t}}{(t-\sqrt{t}+1)^{\nicefrac{3}{4}}}\right).
    \end{align*}
    Since $\frac{\sqrt{t}}{(t-\sqrt{t}+1)^{\nicefrac{3}{4}}} \leq 1.03$ for all $t \geq 1$, we can upper-bound this as:
    \begin{equation*}
        \exp\left(L_1 \hat{\eta} \frac{\sqrt{t}}{(t-\sqrt{t}+1)^{\nicefrac{3}{4}}}\right) \leq \exp(1.03 L_1 \hat{\eta}).
    \end{equation*}
    To continue bounding $\circledTwo$, we need to bound the following term:
    \begin{align*}
       \sum_{t=1}^k \prod_{j=t+1}^{k}(1-\alpha_j) \alpha_t R_t\eta_{t-\delta_t}
        &=  \hat{\eta} \sum_{t=1}^k \prod_{j=t+1}^{k}(1-j^{-\nicefrac{1}{2}}) (t-\delta_t+1)^{-\nicefrac{3}{4}}\\
        &\le \hat{\eta} \sum_{t=1}^k \prod_{j=t+1}^{k}(1-j^{-\nicefrac{1}{2}}) (t-R_t+1)^{-\nicefrac{3}{4}}\\
        &\le \hat{\eta} \sum_{t=1}^k \prod_{j=t+1}^{k}(1-j^{-\nicefrac{1}{2}}) (t-\sqrt{t}+1)^{-\nicefrac{3}{4}} \\
        &\le 2\hat{\eta} \sum_{t=1}^k \prod_{j=t+1}^{k}(1-j^{-\nicefrac{1}{2}}) t^{-\nicefrac{3}{4}} \\
        &\overset{(\ast)}{\le} 4\hat{\eta}\exp\left(\frac{1}{1-\nicefrac{1}{2}}\right) (k+1)^{\nicefrac{1}{2} - \nicefrac{3}{4}}\\
        &\le \frac{36\hat{\eta}}{(k+1)^{\nicefrac{1}{4}}},
    \end{align*}
    where in $(\ast)$ we used inequality \eqref{eq:sum_product_bound_3} from \Cref{lem:sum_product_bound}.

    Using the previous inequality, we bound the resulting sum similar to $\circledOne$:
    \begin{align*}
        \circledTwo &\leq 2 L_0 \exp(1.03 L_1 \hat{\eta}) \sum_{k=0}^{K-1} \eta_k \sum_{t=1}^{k} \prod_{j=t+1}^{k} (1-\alpha_j) \eta_{t-\delta_t} \\
        &\leq 2 L_0 \exp(1.03 L_1 \hat{\eta}) \sum_{k=0}^{K-1} \eta_k \cdot \frac{36\hat{\eta}}{(k+1)^{\nicefrac{1}{4}}} \\
        &\leq 72 L_0 \hat{\eta}^2 \exp(1.03 L_1 \hat{\eta}) (\log K + 1).
    \end{align*}
    \paragraph{Bounding $\circledThree$.}
    For $\circledThree$, we swap the order of summation. Notice that $\eta_{t-1} = \eta_t \frac{(t+1)^{\nicefrac{3}{4}}}{t^{\nicefrac{3}{4}}} \leq 2^{\nicefrac{3}{4}} \eta_t$. Rearranging the sums yields:
    \begin{align*}
        \circledThree &= 2 \sum_{k=1}^{K-1} \eta_k \sum_{t=1}^{k} \prod_{j=t}^{k} (1-\alpha_j) L_1 \Exp{\norm{\nabla f(x_t)}_*} \exp(L_1 \eta_{t-1}) \eta_{t-1} \\
        &\le 2 \sum_{t=1}^{K-1} \eta_{t-1} \exp(L_1 \eta_{t-1}) L_1 \Exp{\norm{\nabla f(x_t)}_*} \sum_{k=t}^{K-1} \eta_k \prod_{j=t+1}^{k} (1-\alpha_j).
    \end{align*}
    For $t\ge 1$, using the empty-product convention and inequality \eqref{eq:sum_product_bound_2} from \Cref{lem:sum_product_bound}, we have
    \begin{align*}
        \sum_{k=t}^{K-1}\eta_k\prod_{j=t+1}^{k}(1-\alpha_j)
        &=
        \eta_t
        +
        \sum_{k=t+1}^{K-1}\eta_k\prod_{j=t+1}^{k}(1-\alpha_j)\\
        &\le
        \frac{\hat{\eta}}{(t+1)^{\nicefrac{3}{4}}}
        +
        3\hat{\eta} t^{-\nicefrac{1}{4}}
        \le
        4\hat{\eta}t^{-\nicefrac{1}{4}}.
    \end{align*}
    Therefore,
    \begin{align*}
        \circledThree &\leq 8 L_1 \sum_{t=1}^{K-1} \eta_{t-1} \exp(L_1 \eta_{t-1}) \frac{\hat{\eta}}{t^{\nicefrac{1}{4}}} \Exp{\norm{\nabla f(x_t)}_*} \\
        &\leq 8 L_1 \sum_{k=1}^{K-1} \eta_{k-1} \exp(L_1 \eta_{k-1}) \frac{\hat{\eta}}{k^{\nicefrac{1}{4}}} \Exp{\norm{\nabla f(x_k)}_*}.
    \end{align*}
    \paragraph{Bounding $\circledFour$.}
    By the inequality \eqref{eq:sum_product_bound_2} from Lemma~\ref{lem:sum_product_bound},
    \begin{align*}
        \sum_{k=t}^{K-1} \prod_{j=t+1}^{k} (1-\alpha_j) \eta_k & = \eta_{t} + \sum_{k=t+1}^{K-1} \prod_{j=t+1}^{k} (1-\alpha_j) \eta_k\\
        &= \hat{\eta} (t+1)^{-\nicefrac{3}{4}} + \hat{\eta}\sum_{k=t+1}^{K-1} \prod_{j=t+1}^{k} (1-j^{-\nicefrac{1}{2}}) (k+1)^{-\nicefrac{3}{4}} \le 4\hat{\eta}t^{-\nicefrac{1}{4}} .
    \end{align*}
    We apply the same sum-swapping technique and the previous inequality:
    \begin{align*}
        \circledFour &= 2 L_1 \sum_{t=1}^{K-1} \sum_{k=t}^{K-1} \prod_{j=t+1}^{k} (1-\alpha_j) \eta_k \left[ \eta_{t-\delta_t} \exp\left(L_1 \frac{\eta_{t-\delta_t}}{\alpha_t}\right) \right] \Exp{\norm{\nabla f(x_t)}_*} \\
        &\leq 8 L_1 \sum_{t=1}^{K-1} \eta_{t-\delta_t} \exp\left(L_1 \hat{\eta}\frac{ \sqrt{t}}{(t-\delta_t+1)^{\nicefrac{3}{4}}}\right) \frac{\hat{\eta}}{t^{\nicefrac{1}{4}}} \Exp{\norm{\nabla f(x_t)}_*} \\
        &\leq 8 L_1 \sum_{t=1}^{K-1} \eta_{t-\delta_t} \exp\left(L_1 \hat{\eta}\frac{\sqrt{t}}{(t-\sqrt{t}+1)^{\nicefrac{3}{4}}}\right) \frac{\hat{\eta}}{t^{\nicefrac{1}{4}}} \Exp{\norm{\nabla f(x_t)}_*} \\
        &\leq 8 L_1 \sum_{k=1}^{K-1} \eta_{k-\delta_k} \exp\left(2 L_1 \hat{\eta} \frac{1}{k^{\nicefrac{1}{4}}}\right) \frac{\hat{\eta}}{k^{\nicefrac{1}{4}}} \Exp{\norm{\nabla f(x_k)}_*},
    \end{align*}
    where in the last inequality we used $\frac{\sqrt{t}}{(t-\sqrt{t}+1)^{\nicefrac{3}{4}}} \leq \frac{2^{\nicefrac{3}{4}}}{t^{\nicefrac{1}{4}}}$.

    \paragraph{Bounding $\circledFive$.}
    The variance term $\circledFive$ involves the standard summation of squared momentum weights. Expanding this geometric-like sum gives the classical logarithmic accumulation:
    \begin{align*}
        \circledFive &= 2\rho \sigma\sum^{K-1}_{k=0} \eta_k\sqrt{\sum_{t=1}^k \prod_{j=t+1}^{k}(1-\alpha_j)^2 \alpha_t^2 + \prod_{j=1}^{k}(1-\alpha_j)^2 \alpha_0^2} \\
        &=  2\rho \sigma\sum^{K-1}_{k=1} \eta_k\sqrt{\sum_{t=1}^k \prod_{j=t+1}^{k}(1-j^{-\nicefrac{1}{2}}) t^{-1} + \prod_{j=1}^{k}(1-j^{-\nicefrac{1}{2}})^2 } +  2\rho \sigma\hat{\eta} \\
        &\leq 2\rho \sigma\sum^{K-1}_{k=1} \eta_k\sqrt{2\exp\left(\frac{1}{1-\nicefrac{1}{2}}\right) \frac{1}{\sqrt{k+1}} } +  2\rho \sigma\hat{\eta}\\
        &\leq 2\rho \sigma\sum^{K-1}_{k=1} \frac{\hat{\eta}}{(k+1)^{\nicefrac{3}{4}}}\sqrt{\frac{18}{(k+1)^{\nicefrac{1}{2}}}}  +  2\rho \sigma\hat{\eta}\\
        &\leq 11 \rho \sigma \hat{\eta} (\log K + 1),
    \end{align*}
    where in the first inequality we used the inequality \eqref{eq:sum_product_bound_3} from \Cref{lem:sum_product_bound}.
    \paragraph{Combining the terms.}
    We now piece together $\cT_1, \cT_2,$ and $\cT_3$. Notice that $$\cT_3 = \frac{1}{2}\sum_{k=0}^{K-1} L_1 \eta_k^2 e^{L_1 \eta_k} \Exp{\norm{\nabla f(x_k)}_*}.$$ Combining $\cT_3$ with $\circledThree$ and $\circledFour$, we can group the coefficients in front of $\Exp{\norm{\nabla f(x_k)}_*}$.

    Bounding the ratios of sequential stepsizes $\frac{(k+1)^{\nicefrac{3}{4}}}{k^{\nicefrac{3}{4}}} \leq 2^{\nicefrac{3}{4}}$ and $\frac{(k+1)^{\nicefrac{3}{4}}}{(k-\delta_k+1)^{\nicefrac{3}{4}}} \leq 2^{\nicefrac{3}{4}}$, we consolidate the coefficients into a single envelope:
    \begin{equation*}
        \cT_3 + \circledThree + \circledFour \leq \sum_{k=0}^{K-1} L_1 \eta_k \left( 33 \hat{\eta} \frac{1}{(k+1)^{\nicefrac{1}{4}}} \exp\left(3 L_1 \hat{\eta} (k+1)^{-\nicefrac{1}{4}}\right) \right) \Exp{\norm{\nabla f(x_k)}_*}.
    \end{equation*}
    Substituting this back into the main inequality gives:
    \begin{align*}
        \sum_{k=0}^{K-1}\eta_k\Exp{\norm{\nabla f(x_k)}_*} &\leq \Delta_0 + \frac{3}{2} L_0 \hat{\eta}^2 e^{L_1\hat{\eta}} + 36 L_0 \hat{\eta}^2 e^{L_1\hat{\eta}} (\log K + 1) \\
        &\quad + 72 L_0 \hat{\eta}^2 e^{1.03 L_1\hat{\eta}} (\log K + 1) + 11 \rho \sigma \hat{\eta} (\log K + 1) \\
        &\quad + \sum_{k=0}^{K-1} 33 L_1 \eta_k \frac{\hat{\eta}}{(k+1)^{\nicefrac{1}{4}}} \exp\left(3 L_1 \hat{\eta} (k+1)^{-\nicefrac{1}{4}}\right) \Exp{\norm{\nabla f(x_k)}_*}.
    \end{align*}

    \paragraph{Thresholding trick and absorbing the sum.}
    To handle the recursive expected gradient term on the right-hand side, we split the summation into two parts: a transient phase ($k < k_0$) and a stable phase ($k \geq k_0$). We define the threshold iteration $k_0$ such that the multiplicative factor becomes less than or equal to $\nicefrac{1}{2}$, allowing us to absorb it into the left-hand side. We require:
    \begin{equation*}
        33 L_1 \frac{\hat{\eta}}{(k+1)^{\nicefrac{1}{4}}} \exp\left(3 L_1 \hat{\eta} (k+1)^{-\nicefrac{1}{4}}\right) \leq \frac{1}{2}.
    \end{equation*}
    Because $x \exp(c x)$ is monotonically increasing, we can secure this bound for all $k \geq k_0$ if we conservatively choose:
    \begin{equation*}
        k_0 \eqdef \lceil (70 L_1 \hat{\eta})^4 \rceil - 1.
    \end{equation*}
    For iterations $k \geq k_0$, we subtract $\frac{1}{2}\sum_{k=k_0}^{K-1}\eta_k\Exp{\norm{\nabla f(x_k)}_*}$ from both sides. We are left with the task of explicitly bounding the transient sum from $k=0$ to $k_0-1$.

    First, we bound the local travel distance (radius) $r_k = \norm{x_k - x_0}$. Summing the stepsizes and approximating via an integral:
    \begin{equation*}
        r_k \leq \sum_{\tau=0}^{k-1} \eta_\tau = \hat{\eta} \sum_{\tau=1}^{k} \frac{1}{\tau^{\nicefrac{3}{4}}} \leq \hat{\eta} \left( 1 + \int_1^k \frac{1}{t^{\nicefrac{3}{4}}} \, dt \right) = \hat{\eta} \left( 1 + \left[4t^{\nicefrac{1}{4}}\right]_1^k \right) = 4\hat{\eta} k^{\nicefrac{1}{4}} - 3\hat{\eta} \leq 4\hat{\eta} k^{\nicefrac{1}{4}}.
    \end{equation*}
    Next, we bound $\Exp{\norm{\nabla f(x_k)}_*}$ using the generalized $(L_0, L_1)$-smoothness condition and the triangle inequality:
    \begin{align*}
        \norm{\nabla f(x_k)}_* &\leq \norm{\nabla f(x_k) - \nabla f(x_0)}_* + \norm{\nabla f(x_0)}_* \\
        &\leq \left(L_0 + L_1 \norm{\nabla f(x_0)}_*\right) \exp\left(L_1 \norm{x_k - x_0}\right) \norm{x_k - x_0} + \norm{\nabla f(x_0)}_* \\
        &\leq L_0 r_k e^{L_1 r_k} + L_1 r_k e^{L_1 r_k} \norm{\nabla f(x_0)}_* + \norm{\nabla f(x_0)}_*.
    \end{align*}
    Substituting this gradient bound into our transient sum, we decompose it into three distinct sums to bound individually:
    \begin{align*}
        & \sum_{k=0}^{k_0-1} 33 L_1 \eta_k \frac{\hat{\eta}}{(k+1)^{\nicefrac{1}{4}}} \exp\left(3 L_1 \hat{\eta} (k+1)^{-\nicefrac{1}{4}}\right) \Exp{\norm{\nabla f(x_k)}_*} \leq \\
        &\quad \underbrace{33 L_1 \sum_{k=1}^{k_0-1} \frac{\hat{\eta}^2}{(k+1)} \exp\left(\frac{3 L_1 \hat{\eta}}{(k+1)^{\nicefrac{1}{4}}}\right) L_0 r_k e^{L_1 r_k}}_{\eqdef \circledSix} \\
        &\quad + \underbrace{33 L_1 \sum_{k=1}^{k_0-1} \frac{\hat{\eta}^2}{(k+1)} \exp\left(\frac{3 L_1 \hat{\eta}}{(k+1)^{\nicefrac{1}{4}}}\right) L_1 r_k e^{L_1 r_k} \Exp{\norm{\nabla f(x_0)}_*}}_{\eqdef \circledSeven} \\
        &\quad + \underbrace{33 L_1 \sum_{k=0}^{k_0-1} \frac{\hat{\eta}^2}{(k+1)} \exp\left(\frac{3 L_1 \hat{\eta}}{(k+1)^{\nicefrac{1}{4}}}\right) \Exp{\norm{\nabla f(x_0)}_*}}_{\eqdef \circledEight}.
    \end{align*}

    \paragraph{Bounding $\circledSix$.}
    We substitute the bound $r_k \leq 4\hat{\eta} k^{\nicefrac{1}{4}} - 3\hat{\eta}$ and factor out the constants. Noting that $\frac{k^{\nicefrac{1}{4}}}{k+1} \leq \frac{1}{(k+1)^{\nicefrac{3}{4}}}$, we approximate the summation with an integral up to $k_0$:
    \begin{align*}
        \circledSix &\leq 33 L_1 \sum_{k=1}^{k_0-1} \frac{\hat{\eta}^2}{(k+1)} \exp\left(\frac{3 L_1 \hat{\eta}}{(k+1)^{\nicefrac{1}{4}}}\right) L_0 (4 \hat{\eta} k^{\nicefrac{1}{4}} - 3\hat{\eta}) \exp(L_1 (4 \hat{\eta} k^{\nicefrac{1}{4}} - 3\hat{\eta})) \\
        &\leq 132 L_1 L_0 \hat{\eta}^3 \sum_{k=1}^{k_0-1} \frac{1}{(k+1)^{\nicefrac{3}{4}}} \exp(4 L_1 \hat{\eta} k^{\nicefrac{1}{4}}) \\
        &\leq 132 L_1 L_0 \hat{\eta}^3 \exp(4 L_1 \hat{\eta} k_0^{\nicefrac{1}{4}}) \int_1^{k_0} \frac{1}{t^{\nicefrac{3}{4}}} \, dt \\
        &\leq 528 L_1 L_0 \hat{\eta}^3 \exp(4 L_1 \hat{\eta} k_0^{\nicefrac{1}{4}}) k_0^{\nicefrac{1}{4}}.
    \end{align*}
    Since $k_0^{\nicefrac{1}{4}} \le 70 L_1 \hat{\eta}$, we obtain the coarse uniform bound:
    \begin{equation*}
        \circledSix \leq 36960 L_1^2 \hat{\eta}^2 e^{280 \hat{\eta}^2 L_1^2} L_0 \hat{\eta}^2.
    \end{equation*}

    \paragraph{Bounding $\circledSeven$.}
    The second term mimics the structure of the first but scales with $L_1$ and the initial gradient norm. Following the exact same integral approximation:
    \begin{align*}
        \circledSeven &\leq 33 L_1^2 \hat{\eta}^2 \sum_{k=1}^{k_0-1} \frac{1}{(k+1)} (4 \hat{\eta} k^{\nicefrac{1}{4}} - 3\hat{\eta}) \exp\left(L_1 \hat{\eta} \left(\frac{3}{(k+1)^{\nicefrac{1}{4}}} + 4 k^{\nicefrac{1}{4}} - 3\right)\right) \Exp{\norm{\nabla f(x_0)}_*} \\
        &\leq 132 L_1^2 \hat{\eta}^3 \sum_{k=1}^{k_0-1} \frac{1}{(k+1)^{\nicefrac{3}{4}}} \exp(L_1 \hat{\eta} \cdot 4 k^{\nicefrac{1}{4}}) \Exp{\norm{\nabla f(x_0)}_*} \\
        &\leq 132 L_1^2 \hat{\eta}^3 \exp(4 L_1 \hat{\eta} k_0^{\nicefrac{1}{4}}) \int_1^{k_0} \frac{1}{t^{\nicefrac{3}{4}}} \, dt \ \Exp{\norm{\nabla f(x_0)}_*} \\
        &\leq 528 L_1^2 \hat{\eta}^3 \exp(4 L_1 \hat{\eta} k_0^{\nicefrac{1}{4}}) k_0^{\nicefrac{1}{4}} \Exp{\norm{\nabla f(x_0)}_*}.
    \end{align*}
    Again, evaluating at the threshold $k_0$, we obtain:
    \begin{equation*}
        \circledSeven \leq 36960 L_1^3 \hat{\eta}^4 \exp(280 L_1^2 \hat{\eta}^2) \Exp{\norm{\nabla f(x_0)}_*}.
    \end{equation*}

    \paragraph{Bounding $\circledEight$.}
    The third term lacks the $r_k$ multiplier, rendering it a pure harmonic series weighted by an exponential. Bounding the exponential globally by its value at $k=0$ and applying the standard harmonic bound $\sum_{k=0}^{k_0-1} \frac{1}{k+1} \leq \log(k_0) + 1$:
    \begin{align*}
        \circledEight &\leq 33 L_1 \hat{\eta}^2 \exp(3 L_1 \hat{\eta}) \sum_{k=0}^{k_0-1} \frac{1}{(k+1)} \Exp{\norm{\nabla f(x_0)}_*} \\
        &\leq 33 L_1 \hat{\eta}^2 \exp(3 L_1 \hat{\eta}) (\log k_0 + 1) \Exp{\norm{\nabla f(x_0)}_*} \\
        &\leq 132 L_1 \hat{\eta}^2 \exp(3 L_1 \hat{\eta}) \max\{1, \log\lceil 70 L_1 \hat{\eta} \rceil\} \Exp{\norm{\nabla f(x_0)}_*}.
    \end{align*}

    \paragraph{Final assembly of the bound.}
    Combining $\cT_1, \cT_2, \cT_3$, moving the $\frac{1}{2}\sum_{k=k_0}^{K-1} \eta_k \Exp{\norm{\nabla f(x_k)}_*}$ to the left-hand side, and inserting our explicit bounds for the transient phase ($\circledSix, \circledSeven, \circledEight$), we yield the global inequality:
    \begin{align*}
        \frac{1}{2} \sum_{k=0}^{K-1} \eta_k \Exp{\norm{\nabla f(x_k)}_*} &\leq \Delta_0 + \frac{3}{2} L_0 \hat{\eta}^2 e^{L_1\hat{\eta}} + 36 L_0 \hat{\eta}^2 e^{L_1\hat{\eta}} (\log K + 1) \\
        &\quad + 72 L_0 \hat{\eta}^2 e^{1.03 L_1\hat{\eta}} (\log K + 1) + 11 \rho  \sigma \hat{\eta} (\log K + 1) \\
        &\quad + 36960 L_1^2 \hat{\eta}^2 e^{280 \hat{\eta}^2 L_1^2} L_0 \hat{\eta}^2 \\
        &\quad + 36960 L_1^3 \hat{\eta}^4 e^{280 L_1^2 \hat{\eta}^2} \Exp{\norm{\nabla f(x_0)}_*} \\
        &\quad + 132 L_1 \hat{\eta}^2 e^{3 L_1 \hat{\eta}} \max\{1, \log\lceil 70 L_1 \hat{\eta} \rceil\} \Exp{\norm{\nabla f(x_0)}_*}.
    \end{align*}
    By Lemma~\ref{lem:gen_smooth}, we have
    \begin{align*}
        4\Delta_0 &\geq \frac{\norm{\nabla f(x_0)}_*^2}{L_0 + L_1\norm{\nabla f(x_0)}_*} \geq \begin{cases}
            \frac{\norm{\nabla f(x_0)}_*^2}{2L_0}, & \norm{\nabla f(x_0)}_*\leq \frac{L_0}{L_1};\\
            \frac{\norm{\nabla f(x_0)}_*}{2L_1}, & \text{otherwise.}
        \end{cases}
    \end{align*}
    Thus, we have
    \begin{equation*}
        \norm{\nabla f(x_0)}_* \leq \max\left\{8L_1\Delta_0, \frac{L_0}{L_1}\right\}.
    \end{equation*}
    Therefore, we obtain
    \begin{align*}
        \frac{1}{2} \sum_{k=0}^{K-1} \eta_k \Exp{\norm{\nabla f(x_k)}_*} &\leq \Delta_0 + \frac{3}{2} L_0 \hat{\eta}^2 e^{L_1\hat{\eta}} + 36 L_0 \hat{\eta}^2 e^{L_1\hat{\eta}} (\log K + 1) \\
        &\quad + 72 L_0 \hat{\eta}^2 e^{1.03 L_1\hat{\eta}} (\log K + 1) + 11 \rho  \sigma \hat{\eta} (\log K + 1) \\
        &\quad + 36960 L_1^2 \hat{\eta}^2 e^{280 \hat{\eta}^2 L_1^2} L_0 \hat{\eta}^2 \\
        &\quad + 295680 L_1^4 \hat{\eta}^4 e^{280 L_1^2 \hat{\eta}^2} \Delta_0 + 36960 L_1^2 \hat{\eta}^2 e^{280 L_1^2 \hat{\eta}^2} L_0\hat{\eta}^2 \\
        &\quad + 1056 L_1^2 \hat{\eta}^2 e^{3 L_1 \hat{\eta}} \max\{1, \log\lceil 70 L_1 \hat{\eta} \rceil\} \Delta_0 \\
        &\quad+ 132e^{3 L_1 \hat{\eta}} \max\{1, \log\lceil 70 L_1 \hat{\eta} \rceil\} L_0\hat{\eta}^2.
    \end{align*}

    \paragraph{Standard smoothness.} In this case we have $L_1 = 0$. Thus, taking $\hat{\eta} = \eta$ we have
    \begin{align*}
        \frac{1}{2} \sum_{k=0}^{K-1} \eta_k \Exp{\norm{\nabla f(x_k)}_*} &\leq \Delta_0 + 11 \rho  \sigma \eta (\log K + 1) + 228 L_0 \eta^2 (\log K + 1) ~.
    \end{align*}

    Dividing by $\sum_{k=0}^{K-1}\eta_k$, collecting the terms and consolidating coefficients using Big-$\cO$ notation, we arrive at the envelope $\Psi(L_0, 0) = \nicefrac{\Delta_0}{\eta} + \rho  \sigma + L_0\eta$.
    Finally, bounding the minimum gradient by the weighted average over $K$ steps:
    \begin{equation*}
        \min_{k \in \{0,\dots, K-1\}} \Exp{\norm{\nabla f(x_k)}_*} \leq \frac{1}{\sum_{k=0}^{K-1} \eta_k} \sum_{k=0}^{K-1} \eta_k \Exp{\norm{\nabla f(x_k)}_*} \leq \cO\left( \frac{\Psi(L_0,0) \log K}{K^{\nicefrac{1}{4}}} \right).
    \end{equation*}
    \paragraph{Generalized smoothness: $L_1$ is unknown.} In this case we have $L_1 > 0 $. Thus, taking $\hat{\eta} = \nicefrac{\eta}{17}$, we obtain
    \begin{align*}
        \sum_{k=0}^{K-1} \eta_k \Exp{\norm{\nabla f(x_k)}_*} &\leq \cO\left(e^{L_1^2\eta^2}\Delta_0 + \rho  \sigma \hat{\eta} (\log K + 1)  + e^{L_1^2\eta^2} L_0 \hat{\eta}^2 (\log K + 1) \right).
    \end{align*}
    Dividing by $\sum_{k=0}^{K-1}\eta_k$ and collecting the terms, we arrive at the envelope $\Psi(L_0, L_1)$. Finally, bounding the minimum gradient by the weighted average over $K$ steps:
    \begin{equation*}
        \min_{k \in \{0,\dots, K-1\}} \Exp{\norm{\nabla f(x_k)}_*} \leq \frac{1}{\sum_{k=0}^{K-1} \eta_k} \sum_{k=0}^{K-1} \eta_k \Exp{\norm{\nabla f(x_k)}_*} \leq \cO\left( \frac{\Psi(L_0,L_1) \log K}{K^{\nicefrac{1}{4}}} \right) ,
    \end{equation*}
    which completes the proof.

    \paragraph{Generalized smoothness: $L_1$ is known.} Thus, taking $\hat{\eta} = \nicefrac{1}{17L_1}$, we obtain
    \begin{align*}
        \sum_{k=0}^{K-1} \eta_k \Exp{\norm{\nabla f(x_k)}_*} &\leq \cO\left(\Delta_0 + \rho \frac{\sigma}{L_1}  (\log K + 1)  + \frac{L_0}{L_1^2} (\log K + 1) \right).
    \end{align*}
    Dividing by $\sum_{k=0}^{K-1}\eta_k$ and collecting the terms, we arrive at the envelope $\Psi(L_0, L_1) = L_1\Delta_0 + \rho \sigma + \frac{L_0}{L_1}$. Finally, bounding the minimum gradient by the weighted average over $K$ steps:
    \begin{equation*}
        \min_{k \in \{0,\dots, K-1\}} \Exp{\norm{\nabla f(x_k)}_*} \leq \frac{1}{\sum_{k=0}^{K-1} \eta_k} \sum_{k=0}^{K-1} \eta_k \Exp{\norm{\nabla f(x_k)}_*} \leq \cO\left( \frac{\Psi(L_0,L_1) \log K}{K^{\nicefrac{1}{4}}} \right),
    \end{equation*}
    which completes the proof.
\end{proof}

\newpage
\section{Time complexity analysis under fixed computation times}
We now prove the time complexity bounds under \Cref{assump:fixed_time}.
We start with the case of a fixed delay threshold.
\begin{restate-lemma}{\ref{lem:time_R}}[Duration of $R$ updates {\citep[Lemma 4.1]{maranjyan2025ringmaster}}]
    Let the delay threshold of \Cref{alg} be fixed as $R_k \equiv R$.
    Under \Cref{assump:fixed_time}, the time needed to complete any $R$ consecutive updates is at most
    \begin{equation*}
       t(R) \coloneqq 2 \min_{m \in [n]} H_m \left( 1 + \frac{R}{m} \right).
    \end{equation*}
\end{restate-lemma}
As an immediate corollary, if the threshold remains fixed at $R$ for the first $K$ updates, then partitioning these updates into $\lceil \nicefrac{K}{R} \rceil$ consecutive blocks and applying \Cref{lem:time_R} yields
\begin{equation*}
    T(K) \le \left\lceil \frac{K}{R} \right\rceil t(R) ~.
\end{equation*}
This blockwise estimate is the key ingredient in the proof of \Cref{thm:time_complexity}.
\subsection{Proof of \Cref{thm:time_complexity}}
\label{proof:time_complexity}
We restate the theorem for convenience.
\begin{restate-theorem}{\ref{thm:time_complexity}}[Fixed-parameter total time complexity]
    Under the assumptions and parameter choices of \Cref{thm:iteration_complexity}, under \Cref{assump:fixed_time}, the time complexity of reaching an $\varepsilon$-stationary point satisfies
    \begin{align*}
        T_\varepsilon
        =
        \mathcal{O}\left(
            \min_{m \in [n]} H_m
            \left(
                \frac{L_0\Delta_0}{\varepsilon^2}
                + \frac{L_1\Delta_0}{\varepsilon} 
                + \frac{\sqrt{L_0\Delta_0}}{\varepsilon}
                + \frac{L_0\Delta_0(\rho\sigma)^2}{m\varepsilon^4}
                + \frac{L_1^2\Delta_0(\rho\sigma)^2}{mL_0\varepsilon^2}
                + \frac{(\rho\sigma)^2}{m\varepsilon^2}
            \right)
        \right).
    \end{align*}
\end{restate-theorem}
\begin{proof}
    Define
    \begin{align*}
        A_\varepsilon &\coloneqq \frac{L_0\Delta_0}{\varepsilon^2},
        &
        B_\varepsilon &\coloneqq \frac{L_1\Delta_0}{\varepsilon},
        \\
        C_\varepsilon &\coloneqq \frac{L_0\Delta_0(\rho\sigma)^2}{\varepsilon^4},
        &
        D_\varepsilon &\coloneqq \frac{L_1^2\Delta_0(\rho\sigma)^2}{L_0\varepsilon^2},
        &
        E_\varepsilon &\coloneqq \frac{(\rho\sigma)^2}{\varepsilon^2}.
    \end{align*}
    and let $\overline{K}_\varepsilon \coloneqq A_\varepsilon+B_\varepsilon+C_\varepsilon+D_\varepsilon+E_\varepsilon$.
    By \Cref{thm:iteration_complexity}, it suffices to take
    $K_\varepsilon = c\overline{K}_\varepsilon$ for a universal constant $c\ge 1$.

    Recall that
    \begin{equation*}
        \alpha
        =
        \begin{cases}
            1, & \sigma=0,\\
            \min\left\{
                1,
                \frac{\Delta_0^{\nicefrac{1}{2}}L_0^{\nicefrac{1}{2}}}{\rho\sigma K_\varepsilon^{\nicefrac{1}{2}}}
            \right\}, & \sigma>0,
        \end{cases}
        \qquad
        R = \frac{1}{\alpha} ~.
    \end{equation*}
    We partition the $K_\varepsilon$ updates into $\lceil \nicefrac{K_\varepsilon}{R} \rceil$ consecutive blocks and apply \Cref{lem:time_R}.
    This gives
    \begin{equation*}
        T_\varepsilon
        \le
        \left\lceil \frac{K_\varepsilon}{R} \right\rceil
        t(R)
        =
        2\left\lceil K_\varepsilon\alpha \right\rceil
        \min_{m \in [n]} H_m
        \left(
            1 + \frac{1}{m\alpha}
        \right).
    \end{equation*}
    Up to constants hidden in $\mathcal{O}(\cdot)$, we can ignore the ceiling, so the last inequality gives
    \begin{equation*}
        T_\varepsilon
        =
        \mathcal{O}\left(
            \min_{m \in [n]} H_m
            \left(
                K_\varepsilon\alpha + \frac{K_\varepsilon}{m}
            \right)
        \right).
    \end{equation*}

    It remains to bound $K_\varepsilon\alpha$.
    Since $K_\varepsilon = c\overline{K}_\varepsilon$, it suffices to bound $\overline{K}_\varepsilon\alpha$.
    If $\sigma=0$, then $C_\varepsilon=D_\varepsilon = E_\varepsilon=0$ and $\alpha=1$, so
    $\overline{K}_\varepsilon\alpha=A_\varepsilon+B_\varepsilon$.
    If $\sigma>0$, then $\alpha\le 1$ and
    \begin{equation*}
        (A_\varepsilon+B_\varepsilon)\alpha
        \le
        A_\varepsilon+B_\varepsilon ~.
    \end{equation*}
    Moreover, since $c\ge 1$ and $C_\varepsilon,D_\varepsilon,E_\varepsilon$ are terms in $\overline{K}_\varepsilon$,
    we have $C_\varepsilon,D_\varepsilon,E_\varepsilon\le K_\varepsilon$, and
    \begin{align*}
        C_\varepsilon\alpha
        &\le
        C_\varepsilon
        \frac{\sqrt{L_0\Delta_0}}{\rho\sigma\sqrt{K_\varepsilon}}
        \le
        C_\varepsilon
        \frac{\sqrt{L_0\Delta_0}}{\rho\sigma\sqrt{C_\varepsilon}}
        =
        A_\varepsilon ~,
        \\
        D_\varepsilon\alpha
        &\le
        D_\varepsilon
        \frac{\sqrt{L_0\Delta_0}}{\rho\sigma\sqrt{K_\varepsilon}}
        \le
        D_\varepsilon
        \frac{\sqrt{L_0\Delta_0}}{\rho\sigma\sqrt{D_\varepsilon}}
        =
        B_\varepsilon ~, \\
        E_\varepsilon\alpha
        &\le
        E_\varepsilon
        \frac{\sqrt{L_0\Delta_0}}{\rho\sigma\sqrt{K_\varepsilon}}
        \le
        E_\varepsilon
        \frac{\sqrt{L_0\Delta_0}}{\rho\sigma\sqrt{E_\varepsilon}}
        =
        \frac{\sqrt{L_0\Delta_0}}{\varepsilon} =: G_\varepsilon ~.
    \end{align*}
    Hence $\overline{K}_\varepsilon\alpha=\mathcal{O}(A_\varepsilon+B_\varepsilon+G_\varepsilon)$, and therefore
    $K_\varepsilon\alpha=\mathcal{O}(A_\varepsilon+B_\varepsilon+G_\varepsilon)$.
    Also,
    \begin{equation*}
        \frac{K_\varepsilon}{m}
        =
        \mathcal{O}\left(
            \frac{A_\varepsilon+B_\varepsilon+C_\varepsilon+D_\varepsilon + E_\varepsilon}{m}
        \right)
        =
        \mathcal{O}\left(
            A_\varepsilon+B_\varepsilon+\frac{C_\varepsilon+D_\varepsilon+E_\varepsilon}{m}
        \right),
    \end{equation*}
    where the last equality uses $m\ge 1$.
    Combining the last three bounds gives the claimed bound.
\end{proof}

The parameter-agnostic case is slightly more delicate, since the threshold does not remain fixed throughout the run.
To handle this, we extend the same blockwise argument to a non-decreasing threshold sequence.
\begin{lemma}[Time for non-decreasing adaptive threshold]\label{lem:time_RK}
    Let $K\ge 1$, and let $R_k\ge 1$ be a non-decreasing delay threshold.
    For each integer $r \in \{\lfloor R_0 \rfloor, \dots, \lfloor R_{K-1} \rfloor\}$, let
    \begin{equation*}
        \mathcal{K}_r \coloneqq \{k\in\{0,\dots,K-1\} \mid \lfloor R_k \rfloor = r\}
    \end{equation*}
    denote the corresponding block of iterations, and define
    \begin{equation*}
        N_r \coloneqq \left\lceil \frac{|\mathcal{K}_r|}{r} \right\rceil .
    \end{equation*}
    Then the time needed by \Cref{alg} to complete its first $K$ updates satisfies
    \begin{equation*}
        T(K) \le \sum_{r=\lfloor R_0 \rfloor}^{\lfloor R_{K-1} \rfloor} N_r\, t(r) ~.
    \end{equation*}
\end{lemma}
\begin{proof}
    Split the iterations $\{0,1,\dots,K-1\}$ into the blocks $\mathcal{K}_r$.
    Since $R_k$ is non-decreasing, all iterations in each $\mathcal{K}_r$ appear next to each other in time.
    For a fixed $r$, partition $\mathcal{K}_r$ into
    \begin{equation*}
        N_r = \left\lceil \frac{|\mathcal{K}_r|}{r} \right\rceil
    \end{equation*}
    consecutive groups, each containing at most $r$ iterations.
    Throughout each such group, $\lfloor R_k\rfloor=r$, and hence $R_k\ge r$.
    Replacing the actual threshold by the fixed threshold $r$ can only make the rule more conservative.
    If a group contains exactly $r$ iterations, then \Cref{lem:time_R} shows that it is completed within time at most $t(r)$.
    If the last group contains fewer than $r$ iterations, its duration is still at most $t(r)$, since completing fewer than $r$ updates cannot take longer than completing $r$ updates under the same threshold.
    Therefore, the total time spent on all iterations in $\mathcal{K}_r$ is at most $N_r t(r)$.
    Summing this bound over all values of $r$ gives
    \begin{equation*}
        T(K) \le \sum_{r=\lfloor R_0 \rfloor}^{\lfloor R_{K-1} \rfloor} N_r\, t(r) ~.
    \end{equation*}
\end{proof}

\subsection{Proof of \Cref{lemma:time_sqrt}}\label{proof:sqrt_threshold}
We now specialize \Cref{lem:time_RK} to the square-root threshold schedule.
\begin{restate-lemma}{\ref{lemma:time_sqrt}}[Time complexity for square-root threshold]
    Let the delay threshold be $R_0=1$ and $R_k = \sqrt{k}$ for $k \ge 1$.
    Under \Cref{assump:fixed_time}, for any $K\ge 1$, the time needed to complete $K$ iterations satisfies
    \begin{equation*}
        T(K)
        = \mathcal{O}\left(
            \min_{m \in [n]} H_m \left( \sqrt{K} + \frac{K}{m} \right)
        \right).
    \end{equation*}
\end{restate-lemma}
\begin{proof}
    Apply \Cref{lem:time_RK} with $R_0=1$ and $R_k = \sqrt{k}$ for $k \ge 1$.
    For any integer $r \ge 1$ and any $k\ge 1$, the condition $\lfloor R_k \rfloor = r$ is equivalent to
    \begin{equation*}
        r^2 \le k < (r+1)^2.
    \end{equation*}
    The extra iteration $k=0$ belongs to $\mathcal{K}_1$.
    Therefore, for every $r\ge 1$,
    \begin{equation*}
        |\mathcal{K}_r|\le 2r+2 ~.
    \end{equation*}
    Hence,
    \begin{equation*}
        N_r = \left\lceil \frac{|\mathcal{K}_r|}{r} \right\rceil
        \le \left\lceil \frac{2r+2}{r} \right\rceil
        \le 4 ~.
    \end{equation*}
    Let $r_{\max} \coloneqq \lfloor R_{K-1}\rfloor$.
    Since $r_{\max}\le \sqrt{K}$, using \Cref{lem:time_RK} and \Cref{lem:time_R}, we obtain
    \begin{equation*}
        T(K) \le \sum_{r=1}^{r_{\max}} N_r\, t(r)
        \le \sum_{r=1}^{r_{\max}} 4\, t(r)
        = 8 \sum_{r=1}^{r_{\max}} \min_{m \in [n]} H_m \left( 1 + \frac{r}{m} \right) .
    \end{equation*}
    Since $\sum_r \min_m a_{r,m} \le \min_m \sum_r a_{r,m}$ for nonnegative $a_{r,m}$, it follows that
    \begin{align*}
        T(K)
        &\le 8 \min_{m \in [n]}
            H_m \sum_{r=1}^{r_{\max}} \left( 1 + \frac{r}{m} \right)
        \\
        &= 8 \min_{m \in [n]} \left(
            H_m r_{\max}
            + H_m \frac{r_{\max}(r_{\max}+1)}{2m}
        \right).
    \end{align*}
    We conclude that
    \begin{equation*}
        T(K)
        = \mathcal{O}\left( \min_{m \in [n]} H_m \left( \sqrt{K} + \frac{K}{m} \right) \right).
    \end{equation*}
    This proves the claim.
\end{proof}
\newpage
\section{Time complexity under time-varying computation rates}\label{sec:arbitrary_time}
\Cref{assump:fixed_time} assumes that each worker has a constant gradient-computation time throughout the run.
This leads to explicit time complexity bounds in terms of the harmonic means $H_m$, but it does not capture clusters whose worker speeds change over time because of interruptions or fluctuating load \citep{dutta2018slow,li2020federated, maranjyan2025gradskip}.
In this section we therefore move to a more general model with time-varying computation rates.
\subsection{Universal computation model}
\label{sec:universal_computation_model}
We use the universal computation model of \citet{tyurin2024tighttimecomplexitiesparallel}.
\begin{assumption}[Universal computation model]
    \label{assump:universal_time}
    For each worker $i \in [n]$, there exists a nonnegative function
    \begin{equation*}
        p_i : \R_{+} \to \R_{+}
    \end{equation*}
    that is continuous almost everywhere.
    For any $0 \le T_1 \le T_2$, the number of stochastic gradients completed by worker $i$ during the interval $[T_1,T_2]$ is
    \begin{equation*}
        N_i(T_1,T_2) \coloneqq \left\lfloor \int_{T_1}^{T_2} p_i(s)\,ds \right\rfloor .
    \end{equation*}
\end{assumption}
This assumption includes \Cref{assump:fixed_time} as a special case: if $p_i(s)=1/\tau_i$ for all $s \ge 0$, then
\begin{equation*}
    N_i(T_1,T_2) = \left\lfloor \frac{T_2-T_1}{\tau_i} \right\rfloor .
\end{equation*}
\subsection{Fixed-parameter schedule}
We begin with the fixed-parameter regime, which is the counterpart of \Cref{thm:time_complexity} under \Cref{assump:universal_time}.
Unlike in the fixed computation model, the duration of a block of $R$ updates now depends on the starting time of that block.
For this reason, the resulting time bound is recursive rather than explicit.
\begin{lemma}[Duration of $R$ updates under the universal computation model {\citep[Lemma 5.1]{maranjyan2025ringmaster}}]
    \label{lem:time_R_arbitrary}
    Assume \Cref{assump:universal_time}.
    Let the delay threshold of \Cref{alg} be fixed as $R_k \equiv R$.
    Assume that an iteration starts at time $T_0$.
    Then the next $R$ updates of \Cref{alg} are completed within elapsed time
    \begin{equation*}
           t(R;T_0) \eqdef \min \left\{t \geq 0 : \sum \limits_{i=1}^{n} \flr{\frac{1}{4} \int_{T_0}^{T_0 + t} p_i(s)\, ds} \geq R\right\}.
    \end{equation*}
\end{lemma}
\begin{theorem}[Fixed-parameter time complexity under the universal computation model]
    \label{thm:time_complexity_fixed_universal}
    Under the assumptions and parameter choices of \Cref{thm:iteration_complexity}, let $K_\varepsilon$ be the number of iterations needed to reach an $\varepsilon$-stationary point, and let $R=1/\alpha$ be the corresponding fixed delay threshold. Then, under \Cref{assump:universal_time}, the time complexity of reaching an $\varepsilon$-stationary point satisfies
    \begin{equation*}
        T_\varepsilon \le T_{\lceil K_\varepsilon/R \rceil},
    \end{equation*}
    where the sequence $\{T_k\}_{k \ge 0}$ is defined recursively by
    \begin{align*}
        T_0 &\coloneqq 0, \\
        T_k &\coloneqq T_{k-1} + t(R;T_{k-1}), \qquad k \ge 1 .
    \end{align*}
\end{theorem}
\begin{proof}
    Partition the first $K_\varepsilon$ updates into $\lceil K_\varepsilon/R \rceil$ consecutive blocks of length at most $R$. Applying \Cref{lem:time_R_arbitrary} blockwise yields the stated recursion and therefore the bound
    \begin{equation*}
        T_\varepsilon \le T_{\lceil K_\varepsilon/R \rceil}.
    \end{equation*}
\end{proof}
Here $K_\varepsilon$ denotes the iteration complexity from \Cref{thm:iteration_complexity}, namely, the number of updates needed to reach an $\varepsilon$-stationary point. Thus the theorem converts the iteration bound into a time bound by grouping these $K_\varepsilon$ updates into blocks of size $R$ and tracking their completion times through the recursion above. In the classical smooth case $L_1=0$, the iteration complexity from \Cref{thm:iteration_complexity} has the same $\varepsilon$-dependence as \ringmaster, up to the norm-equivalence factor $\rho$. In the Euclidean setting, where $\rho=1$, this recovers the \ringmaster scaling and hence the corresponding optimal time complexity.
We next extend \Cref{lem:time_R_arbitrary} to non-decreasing threshold sequences.
\subsection{Adaptive threshold schedule}
We now turn to the parameter-agnostic schedule, where the threshold increases with the iteration counter.
We first derive a generic recursive bound for arbitrary non-decreasing thresholds and then specialize it to the square-root rule $R_0=1$ and $R_k = \sqrt{k}$ for $k \ge 1$.
\begin{lemma}[Time for general adaptive threshold under the universal computation model]
    \label{lem:time_RK_arbitrary}
    Let $R_k$ be a non-decreasing delay threshold.
    For each integer $r \in \{\lfloor R_0 \rfloor, \dots, \lfloor R_{K-1} \rfloor\}$, let
    \begin{equation*}
        \mathcal{K}_r \coloneqq \{k \mid \lfloor R_k \rfloor = r\}
    \end{equation*}
    denote the block of iterations where the threshold is constant.
    For every such $r$, partition $\mathcal{K}_r$ into
    \begin{equation*}
        N_r \coloneqq \left\lceil \frac{|\mathcal{K}_r|}{r} \right\rceil
    \end{equation*}
    consecutive sub-blocks, each containing at most $r$ iterations.

    Let $B \coloneqq \sum_{r=\lfloor R_0 \rfloor}^{\lfloor R_{K-1} \rfloor} N_r$, and enumerate all these sub-blocks in chronological order.
    Denote by $r_j$ the threshold value attached to the $j$-th sub-block, where $j = 1, \dots, B$.
    Define the sequence $\{T_j\}_{j=0}^{B}$ recursively by
    \begin{align*}
        T_0 &\coloneqq 0, \\
        T_j &\coloneqq T_{j-1} + t(r_j; T_{j-1}), \qquad j = 1, \dots, B.
    \end{align*}
    Then the time needed to complete the first $K$ iterations of \Cref{alg} satisfies
    \begin{equation*}
        T(K) \le T_B.
    \end{equation*}
\end{lemma}
\begin{proof}
    Partition the iterations $\{0,1,\dots,K-1\}$ into disjoint blocks $\mathcal{K}_r$ according to the integer value of $\lfloor R_k \rfloor$.
    Since $R_k$ is non-decreasing, all iterations in each $\mathcal{K}_r$ appear next to each other in time.

    Fix some $r$ and consider one of its sub-blocks.
    By construction, this sub-block contains at most $r$ consecutive iterations, and throughout this sub-block the threshold satisfies
    \begin{equation*}
        r \le R_k < r+1.
    \end{equation*}
    Replacing the actual threshold by $r$ can only make the rule more conservative on this sub-block.
    Hence \Cref{lem:time_R_arbitrary} implies that, if this sub-block starts at time $T$, then all its iterations are completed within time at most $t(r;T)$.

    Now enumerate all sub-blocks in chronological order and let $r_j$ be the threshold associated with the $j$-th one.
    Starting from time $T_0 = 0$, the first sub-block is completed by time at most
    \begin{equation*}
        T_1 = T_0 + t(r_1;T_0).
    \end{equation*}
    Repeating the same argument inductively, if the first $j-1$ sub-blocks are completed by time at most $T_{j-1}$, then the $j$-th sub-block is completed by time at most
    \begin{equation*}
        T_j = T_{j-1} + t(r_j;T_{j-1}).
    \end{equation*}
    After all
    \begin{equation*}
        B = \sum_{r=\lfloor R_0 \rfloor}^{\lfloor R_{K-1} \rfloor} \left\lceil \frac{|\mathcal{K}_r|}{r} \right\rceil
    \end{equation*}
    sub-blocks are processed, all first $K$ iterations have been completed. Therefore,
    \begin{equation*}
        T(K) \le T_B.
    \end{equation*}
\end{proof}

This is the universal-model analogue of \Cref{lem:time_RK}: it reduces the analysis of a non-decreasing threshold sequence to a recursion over threshold-constant sub-blocks.

\begin{lemma}[Square-root threshold under the universal computation model]
    \label{lem:time_sqrt_arbitrary}
    Let $R_0=1$ and $R_k = \sqrt{k}$ for $k \ge 1$.
    Under \Cref{assump:universal_time}, the time needed to complete the first $K$ iterations of \Cref{alg} is bounded by
    \begin{equation*}
        T(K) \le S_{3\lfloor \sqrt{K} \rfloor + 1},
    \end{equation*}
    where the sequence $\{S_j\}_{j=0}^{3\lfloor \sqrt{K} \rfloor + 1}$ is defined recursively as
    \begin{align*}
        S_0 &\coloneqq 0, \\
        S_j &\coloneqq S_{j-1} + t\!\left(\max\left\{1,\left\lceil \frac{j-1}{3} \right\rceil\right\}; S_{j-1}\right), \qquad j = 1, \dots, 3\lfloor \sqrt{K} \rfloor + 1 .
    \end{align*}
\end{lemma}

\begin{proof}
    We apply \Cref{lem:time_RK_arbitrary} with $R_0=1$ and $R_k = \sqrt{k}$ for $k \ge 1$.
    For every integer $r \ge 1$ and every $k \ge 1$, the condition $\lfloor \sqrt{k} \rfloor = r$ is equivalent to
    \begin{equation*}
        r^2 \le k < (r+1)^2,
    \end{equation*}
    The extra iteration $k=0$ belongs to $\mathcal{K}_1$. Hence
    \begin{equation*}
        |\mathcal{K}_r| \le 2r+2
    \end{equation*}
    for all $r \ge 1$. Therefore $N_1 \le 4$, while $N_r \le 3$ for all $r \ge 2$.

    Let $r_{\max} \coloneqq \lfloor \sqrt{K} \rfloor$.
    Hence the total number of sub-blocks is at most
    \begin{equation*}
        B \le 4 + 3(r_{\max}-1) = 3r_{\max}+1.
    \end{equation*}
    Moreover, after the first threshold level, each new level contributes at most three additional sub-blocks, so the threshold attached to the $j$-th sub-block is at most
    \begin{equation*}
        \max\left\{1,\left\lceil \frac{j-1}{3} \right\rceil\right\}.
    \end{equation*}
    Therefore, using the monotonicity of $t(r;T)$ in $r$, the recursion from \Cref{lem:time_RK_arbitrary} is dominated by
    \begin{align*}
        S_0 &\coloneqq 0, \\
        S_j &\coloneqq S_{j-1} + t\!\left(\max\left\{1,\left\lceil \frac{j-1}{3} \right\rceil\right\}; S_{j-1}\right), \qquad j = 1, \dots, 3r_{\max}+1,
    \end{align*}
    and the resulting completion time satisfies
    \begin{equation*}
        T(K) \le S_{3r_{\max}+1} = S_{3\lfloor \sqrt{K} \rfloor + 1}.
    \end{equation*}
\end{proof}

\begin{theorem}[Parameter-agnostic time complexity under the universal computation model]
    \label{thm:time_arbitrary}
    Consider \Cref{alg} with
    \begin{equation*}
        \alpha_0 = 1,\qquad
        \alpha_k = \frac{1}{\sqrt{k}}\quad(k \ge 1),\qquad
        R_0 = 1,\qquad
        R_k = \sqrt{k}\quad(k \ge 1),
        \qquad
        \eta_k = \frac{\eta}{(k+1)^{3/4}}
    \end{equation*}
    Let
    \begin{equation*}
        K_{\varepsilon}
        =
        \widetilde{\mathcal{O}}\left(\frac{\Psi^4}{\varepsilon^4}\right)
    \end{equation*}
    be the iteration complexity from \Cref{thm:iteration_complexity_param_agnostic}.
    Then, under \Cref{assump:universal_time}, the time complexity of reaching an $\varepsilon$-stationary point satisfies
    \begin{equation*}
        T_{\varepsilon}
        \le
        S_{3\lfloor \sqrt{K_{\varepsilon}} \rfloor + 1}
        =
        S_{\widetilde{\mathcal{O}}(\Psi^2/\varepsilon^2)},
    \end{equation*}
    Here and below, $\widetilde{\mathcal{O}}(\cdot)$ hides logarithmic factors in $1/\varepsilon$.
    The sequence $\{S_j\}$ is defined recursively by
    \begin{align*}
        S_0 &\coloneqq 0, \\
        S_j &\coloneqq S_{j-1} + t\!\left(\max\left\{1,\left\lceil \frac{j-1}{3} \right\rceil\right\}; S_{j-1}\right).
    \end{align*}
\end{theorem}

\begin{proof}
    By \Cref{thm:iteration_complexity_param_agnostic}, the method reaches an $\varepsilon$-stationary point within
    \begin{equation*}
        K_{\varepsilon}
        =
        \widetilde{\mathcal{O}}\left(\frac{\Psi^4}{\varepsilon^4}\right)
    \end{equation*}
    iterations.
    Applying \Cref{lem:time_sqrt_arbitrary} with $K = K_{\varepsilon}$ yields
    \begin{equation*}
        T_{\varepsilon}
        \le
        S_{3\lfloor \sqrt{K_{\varepsilon}} \rfloor + 1}.
    \end{equation*}
    Since $\sqrt{K_{\varepsilon}} = \widetilde{\mathcal{O}}(\Psi^2/\varepsilon^2)$, the claim follows.
\end{proof}

\subsection{Comparison with \ringmastertitle}
We briefly compare the recursive time bounds obtained for \ringmaster and \algn under the universal computation model.
Because the duration function $t(r;T)$ depends on both the block size $r$ and the starting time $T$, this comparison is necessarily weaker than in the fixed computation model.

For \ringmaster, the iteration complexity in \citet{maranjyan2025ringmaster} is
\begin{equation*}
    \bar{K} \eqdef \left\lceil\frac{48 L \Delta \sigma^2}{\varepsilon^4}\right\rceil .
\end{equation*}
Take the fixed threshold
\begin{equation*}
    R \eqdef \left\lceil \frac{\sigma}{\varepsilon^2} \right\rceil .
\end{equation*}
Applying the fixed-threshold bound above with this choice of $R$ gives
\begin{equation*}
    T_{\varepsilon}^{\rm Ringmaster}
    \le
    T^{(R)}_{N_{\rm Ringmaster}},
\end{equation*}
where
\begin{equation*}
    N_{\rm Ringmaster} \coloneqq \left\lceil \frac{\bar{K}}{R} \right\rceil
\end{equation*}
and the sequence $\{T^{(R)}_k\}$ is defined by
\begin{align*}
    T^{(R)}_0 &\coloneqq 0, \\
    T^{(R)}_k &\coloneqq T^{(R)}_{k-1} + t(R;T^{(R)}_{k-1}).
\end{align*}
Since $\bar{K} = \mathcal{O}(\sigma^2 L \Delta/\varepsilon^4)$ and $R = \Theta(\sigma/\varepsilon^2)$, we obtain
\begin{equation*}
    N_{\rm Ringmaster}
    =
    \mathcal{O}\left(\frac{\sigma L\Delta}{\varepsilon^2}\right).
\end{equation*}

For our method, the iteration complexity from \Cref{thm:iteration_complexity_param_agnostic} is
\begin{equation*}
    K_{\varepsilon} = \widetilde{\mathcal{O}}\left(\frac{\Psi^4}{\varepsilon^4}\right),
\end{equation*}
and \Cref{thm:time_arbitrary} yields the time bound
\begin{equation*}
    T_{\varepsilon}
    \le
    S_{N_{\rm ours}},
\end{equation*}
where
\begin{equation*}
    N_{\rm ours} \coloneqq 3\lfloor \sqrt{K_{\varepsilon}} \rfloor + 1
\end{equation*}
and the sequence $\{S_j\}$ is defined by
\begin{align*}
    S_0 &\coloneqq 0, \\
    S_j &\coloneqq S_{j-1} + t\!\left(\max\left\{1,\left\lceil \frac{j-1}{3} \right\rceil\right\}; S_{j-1}\right).
\end{align*}
Since $\sqrt{K_{\varepsilon}} = \widetilde{\mathcal{O}}(\Psi^2/\varepsilon^2)$, this gives
\begin{equation*}
    N_{\rm ours}
    =
    \widetilde{\mathcal{O}}\left(\frac{\Psi^2}{\varepsilon^2}\right).
\end{equation*}

Thus, from the point of view of the $\varepsilon$-dependence, the two methods are already quite close: both require on the order of $1/\varepsilon^2$ recursion steps, up to logarithmic factors for our method.
Their problem-dependent prefactors are different, so this should be interpreted only as an $\varepsilon$-scaling comparison.

The difference lies in how these blocks are formed.
\ringmaster uses the same threshold $R$ in every block, while our method uses the increasing thresholds
\begin{equation*}
    1,1,1,1,2,2,2,3,3,3,\dots,\lfloor \sqrt{K_{\varepsilon}} \rfloor,\lfloor \sqrt{K_{\varepsilon}} \rfloor,\lfloor \sqrt{K_{\varepsilon}} \rfloor .
\end{equation*}
Because the quantity $t(r;T)$ depends on both the threshold $r$ and the starting time $T$, the universal computation model does not by itself give a constant-factor comparison between
\begin{equation*}
    T^{(R)}_{N_{\rm Ringmaster}}
    \qquad \text{and} \qquad
    S_{N_{\rm ours}}.
\end{equation*}
Indeed, the worker speeds may change arbitrarily with time, so two recursions with a similar number of steps need not be within a constant factor of each other.

Therefore, in the fully general universal model, the clean conclusion is that both methods involve about $1/\varepsilon^2$ recursion steps, but we cannot claim that our bound matches the \ringmaster bound up to a universal constant without extra assumptions on the computation-rate functions $\{p_i\}_{i=1}^n$.

If one wants a genuine constant-factor time comparison, a sufficient extra condition is that block durations of the relevant size remain stable under shifts of the starting time.
Let
\begin{equation*}
    R_{\varepsilon}^{\star} \coloneqq \lfloor \sqrt{K_{\varepsilon}} \rfloor,
\end{equation*}
and assume that there exists a constant $C_{\star} \ge 1$ such that
\begin{equation*}
    t(R_{\varepsilon}^{\star};T)
    \le
    C_{\star}\, t(R_{\varepsilon}^{\star};0)
    \qquad \text{for all } T \ge 0.
\end{equation*}
Since
\begin{equation*}
    \left\lceil \frac{j}{3} \right\rceil \le R_{\varepsilon}^{\star}
    \qquad \text{for } j = 1,\dots,N_{\rm ours},
\end{equation*}
and $t(r;T)$ is nondecreasing in $r$, every increment in the recursion for $\{S_j\}$ satisfies
\begin{equation*}
    t\!\left(\left\lceil \frac{j}{3} \right\rceil;S_{j-1}\right)
    \le
    t(R_{\varepsilon}^{\star};S_{j-1})
    \le
    C_{\star}\, t(R_{\varepsilon}^{\star};0).
\end{equation*}
Therefore,
\begin{equation*}
    T_{\varepsilon}
    \le
    S_{N_{\rm ours}}
    \le
    C_{\star}\, N_{\rm ours}\, t(R_{\varepsilon}^{\star};0)
    =
    \widetilde{\mathcal{O}}\left(
        \frac{\Psi^2}{\varepsilon^2}\, t(\sqrt{K_{\varepsilon}};0)
    \right).
\end{equation*}
In other words, under this localized start-time stability assumption, our adaptive schedule has the same time scaling as a fixed-threshold scheme at the natural block size $R_{\varepsilon}^{\star}$, up to constant and logarithmic factors.

\end{document}